\documentclass{article}

\usepackage[protrusion=true,expansion=true]{microtype}
\usepackage{graphicx}
\usepackage{booktabs} 
\usepackage{mathtools}
\usepackage{caption}
\usepackage{subcaption}
\usepackage{bbm} 
\usepackage{multirow} 
\usepackage{amsmath}
\usepackage{amssymb}
\usepackage{fontawesome5}
\usepackage{amsthm}
\usepackage{enumitem}
\usepackage{tabularx,booktabs,array}
\usepackage{xargs}
\usepackage[dvipsnames,table]{xcolor}   
\usepackage[most]{tcolorbox}
\usepackage{xparse}
\usepackage{booktabs, siunitx, multirow, graphicx,threeparttable,subcaption}  
\usepackage{arydshln}                    
\definecolor{refrow}{HTML}{E2E2F0}
\usepackage[numbers,sort&compress]{natbib}
\usepackage{wrapfig}
\usepackage{algorithm}
\usepackage{algorithmic}
\usepackage{hyperref}
\usepackage{comment}
\usepackage[normalem]{ulem} 
\usepackage[capitalize,noabbrev]{cleveref}
\usepackage{xspace}
\definecolor{projblue}{HTML}{BFEFFF}
\definecolor{gtrade}{HTML}{2E8B57} 

\hypersetup{
    colorlinks=true,
    breaklinks=true,
    citecolor=[RGB]{0, 102, 204},
    linkcolor=Maroon,
    urlcolor=[RGB]{0, 51, 153}
}
\usepackage{tikz}
\usetikzlibrary{decorations.pathreplacing}
\usetikzlibrary{shapes.geometric}
\definecolor{anchorblue}{RGB}{50,100,170}
\definecolor{cA}{RGB}{80,110,175}
\definecolor{cB}{RGB}{185,65,50}
\definecolor{cC}{RGB}{130,160,70}
\definecolor{maxred}{RGB}{190,55,55}
\definecolor{reassigngreen}{RGB}{40,140,80}
\definecolor{splitpurple}{RGB}{120,60,160}
\definecolor{cmred}{RGB}{200,50,50}
\definecolor{cmgreen}{RGB}{30,130,60}
\definecolor{gtrade}{RGB}{46,139,87}
\definecolor{rtrade}{RGB}{205,65,65}

\definecolor{floworg}{RGB}{200,130,40}
\definecolor{brenierblue}{RGB}{45,90,165}

\definecolor{icnnblue}{RGB}{0, 80, 160}
\definecolor{mlpred}{RGB}{180, 40, 30}

\definecolor{gtrade}{RGB}{72,145,100}
\definecolor{rtrade}{RGB}{190,82,82}

\usepackage[final,main]{neurips_2026}

\usepackage[colorinlistoftodos,prependcaption,textsize=tiny]{todonotes}
\newcommandx{\Daniel}[2][1=]{\todo[linecolor=orange,backgroundcolor=orange!25,bordercolor=orange,#1]{[Daniel]: #2}}
\newcommandx{\Alireza}[2][1=]{\todo[linecolor=red,backgroundcolor=red!25,bordercolor=red,#1]{[alireza]: #2}}
\newcommandx{\Ehsan}[2][1=]{\todo[linecolor=blue,backgroundcolor=blue!25,bordercolor=blue,#1]{[Ehsan]: #2}}
\newcommandx{\Buse}[2][1=]{\todo[linecolor=green!30!black!20,backgroundcolor=green!25!black!20,bordercolor=green!30!black!20,#1]{[Buse]: #2}}
\newcommandx{\Marco}[2][1=]{\todo[linecolor=purple,backgroundcolor=purple!25,bordercolor=purple,#1]{[Marco]: #2}}

\newcommand{\Autorefs}[2]{\hyperref[#1]{Assumptions~\ref*{#1}} and~\ref{#2}}

\theoremstyle{plain}
\newtheorem{theorem}{Theorem}[section]
\newtheorem{proposition}{Proposition}[section]
\newtheorem{lemma}{Lemma}[section]
\newtheorem{corollary}{Corollary}[section]
\newtheorem{example}{Example}[section]

\theoremstyle{definition}
\newtheorem{definition}{Definition}[section]
\newtheorem{assumption}{Assumption}[section]

\theoremstyle{remark}
\newtheorem{remark}{Remark}[section]

\crefname{theorem}{theorem}{theorems}
\Crefname{theorem}{Theorem}{Theorems}
\crefname{proposition}{proposition}{propositions}
\Crefname{proposition}{Proposition}{Propositions}
\crefname{lemma}{lemma}{lemmas}
\Crefname{lemma}{Lemma}{Lemmas}
\crefname{corollary}{corollary}{corollaries}
\Crefname{corollary}{Corollary}{Corollaries}
\crefname{definition}{definition}{definitions}
\Crefname{definition}{Definition}{Definitions}
\crefname{assumption}{assumption}{assumptions}
\Crefname{assumption}{Assumption}{Assumptions}
\crefname{remark}{remark}{remarks}
\Crefname{remark}{Remark}{Remarks}
\crefname{example}{example}{examples}
\Crefname{example}{Example}{Examples}

\usepackage[textsize=tiny]{todonotes}

\newcommand{\cA}{\mathcal{A}}

\newcommand{\cC}{\mathcal{C}}

\newcommand{\cM}{\mathcal{M}}

\newcommand{\cO}{\mathcal{O}}
\newcommand{\cP}{\mathcal{P}}

\newcommand{\cR}{\mathcal{R}}
\newcommand{\cS}{{\mathcal{S}}}
\newcommand{\cT}{{\mathcal{T}}}

\newcommand{\cZ}{\mathcal{Z}}
\newcommand{\diff}{\mathrm{d}}

\newcommand{\EE}{\mathbb{E}}

\newcommand{\NN}{\mathbb{N}}
\newcommand{\PP}{\mathbb{P}}

\newcommand{\RR}{\mathbb{R}}

\newcommand{\ZZ}{\mathbb{Z}}

\let\hat\widehat
\let\tilde\widetilde

\DeclareMathOperator*{\argmax}{arg\,max}

\newcommand{\ie}{\textit{i.e.}}
\newcommand{\eg}{\textit{e.g.}}
\newcommand{\cf}{\textit{cf.}}
\newcommand{\wasserstein}{\mathsf{W}}
\newcommand{\namedref}[2]{\hyperref[#1]{#2~(\ref*{#1})}}
\newcommand{\OT}{\mathrm{OT}}

\newcommand{\myparagraph}{\textbf}
\newcommand{\methodname}{ICNN-DRO\xspace}

\renewcommand{\algorithmicrequire}{\textbf{Input:}}
\renewcommand{\algorithmicensure}{\textbf{Output:}}

\NewDocumentCommand{\prettybox}{ O{} +m }{%
\begin{tcolorbox}[
  colback=orange!20,
  colframe=orange!20,
  colbacktitle=orange!20,
  coltitle=black,
  fonttitle=\bfseries,
  title={#1},
  boxrule=0.8pt,
  arc=3mm,
  left=4pt,
  right=4pt,
  top=4pt,
  bottom=4pt,
  width=\textwidth,
  before skip=4pt,
  after skip=4pt,
  breakable,
  before upper={%
    \setlength{\abovedisplayskip}{4pt}%
    \setlength{\abovedisplayshortskip}{0pt}%
    \setlength{\belowdisplayskip}{4pt}%
    \setlength{\belowdisplayshortskip}{2pt}%
  },
]
#2
\end{tcolorbox}
}

\NewDocumentCommand{\prettygreenbox}{ o +m }{%
  \IfNoValueTF{#1}{%
    \begin{tcolorbox}[
      enhanced,
      breakable,
      colback=green!15,
      colframe=green!60!black,
      boxrule=0.3pt,
      arc=3mm,
      left=4pt,
      right=4pt,
      top=4pt,
      bottom=4pt,
      width=\textwidth,
      before skip=4pt,
      after skip=4pt
    ]
    #2
    \end{tcolorbox}%
  }{%
    \begin{tcolorbox}[
      enhanced,
      breakable,
      colback=green!5,
      colframe=green!80!black,
      fonttitle=\bfseries,
      coltitle=black,
      title={#1},
      boxrule=0.8pt,
      arc=3mm,
      left=4pt,
      right=4pt,
      top=4pt,
      bottom=4pt,
      width=\textwidth,
      before skip=4pt,
      after skip=4pt
    ]
    #2
    \end{tcolorbox}%
  }%
}

\definecolor{cnneg}{rgb}{0.65,0.13,0.13}      
\definecolor{cuncon}{rgb}{0.13,0.13,0.65}     
\definecolor{cqedge}{rgb}{0.0,0.42,0.42}      
\definecolor{couterq}{rgb}{0.45,0.18,0.55}    

\title{Brenier Meets Adversarial Training:\\
Optimal Transport Geometry for Robust Learning}

\author{%
  \begin{tabular}{@{}c@{\hspace{1.8em}}c@{\hspace{1.8em}}c@{}}
    Alireza Abdollahpoorrostam$^{\star,1}$ &
    Ehsan Sharifian$^{\star,1}$ &
    Buse \c{S}en$^{\star,1}$
    \\[2pt]
    \multicolumn{3}{c}{%
      Marco Cuturi$^{2}$
      \hspace{3em}
      Daniel Kuhn$^{1}$%
    }
    \\[6pt]
    \multicolumn{3}{c}{\normalfont\small
      $^{1}$EPFL
      \qquad
      $^{2}$Apple
    }
    \\[8pt]
    \multicolumn{3}{c}{\normalfont
      \href{https://github.com/alirezaabdollahpour/ICNN-DRO}{\faGithub\ \texttt{Code}}
    }
  \end{tabular}%
}

\begin{document}
\addtocontents{toc}{\protect\setcounter{tocdepth}{-1}}

\maketitle
\renewcommand{\thefootnote}{}
\footnotetext{$^{\star}$Equal contribution.}
\renewcommand{\thefootnote}{\arabic{footnote}}

\begin{abstract}
Distributionally robust optimization (DRO) provides a principled framework for learning under distribution shift, but its practical use is hindered by the difficulty of evaluating worst-case risks for nonconvex loss functions. We study a penalized DRO formulation in which the adversary may choose any distribution but incurs a Wasserstein penalty for deviating from the empirical distribution. We show that the adversary’s problem can be reformulated as an optimization problem over transport maps that push empirical samples to adversarial ones, and we prove that optimal maps are cyclically monotone. We also show that standard adversarial training---based on per-sample local optimization---violates cyclical monotonicity and wastes transport costs unless the adversary is severely restricted. We propose two remedies. First, we introduce multi-start particle ascent, which alternates parallel gradient ascent with reassignment to enforce cyclical monotonicity across samples. Second, we parameterize adversarial maps as gradients of input-convex neural networks, which guarantees cyclical monotonicity by construction. Experiments on robust regression, image classification, and robust control show that our methods consistently outperform standard adversarial training and state-of-the-art baselines, achieving improved robustness and better generalization under distribution shift.
\looseness=-1
\end{abstract}

\vspace{-0.3cm}
\section{Introduction}
\vspace{-0.3cm}
\label{sec:intro}
Adversarial perturbations expose a pronounced fragility of modern machine learning systems, from early evidence of adversarial examples~\citep{szegedy2014intriguing, goodfellow2014explaining, nguyen2015deep, kurakin2017adversarial} to a broad robustness literature spanning stronger attacks, defense mechanisms~\citep{rade2022reducing, madry2018towards, Enemy} and
evaluation protocols~\citep{carlini2017towards, croce2020reliableevaluationadversarialrobustness, athalye2018obfuscatedgradientsfalsesense}. Yet, robustness to diverse threat models, common corruptions, and distribution shifts remains an open challenge even for well-trained models and certified defenses~\citep{tramer2019adversarial, Enemy, rade2022reducing,
taori2020measuringrobustnessnaturaldistribution}. Robust learning is therefore not only about surviving a prescribed attack, but about generalizing under broader distributional variation. \looseness=-1

Distributionally robust optimization (DRO)~\cite{kuhn2024distributionally} provides a principled framework for adversarial training. Instead of relying on a fixed attack, DRO models a worst-case adversary that perturbs the data distribution. A widely used approach is Wasserstein DRO, which requires the learned model to perform well for all distributions within a prescribed Wasserstein distance of the empirical distribution~\citep{mohajerin2018data, kuhn2019wasserstein, duchi2020learningmodelsuniformperformance, blanchet2017quantifyingdistributionalmodelrisk, gao2024wasserstein}. This formulation is appealing because the Wasserstein geometry encodes which distribution shifts are plausible, while the radius controls their magnitude. \looseness=-1

While conceptually appealing, DRO is not easy to use in practice because the worst-case problem searches over distributions. A common practical variant replaces the hard Wasserstein constraint with a penalized objective, in which the adversary may choose any distribution but incurs a Wasserstein penalty for deviating from the empirical distribution. Evaluating the resulting worst-case risk reduces to a finite-dimensional optimization problem over transported training samples. For standard machine learning models, this problem is typically nonconvex. Existing methods therefore approximate it through independent sample-wise updates or entropic surrogates~\citep{sinha2020certifying, wang2021sinkhorn, xu2025gradientflowsamplerbaseddistributionally}. 
A sufficiently strong penalty on moving samples can make this problem better behaved (more concave). But it does so only by making substantial perturbations too expensive, and hence by weakening the adversary itself. \looseness=-1

This paper takes a first step toward a theory of penalty-based Wasserstein DRO in the practically relevant regime where the adversary is powerful and thus faces a nonconvex problem. We also aim to develop strong approximate solution methods. The key enabling idea is to recast the adversary’s problem as an optimization problem over transport maps, which give rise to the following main contributions.
\looseness=-1
\begin{itemize}[itemsep=0.0pt, topsep=1pt, leftmargin=0.9em]
    \item \textbf{\autoref{sec:Preliminaries}}: We argue that the adversary’s optimal transport map is cyclically monotone and that restricting the adversary to cyclically monotone maps benefits {\em both} the adversary and the learner. \looseness=-1

    \item \textbf{\autoref{sec:implicit}}: We show that standard particle ascent methods for the dual DRO problem generally violate cyclical monotonicity in the nonconvex regime. As a repair, we introduce Multi-start Particle Ascent (MPA), which approximately solves the adversary’s problem by alternating parallel particle ascent over a batch of samples with optimal reassignment of the resulting particles. \looseness=-1

    \item \textbf{\autoref{sec:explicit_maps}}: We also propose an algorithm that models the adversary’s transport map as the gradient of an input-convex neural network (ICNN)~\citep{amos2017input}, thereby enforcing cyclical monotonicity by design. \looseness=-1

    \item \textbf{\autoref{sec:experiment}}: We evaluate both new methods on adversarial training for CIFAR-10 and on robust control tasks with uncertainty in physical parameters. We find that the ICNN-based method delivers the strongest overall robustness, while MPA remains competitive with the DRO baselines.
\end{itemize}

We do not focus on building a state-of-the-art robust training pipeline, which typically incorporates additional regularization techniques~\citep{Trades,rade2022reducing,Enemy}. Instead, our methods (MPA and the ICNN-based approach) are complementary and can be seamlessly integrated with such strategies.

\vspace{-0.3cm}
\section{Background}\label{sec:background}
\vspace{-0.3cm}
We first review some basics of optimal transport and robust machine learning. 

\myparagraph{Optimal transport.} We use $\mathcal{P}$ to denote the space of probability measures $\PP$ on $\mathbb{R}^m$ with finite second moments (\textit{i.e.}, $\mathbb{E}_{\PP}[\|Z\|_2^2] < \infty$), and we equip $\mathcal{P}$ with the Wasserstein distance. 
\begin{definition}[Wasserstein distance]\label{def:wass-defn} The Wasserstein distance of order 2 between two distributions $\PP, \hat{\PP}\in \mathcal{P}$ is given by  $\wasserstein(\mathbb{P}, \hat{\mathbb{P}}) \coloneqq \inf \{ \mathbb{E}_{\gamma} [\|Z - \hat{Z}\|_2^2] ^{1/2}: \gamma \in \Gamma(\mathbb{P}, \hat{\mathbb{P}})\}$, where $\Gamma(\mathbb{P}, \hat{\mathbb{P}})$ denotes the set of joint distributions or couplings of $Z$ and $\hat Z$ with marginals $\mathbb{P}$ and $\hat{\mathbb{P}}$, respectively. \looseness=-1
\end{definition}
The infimum in the definition of $\wasserstein(\PP,\hat{\PP})$ is always attained by some optimal coupling~$\gamma^\star$. Below we use~$\mathcal{T}$ to denote the set of all Borel-measurable functions $T: \mathbb{R}^m \to \mathbb{R}^m$, and we call elements of~$\mathcal{T}$ transport maps. 

If~$\hat{\mathbb{P}}$ is absolutely continuous with respect to the Lebesgue measure, then the optimal coupling~$\gamma^\star$ is induced by a transport map. Specifically, Brenier's theorem asserts that there is a unique $T^\star \in \mathcal{T}$ that is cyclically monotone and satisfies $\gamma^\star = (T^\star, \mathrm{id})_{\#} \hat{\mathbb{P}}$ and $\PP=T^\star_\#\hat{\PP}$ \cite{Brenier1991}.
\begin{definition}[Cyclical monotonicity]
\label{def:cm_functional_map}
Let $\cS_k$ denote the set of all permutations of $\{1, \dots, k\}$. A set $\cC\subseteq \mathbb R^m \times \mathbb R^m$ is called cyclical monotone if for any $k\in\NN$, any collection of pairs $\{(\hat{z}_1, z_1), \dots, (\hat{z}_k, z_k)\} \subseteq\cC$ and any $\sigma\in\cS_k$, the following equivalent conditions hold:
\vspace{-0.099cm}
{%
    \abovedisplayskip=1pt plus 1pt minus 1pt
    \belowdisplayskip=1pt plus 1pt minus 1pt
    \begin{equation*}
         \sum_{i=1}^{k} \|\hat{z}_i - z_i\|_2^2 \leq \sum_{i=1}^{k} \|\hat{z}_i - z_{\sigma(i)}\|_2^2 \;\; \iff \;\;  \sum_{i=1}^{k} \langle z_{\sigma(i)} - z_i, \hat{z}_i \rangle \leq 0.
    \end{equation*}
}%
We say that a map $T\in\cT$ is cyclically monotone on a set $\cZ\subseteq \RR^m$ if its graph $ \{(\hat{z}, T(\hat{z})) : \hat{z} \in \cZ\}$ is a cyclically monotone set. If $\cZ=\RR^m$, we simply say that $T$ is cyclically monotone.
\end{definition}
By Rockafellar's theorem~\citep[Theorem~2.5.3]{figalli2021invitation}, a map $T \in \mathcal{T}$ is cyclically monotone if and only if there exists a convex function $\psi: \mathbb{R}^m \to \mathbb{R} \cup \{\infty\}$ such that $T(z) \in \partial\psi(z)$  for all $z\in\RR^m$. Any transport map that violates cyclical monotonicity induces a suboptimal coupling between~$\PP$ and~$\hat\PP$, and the degree of suboptimality is sometimes referred to as the Monge gap~\cite{uscidda2023monge}.
\begin{definition}[Monge gap]\label{def:monge_gap} Given any $\hat{\PP}\in\cP$, the Monge gap of a map $T\in\cT$ is defined as
{%
    \abovedisplayskip=0.8pt plus 0.8pt minus 0.8pt
    \belowdisplayskip=0.5pt plus 0.8pt minus 0.8pt
    \begin{equation*}
        \cM_{\hat{\PP}}(T)\coloneqq \EE_{\hat{\PP}}[\|\hat Z -T(\hat Z)\|_2^2] - \wasserstein^2(T_{\#}\hat{\PP}, \hat{\PP}).
    \end{equation*}
}%
\end{definition}
Note that $\cM_{\hat{\PP}}(T)\geq0$, and $\cM_{\hat{\PP}}(T)=0$ if and only if $T$ is cyclically monotone on the support of $\hat\PP$. \looseness=-1

\myparagraph{Robust learning.}
Let $f : \Theta \times \mathbb{R}^m \to \mathbb{R}$ be a function that assigns each model parameter~$\theta$ from within a closed convex set~$\Theta \subseteq \mathbb{R}^d$ and each data point $z \in \mathbb{R}^m$ a training loss~$f(\theta,z)$. Throughout the paper we assume that~$f(\theta,z)$ is continuous in~$\theta$ and differentiable as well as Lipschitz continuous in~$z$.
Empirical Risk Minimization (ERM) seeks to minimize the expected loss under some $\hat{\mathbb{P}}\in\mathcal P$.
{%
\abovedisplayskip=1pt plus 1pt minus 1pt
\belowdisplayskip=1pt plus 1pt minus 1pt
\begin{equation}\tag{ERM}\label{ERM Problem}
    \min_{\theta \in \Theta} \; \mathbb{E}_{\hat{\mathbb{P}}} [f(\theta, \hat{Z})]
\end{equation}
}%
The data distribution $\hat{\PP}$ is typically set to the empirical measure $\hat{\mathbb{P}} = \frac{1}{N}\sum_{i=1}^N \delta_{\hat{z}_i}$ corresponding to a given training dataset. While computationally convenient, ERM is sensitive to distribution shifts and tends to overfit to the training data. As a remedy, Robust Optimization (RO) minimizes the expected value of the {\em worst-case} loss with respect to all data realizations in an $\varepsilon$-neighborhood of~$\hat Z$ \cite{madry2018towards}.
\begin{equation}\tag{RO}\label{prob:robust-opt}
    \min_{\theta \in \Theta} \mathbb{E}_{\hat{\mathbb{P}}} \Big[ \max \left\{ f(\theta, z) : \|z - \hat{Z}\|_2 \leq \varepsilon\right\} \Big]
\end{equation}
While RO accounts for local data perturbations, it does not provide protection against global distribution shifts. Moreover, unless~$f(\theta, z)$ is concave in~$z$, the embedded maximization problems in~\eqref{prob:robust-opt} can only be solved approximately, which is usually done via iterative first-order methods~\cite{madry2018towards}. Distributionally Robust Optimization (DRO) offers a principled approach to adversarial training that models distribution shifts explicitly. In penalty-based Wasserstein DRO a fictitious adversary maximizes the expected loss over all possible data distributions~$\PP\in\mathcal P$ subject to a penalty proportional to $\wasserstein^2(\mathbb{P},\hat{\mathbb{P}})$ for choosing distributions that differ from~$\hat\PP$. Thus, the training goal is to solve
\begin{equation}\tag{DRO}\label{prob:regularized_dro}
\min_{\theta \in \Theta} \;\bigl\{F(\theta) \coloneqq \max_{\mathbb{P} \in \cal P}\; \mathbb{E}_{\mathbb{P}}  \left[f(\theta, Z)\right] - \lambda\wasserstein^2(\mathbb{P},\hat{\mathbb{P}})\bigr\},
\end{equation}
where~$\lambda > 0$ captures the strength of the penalty. We refer to~$F(\theta)$ as the \emph{adversarial risk}.

\vspace{-0.3cm}
\section{Reparameterizations of the Adversarial Risk}
\vspace{-0.3cm}
\label{sec:Preliminaries}
Evaluating the adversarial risk~$F(\theta)$ amounts to maximizing a concave functional over an infinite-dimensional space of probability measures. Standard duality results in Wasserstein DRO such as \citep[Proposition~1]{sinha2020certifying} allow us to reformulate~$F(\theta)$ as the empirical expectation of the \emph{adversarial loss}, which is defined as the negative Moreau envelope of~$-f(\theta,z)$; see~\autoref{app:Technical_Background} for details.
\begin{equation}\label{eq:regularized_dro-dual}
    F(\theta) = \mathbb{E}_{\hat{\mathbb{P}}} \Big[\max_{z \in \RR^m} \bigl\{ f(\theta, z) - \lambda \|z - \hat{Z}\|_2^2 \bigr\}\Big].
\end{equation}
Hence, evaluating~$F(\theta)$ reduces to solving an unconstrained global maximization problem over~$z$ for each empirical sample~$\hat{z}$. The main results of this paper will exploit yet another reformulation of~$F(\theta)$ as the maximum of a functional optimization problem over the space of transport maps~$\mathcal{T}$. Indeed, by the interchangeability principle~\citep[Theorem 14.60]{rockafellar1998variational}, we can recast the adversarial risk in~\eqref{eq:regularized_dro-dual} as
\begin{align}\label{eq:maximize-over-maps}
\begin{aligned}
     F(\theta) =\max_{T \in \mathcal{T}}\; \mathbb{E}_{\hat{\mathbb{P}}} \left[ f(\theta, T(\hat{Z})) - \lambda \|T(\hat{Z}) - \hat{Z}\|_2^2 \right].
\end{aligned}
\end{align}
The reformulation~\eqref{eq:maximize-over-maps} seeks an adversarial transport map~$T^\star$ that simultaneously solves the global optimization problems associated with \emph{all} empirical samples in~\eqref{eq:regularized_dro-dual}. Even though the objective function value of~$T^\star$ in problem~\eqref{eq:maximize-over-maps} depends only on its behavior on the support of~$\hat{\mathbb{P}}$, it is expedient to extend~$T^\star$ to the entire ambient space. Formally, this can be done by constructing a set-valued mapping that assigns each empirical sample~$\hat{z} \in \mathbb{R}^m$ the corresponding set of adversarial samples
\begin{align}\label{eq:T_star_argmax}
     S(\hat{z}) \coloneqq \arg\max_{z \in \RR^m} \left\{ f(\theta, z) - \lambda \|z - \hat{z}\|_2^2 \right\}.
\end{align}
Next, we define the optimal adversarial map~$T^\star$ as a Borel-measurable selector of this set-valued mapping. The existence of such a selector is shown in \autoref{prop:measurability_Tstar}. Since $T^\star(\hat z)\in S(\hat{z})$ for all $\hat z\in\RR^m$, \citep[Theorem 14.60]{rockafellar1998variational} implies that~$T^\star$ solves~\eqref{eq:maximize-over-maps}. In addition, the pushforward $T^\star_\# \hat{\mathbb{P}}$ constitutes a worst-case distribution for the inner maximization problem in~\eqref{prob:regularized_dro}; see \autoref{prop:inner_deterministic_optimizer}. 
\looseness=-1

The equivalent reformulations~\eqref{eq:regularized_dro-dual} and~\eqref{eq:maximize-over-maps} of the adversarial risk~$F(\theta)$ give rise to two complementary methods for constructing adversarial maps. The pointwise formulation~\eqref{eq:regularized_dro-dual} gives rise to methods for constructing \emph{implicit} maps, which are only defined implicitly through simple algorithms for evaluating the map at a given input~$\hat z\in\mathbb{R}^m$. Conversely, the global formulation~\eqref{eq:maximize-over-maps} inspires the design of \emph{explicit} maps, which are modeled as explicit parametric functions such as neural networks.
\looseness=-1

A simple approach for constructing implicit maps is due to \citet{sinha2020certifying}. They show that if $\nabla_z f(\theta, z)$ is $L_{zz}$-Lipschitz continuous in~$z$ and~$\lambda > L_{zz}/2$, then problem~\eqref{eq:T_star_argmax} has a strongly concave objective function and thus a unique maximizer. Hence, $T^\star$ can be evaluated pointwise by solving~\eqref{eq:T_star_argmax} via gradient ascent. They further show that problem~\eqref{eq:regularized_dro-dual} is amenable to standard stochastic gradient descent (SGD) methods. Specifically, for each sample $\hat{z} \sim \hat{\mathbb{P}}$, they compute~$T^\star(\hat z)$ by solving~\eqref{eq:T_star_argmax} and then use Danskin's theorem to differentiate the parametric maximum in~\eqref{eq:regularized_dro-dual} with respect to~$\theta$, thus constructing an unbiased stochastic gradient for~$F(\theta)$. This naturally leads to an efficient inner-maximization outer-minimization loop. However, this procedure hinges on the condition $\lambda >  L_{zz}/2$, which is difficult to satisfy in practice if $f(\theta, z)$ depends on~$z$ through a neural network. As shown in~\cite{virmaux2018lipschitz}, even computing first-order Lipschitz constants is NP-hard, with loose upper bounds reaching the order of~$10^6$ for standard architectures like AlexNet~\citep{krizhevsky2012imagenet}. Bounding~$L_{zz}$, which relates to the network's Hessian, generally inherits and often exacerbates these first-order intractabilities. When $\lambda$ is set sufficiently large to guarantee strong concavity, adversarial perturbations are pulled close to the training samples, and the adversarial risk degenerates toward the empirical risk. In other words, the very condition that makes the problem tractable undermines its robustness. 
\looseness=-1

In practice, reasonable values of~$\lambda$ that correspond to a sufficiently powerful adversary are significantly smaller than~$L_{zz}$ and invariably render problem~\eqref{eq:T_star_argmax} nonconvex. Given the intractability of nonconvex optimization, the best we can thus hope for is to construct (implicit or explicit) adversarial maps that solve~\eqref{eq:maximize-over-maps} approximately. The \emph{adversarial error} $\varepsilon(\theta, T)$ of a given map~$T\in\mathcal T$ in problem~\eqref{eq:maximize-over-maps} corresponding to a fixed parameter~$\theta \in\Theta$ is defined as the degree of suboptimality of~$T$, that is,
\begin{equation}
    \label{eq:oracle_error}
    \varepsilon(\theta, T) \coloneqq F(\theta) - \EE_{\hat{\PP}}\bigl[f(\theta,T(\hat Z))-\lambda\|T(\hat Z)-\hat Z\|_2^2\bigr] \ge  0.
\end{equation}
The following proposition shows that access to an oracle that approximately solves~\eqref{eq:maximize-over-maps} with small adversarial error enables one to solve~\eqref{prob:regularized_dro} to high accuracy via projected gradient descent.
\looseness=-1

\begin{proposition}[Convergence with inexact adversary]
\label{prop:convex_bound_inexact}
Assume that $f(\theta, z)$ is convex in $\theta$ with
uniformly bounded gradients $\|\nabla_\theta f(\theta, z)\|_2 \le L_\theta$
for all $\theta \in \Theta$ and $z \in \RR^m$. Fix any~$\theta_0\in\Theta$, let
$\theta^\star \in \arg\min_{\theta \in \Theta} F(\theta)$, and consider $H$ projected gradient descent iterates satisfying
\[
    \theta_{t+1}=\operatorname{Proj}_\Theta\!\bigl(\theta_t - \alpha\, g_t\bigr), \quad\text{where}
    \quad
    g_t =\EE_{\hat{\PP}} \bigl[\nabla_\theta f(\theta_t, T_t(\hat Z))\bigr],
\]
where $T_t\in\cT$ is a near-optimal adversarial map that approximately solves~\eqref{eq:maximize-over-maps} at~$\theta=\theta_t$. If the constant step size is set to $\alpha = \|\theta_0 - \theta^\star\|_2 / L_{\theta}\sqrt{H}$, the average iterate $\bar\theta_H \coloneqq \frac{1}{H}\sum_{t=0}^{H-1}\theta_t$ satisfies
\begin{equation*}
\label{eq:convex_bound_inexact}
F(\bar\theta_H) - F(\theta^\star)
\;\le\;
\frac{L_{\theta}\,\|\theta_0 - \theta^\star\|_2}{\sqrt{H}}
\;+\;
\frac{1}{H}\sum_{t=0}^{H-1}\varepsilon(\theta_t, T_t).
\end{equation*}
\end{proposition}
\vspace{-0.3cm}
\autoref{prop:convex_bound_inexact} adapts a standard convergence result for minimax problems to our setting~\citep[\S~3.3]{bertsekas2015convex}. For completeness we provide a short proof in~\autoref{app:ot_structure}. It suggests that, in order to compute near-optimal minimizers for~\eqref{prob:regularized_dro}, one needs access to a strong adversarial oracle that outputs near-optimal adversarial maps. While \autoref{prop:convex_bound_inexact} only holds under the simplifying assumption that the loss function~$f(\theta, z)$ is convex in~$\theta$, we expect the conclusions of \autoref{prop:convex_bound_inexact} to remain valid in the vicinity of local minimizers even if the loss is nonconvex. The next proposition shows that the adversarial error of any transport map is bounded below by its Monge gap. \looseness=-1
{\setlength{\abovedisplayskip}{2pt}
 \setlength{\belowdisplayskip}{2pt}
 \setlength{\abovedisplayshortskip}{0pt}
 \setlength{\belowdisplayshortskip}{0pt}
\begin{proposition}[Suboptimality from wasteful transport]
\label{prop:wasted_transport}
We have \(
   \varepsilon(\theta, T)  \geq \lambda \cM_{\hat\PP}(T)
\)
for all~$T \in \mathcal{T}$.
\end{proposition}}
Note that $\cM_{\hat \PP}(T)>0$ indicates that~$T$ transports $\hat\PP$ to $T_\#\hat \PP$ in an inefficient manner. \autoref{prop:wasted_transport} implies that such inefficient transport necessarily leads to suboptimality in~\eqref{eq:maximize-over-maps} and weakens the convergence guarantees of \autoref{prop:convex_bound_inexact}. The following result is a direct corollary of \autoref{prop:wasted_transport}. \looseness=-1

\begin{proposition}
    \label{prop:optimal_map_cm} The optimal adversarial map $T^\star$  is cyclically monotone, and thus $\cM_{\hat \PP}(T^\star)=0$.
\end{proposition}

\autoref{prop:optimal_map_cm} establishes a necessary optimality condition for Wasserstein DRO problems. It is reminiscent of Brenier's theorem~\citep{Brenier1991}, which establishes a similar optimality condition for optimal transport problems. However, it holds even if~$\hat{\mathbb{P}}$ is discrete and~$f(\theta,z)$ fails to be concave in~$z$. 

\looseness=-1

Evaluating~$F(\theta)$ requires solving a separate optimization problem for every point in the support~of~$\hat \PP$; see~\eqref{eq:regularized_dro-dual}. Unless~$\lambda$ is astronomically large, all of these optimization problems are nonconvex and thus intractable. However, these global optimization problems are {\em coupled} because their objective functions differ only with regard to the anchor point~$\hat z$ of the quadratic penalty term. Accordingly, 
\autoref{prop:optimal_map_cm} reveals that their respective global maximizers~$T^\star(\hat z)$ are connected through the cyclical monotonicity of~$T^\star$. 
Even though the exact adversarial map~$T^\star$ is hard to compute, the insight that a map~$T$ can only be optimal if it is cyclically monotone reduces the search space for good adversarial maps and guides us, by virtue of \autoref{prop:convex_bound_inexact}, toward better solutions for~\eqref{prob:regularized_dro}.
\looseness=-1

\vspace{-0.3cm}
\section{Implicit Adversarial Maps}
\label{sec:implicit}
\vspace{-0.3cm}
\emph{Implicit} adversarial maps are defined in terms of an algorithm for point evaluation. A simple implicit map is obtained by using gradient ascent initialized at~$\hat z$ for solving the nonconvex maximization problem~\eqref{eq:regularized_dro-dual} to local optimality. This method naturally generalizes the approach by \citet{sinha2020certifying} to nonconvex instances of~\eqref{eq:regularized_dro-dual} with moderate penalty parameters~$\lambda\leq L_{zz}/2$. We first provide a formal definition of this particle ascent method. \looseness=-1
\begin{wrapfigure}[12]{r}{0.60\textwidth}
\vspace{-0.3cm}
\centering
\includegraphics[width=1.0\linewidth, clip] 
{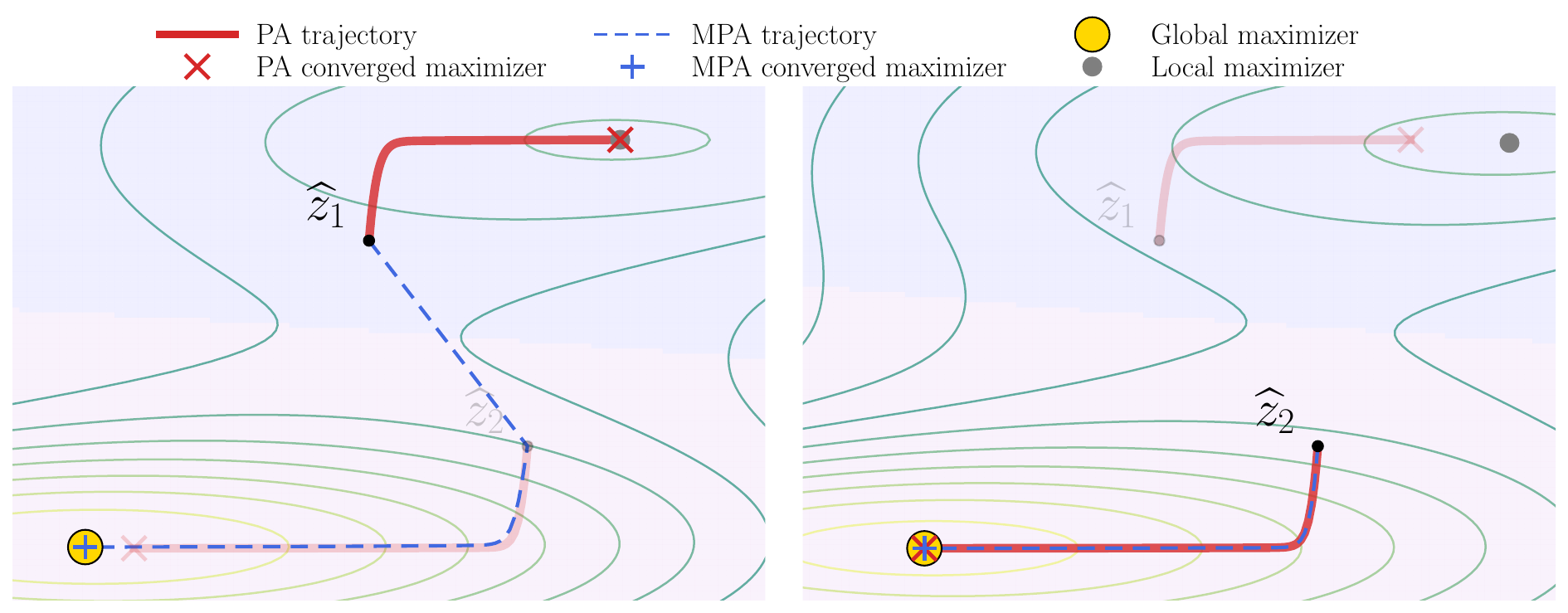}
\caption{PA and MPA trajectories over the contours of $f(\theta, z) - \lambda\|z - \hat{z}_i\|_2^2$ for $i=1$ (left) and $i=2$ (right).}
\label{fig:2D_counterexample}
\vspace{-0.15cm}
\end{wrapfigure}

\myparagraph{Particle Ascent (PA).}
PA approximately solves the maximization problem~\eqref{eq:regularized_dro-dual} by using gradient ascent initialized at~$\hat{z}$, either in continuous or discrete time. In both cases, we use~$T_{\mathrm{pa}}$ to denote the resulting family of implicit maps parametrized by time. The correct interpretation of~$T_{\mathrm{pa}}$ will always be clear from the context. In continuous time, $T_{\mathrm{pa}}: \RR_+ \times \RR^m \to \RR^m$ is defined via the ODE
\begin{align*}
  \partial_t T_{\mathrm{pa}}(t,\hat{z})
  = \nabla_z f\!\left(\theta, T_{\mathrm{pa}}(t,\hat{z})\right)
    - 2\lambda\bigl(T_{\mathrm{pa}}(t,\hat{z}) - \hat{z}\bigr)
\end{align*}
initialized by~$T_{\mathrm{pa}}(0,\hat{z}) = \hat{z}$. In discrete time, $T_{\mathrm{pa}}: \ZZ_{+} \times \RR^m \to \RR^m$ is defined via the recursion
\begin{equation*}
  T_{\mathrm{pa}}(t+1,\hat{z})
  = T_{\mathrm{pa}}(t,\hat{z})
  + \eta\!\left(
      \nabla_z f(\theta, T_{\mathrm{pa}}(t,\hat{z}))
      - 2\lambda\bigl(T_{\mathrm{pa}}(t,\hat{z}) - \hat{z}\bigr)
    \right)
\end{equation*}
initialized by $T_{\mathrm{pa}}(0,\hat{z}) = \hat{z}$, where $\eta>0$ is the step size. When~$t$ is fixed and clear from context, we abbreviate $T_{\mathrm{pa}}(t,\hat z)$ as $T_{\mathrm{pa}}(\hat z)$. Note that~$T_{\mathrm{pa}}$ is implicitly defined through an optimization oracle.
\looseness=-1

The per-sample formulation of PA is computationally appealing, as it allows different particles to be updated in parallel. Even though $T_{\mathrm{pa}}$ is generically suboptimal in~\eqref{eq:maximize-over-maps} unless $\lambda>L_{zz}/2$, it is guaranteed to be cyclically monotone (in fact, monotone) for one-dimensional data.
\begin{proposition}[Monotonicity of $T_{\mathrm{pa}}$ when $m=1$]
\label{prop:pa_monotonicity}
If $m=1$ and $\nabla_z f(\theta,\cdot)$ is $L_{zz}$-Lipschitz, then 
\begin{itemize}[leftmargin=2em]\vspace{-0.2cm}
    \item the continuous-time PA map $T_{\mathrm{pa}}(t,\cdot)$ is cyclically monotone for all $t\in\RR_+$;
    \vspace{-0.2cm}
    \item the discrete-time PA map $T_{\mathrm{pa}}(t,\cdot)$ is cyclically monotone for all $t \in \ZZ_+$ if $\eta\le 1/(L_{zz}+2\lambda)$.
    \vspace{-0.2cm}
\end{itemize}
\end{proposition}
Unfortunately, $T_{\mathrm{pa}}$ ceases to be cyclically monotone already in two dimensions.

\begin{example}[$T_{\mathrm{pa}}$ violates cyclical monotonicity]
\label{ex:counterexample}
Set $m=2$, and consider the loss function
\(
    f(\theta, z) = \sum_{k=1}^2 a_k
    e^{-(z-\mu_k)^\top\Sigma_k^{-1}(z-\mu_k)}
\)
with $a_1=200$, $a_2=500$, $\mu_1 = (6,2)^\top$, $\mu_2=(-6,-2)^\top$, and $\Sigma_1 = \Sigma_2 = \mathrm{diag}(64, 1)$. We compute
$T_{\mathrm{pa}}$ in discrete time for $\lambda = 3$ and $t = 8{,}000$ steps with step size $\eta = 0.001$, initialized at $\hat{z}_1 = (-1,1)^\top$ and $\hat{z}_2 = (1,-1)^\top$. The two particles converge to different local maxima, $T_{\mathrm{pa}}(\hat{z}_1) = (2.17, 1.98)^\top$ and $T_{\mathrm{pa}}(\hat{z}_2) = (-3.96, -2.00)^\top$ (see \autoref{fig:2D_counterexample}), yielding $\langle T_{\mathrm{pa}}(\hat{z}_1) - T_{\mathrm{pa}}(\hat{z}_2),\, \hat{z}_1 - \hat{z}_2 \rangle = -4.32 < 0$, a strict violation of cyclical monotonicity.\looseness=-1
\end{example}

\autoref{ex:counterexample} shows via \autoref{prop:wasted_transport} that PA wastes transport costs and generates a suboptimal adversarial map. While we cannot hope to solve the nonconvex problem~\eqref{eq:maximize-over-maps} to global optimality, we can efficiently improve any suboptimal map~$T$ by forcing its Monge gap to~0. To see this, assume that $\hat{\PP} = \tfrac{1}{N}\sum_{i=1}^N \delta_{\hat z_i}$. Then, the Monge gap satisfies
\mbox{$\cM_{\hat\PP}(T)
    = \max_{\sigma \in \cS_N}
    \frac{2}{N}\sum_{i=1}^{N}
    \bigl\langle T(\hat z_{\sigma(i)}) - T(\hat z_i), \hat z_i \bigr\rangle$} 
\citep{uscidda2023monge}. The optimal permutation~$\sigma$ identifies a reassignment of the empirical samples~$\{\hat z_i\}_{i=1}^N$ to the adversarial samples $\{T(\hat z_i)\}_{i=1}^N$, thus inducing an improved map~$T^\sigma$ with a strictly higher objective function value in~\eqref{eq:maximize-over-maps} while ensuring that $T_\#\hat\PP = T^\sigma_\#\hat\PP$. An even better map is obtained by allowing many-to-one assignments rather than bijective permutations, which collapse multiple empirical samples to a common output and thereby modify $T_\#\hat\PP$ itself. Such an assignment can reduce the adversarial error by strictly more than the Monge gap. From now on, we call a map~$T$ assignment-stationary if~$T(\hat z_i)$ attains the largest objective value in problem~\eqref{eq:T_star_argmax} with~$\hat z=\hat z_i$ among all candidate solutions~$\{T(\hat z_j)\}_{j=1}^N$. It cannot be improved by changing the output of~$\hat z_i$ to that of~$\hat z_j$ for any~$j\neq i$. 

\begin{definition}[Assignment-stationarity]
\label{def:assignment_stationary}
A map \(T \in \mathcal{T}\) is \emph{assignment-stationary} on
\(\{\hat z_i\}_{i=1}^N\) if 
\begin{equation}
\label{eq:assign-stat}
f(\theta,T(\hat z_i))-\lambda\|T(\hat z_i)-\hat z_i\|_2^2
\;\ge\;
f(\theta,T(\hat z_j))-\lambda\|T(\hat z_j)-\hat z_i\|_2^2
\qquad
\forall i,j\in[N].
\end{equation}
\end{definition}

Assignment-stationarity can be viewed as a discrete, batch-wide analogue of standard continuous stationarity. Just as an algorithm for solving a single maximization problem should not stop if the objective function can be improved locally by moving along an ascent direction, an algorithm for solving $N$ maximization problems in tandem should not stop if the $i$-th objective function can be improved by jumping to the candidate solution of the $j$-th problem for some $j\neq i$ (even if the $i$-th candidate solution is a local maximizer). Beyond guaranteeing that no sample is trapped in a suboptimal assignment, assignment-stationarity implies cyclical monotonicity.

\begin{lemma}
\label{lem:assign_stat_implies_cm}
If $\hat{\mathbb{P}} = \frac{1}{N}\sum_{i=1}^N \delta_{\hat z_i}$ while \(T \in \mathcal{T}\) is assignment-stationary on \(\{\hat z_i\}_{i=1}^N\), then $T$ is cyclically monotone on \(\{\hat z_i\}_{i=1}^N\).
\end{lemma}
The converse implication fails. A cyclically monotone map need not be assignment-stationary because cyclical monotonicity ignores the loss~$f$ (counterexamples exist even for~$T=\mathrm{id}$).

From \autoref{ex:counterexample} we know that the PA map generically violates cyclical monotonicity and consequently fails to be assignment-stationary (\autoref{lem:assign_stat_implies_cm}). We thus propose Multi-start Particle Ascent (MPA) as a refinement of PA that enforces assignment-stationarity on the batch to which it is applied.

\myparagraph{Multi-Start Particle Ascent (MPA).}
MPA computes the right-hand side of~\eqref{eq:regularized_dro-dual} approximately by constructing an adversarial sample for each empirical sample in a given batch of size~$B$. The algorithm alternates between batch-wide reassignment of each empirical sample~$\hat z_i$ to the best adversarial sample in the current pool~$\mathcal Z=\{z_i\}_{i=1}^B$ and parallel local optimization of the $B$ adversarial samples in~$\cZ$; see the inner loop in Algorithm~\ref{alg:implicit_mpa}, which runs over~$R$ rounds. MPA ends with a final reassignment step to enforce assignment-stationarity and thus cyclical monotonicity on~$\{\hat z_i\}_{i=1}^B$. When $R=1$, the reassignment steps highlighted in blue are disabled, and Algorithm~\ref{alg:implicit_mpa} reduces to PA. When $R>1$, the reassignment steps are activated. Both PA and MPA can be embedded into a projected gradient descent algorithm for training~$\theta$; see the outer loop in Algorithm~\ref{alg:implicit_mpa}, which runs over~$E$ epochs.

\begin{wrapfigure}[18]{r}{0.65\textwidth}
\vspace{-0.8\baselineskip}
\begin{minipage}{0.65\textwidth}
\hrule\vspace{2pt}
\captionsetup{font=footnotesize, justification=raggedright, singlelinecheck=false}
\captionof{algorithm}{Training with implicit adversarial maps}\label{alg:implicit_mpa}
\vspace{-7pt}\hrule\vspace{2pt}
\footnotesize
\begin{algorithmic}
\REQUIRE distribution $\hat\PP$; batch size $B$; initial  iterate $\theta$; step sizes $\alpha$, $\eta$; \#~epochs $E$; \#~inner steps $K$; \#~rounds $R$ ($R\!=\!1$: PA;\ \colorbox{projblue}{$R\!>\!1$: MPA})
\FOR{$e\in[E]$} \vspace{-0.1cm}
    \STATE Sample $\{\hat z_i\}_{i=1}^B \stackrel{iid}{\sim} \hat \PP$, and set $z_i \gets \hat z_i $ for all $i\in [B]$
    \FOR{$r\in[R]$}
            \STATE \colorbox{projblue}{$\mathcal Z \gets \{z_i\}_{i=1}^B$}
            \STATE \colorbox{projblue}{$z_i \gets
                   \arg\max_{z\in\mathcal Z} f(\theta\!,z)
                   - \lambda\|z-\hat z_i\|_2^2\ \;\text{for all} \; i\in[B]$}
            \STATE \textbf{do} $K$ \textbf{steps of} $z_i \gets z_i + \eta\left(\nabla_z f(\theta\!,z_i)
                - 2\lambda(z_i-\hat z_i)\right)$ for all $i\in[B]$
    \ENDFOR
    \STATE \colorbox{projblue}{$\mathcal Z \gets \{z_i\}_{i=1}^B$}
    \STATE \colorbox{projblue}{$z_i \gets
           \arg\max_{z\in\mathcal Z} f(\theta\!,z)
           - \lambda\|z-\hat z_i\|_2^2\ \;\text{for all} \; i\in[B]$}
    \STATE $\theta\!\gets\! \mathrm{Proj}_\Theta\!\bigl(\theta
            - \tfrac{\alpha}{B}\!\sum_{i=1}^B\! \nabla_\theta f(\theta\!,z_i)\bigr)$
\ENDFOR
\ENSURE $\theta$
\end{algorithmic}
\vspace{2pt}\hrule
\end{minipage}
\end{wrapfigure}

MPA is reminiscent of multi-start heuristics in non-convex optimization. Indeed, a standard multi-start heuristic launches several local searches from different starting points. Their purpose is to explore different regions (or basins) of the objective landscape. The resulting solutions form a pool of alternatives, from which one can select the best solution. MPA extends this familiar idea from a single optimization problem to a family of $B$ coupled optimization problems. During reassignment, the local maxima discovered by the different problems using PA are shared as candidate starting solutions. Thus, each problem is no longer confined to the basin reached from its own empirical sample. Instead, it can jump to promising basins discovered by other samples.

MPA has several striking properties. First, as the inner loop of Algorithm~\ref{alg:implicit_mpa} is followed by a reassignment step, all adversarial maps generated by MPA are guaranteed to be assignment-stationary and therefore also cyclically monotone on the batch (see \autoref{lem:assign_stat_implies_cm}). Second, if each local optimization phase is run over $K=\infty$ many iterations, which ensures that a stationary point is reached, then the inner loop of Algorithm~\ref{alg:implicit_mpa} is guaranteed to converge after finitely many rounds. 

\begin{theorem}[Finite termination of MPA (informal version of \autoref{thm:mpa_finite})]
\label{thm:mpa_finite_termination_informal}
If~$K=\infty$, then the adversarial samples $\{z_i\}_{i=1}^B$ generated by MPA cease to change after finitely many rounds.
\end{theorem}

Third, MPA guarantees global optimality whenever the batch of empirical samples is sufficiently~large.

\begin{theorem}[Global optimality guarantee of MPA (informal version of \autoref{thm:mpa_global_optimality})]
\label{thm:mpa_global_optimality_informal}
Suppose that the empirical samples are drawn independently from a distribution with an everywhere positive density. Then, for every tolerance~$\epsilon>0$, the implicit adversarial maps generated by MPA are $\epsilon$-optimal in~\eqref{eq:maximize-over-maps} with probability at least~$1-\beta(B,\epsilon)$, where~$\beta(B,\epsilon)>0$ converges to~$0$ as~$B$ grows.
\end{theorem}

Intuitively, \autoref{thm:mpa_global_optimality_informal} holds because MPA resembles a multi-start heuristic and because sufficiently many starting points enable such methods to solve nonconvex optimization problems globally. However, adding a sample to the batch not only enlarges the pool of candidate starting points but also introduces a new global optimization problem, which complicates the proof of \autoref{thm:mpa_global_optimality_informal}. Notably, the global optimality guarantee of MPA holds for {\em every} penalty parameter~$\lambda>0$, thereby removing the restriction~$\lambda>L_{zz}/2$ required for PA to be exact~\citep{sinha2020certifying}. See~\autoref{app:MPA_properties} for further details.

\vspace{-0.3cm}
\section{Explicit Adversarial Maps} 
\vspace{-0.3cm}
\label{sec:explicit_maps}
To overcome the transductive nature of MPA, which constructs an implicit adversarial map restricted to the given batch of training samples and prevents generalization to unseen data, we now construct \emph{explicit} adversarial maps of the form $T_\omega\in\mathcal T$ parametrized by~$\omega\in\Omega$, where~$\Omega$ is a Euclidean space. Restricting~$\mathcal T$ in~\eqref{eq:maximize-over-maps} to this parametric family allows us to work with the parametric objective function
\begin{equation*} \label{eq:explicit_objective}
     \mathcal{L}(\theta;\omega)\coloneqq\EE_{\hat \PP} \left[\, \ell(\theta, \omega,\hat Z) \coloneqq f\bigl(\theta, T_\omega(\hat{Z})\bigr) - \lambda\bigl\|T_\omega(\hat{Z}) - \hat{Z}\bigr\|_2^2\,\right].
\end{equation*}
To jointly learn the parameters~$\theta$ and~$\omega$ of the model and the adversarial map, respectively, we employ an alternating optimization scheme consisting of two nested loops. The inner loop applies $K$ gradient-ascent steps in~$\omega$, while the outer loop applies $E$ projected gradient descent steps in~$\theta$. Here, the outer gradient $\nabla_\theta \max_{\omega\in\Omega} \mathcal L(\theta;\omega)$ is computed approximately via Danskin's theorem, treating the current iterate~$\omega$ as an approximate maximizer. We formalize this procedure in Algorithm~\ref{alg:explicit_maps}. 

\begin{wrapfigure}[10]{r}{0.45\textwidth}   
\vspace{-0.9\baselineskip}
\begin{minipage}{0.45\textwidth}
\hrule\vspace{2pt}
\captionsetup{font=footnotesize, justification=raggedright, singlelinecheck=false}
\captionof{algorithm}{Training with explicit adversarial maps}
\label{alg:explicit_maps}
\vspace{-7pt}\hrule\vspace{2pt}
\footnotesize
\begin{algorithmic}
\REQUIRE distribution $\hat\PP$; batch size $B$; initial iterates $\theta,\omega$; 
         step sizes~$\alpha, \eta$; \#~epochs $E$; \#~inner steps $K$
\FOR{$e \in[E]$} \vspace{-0.1cm}
    \STATE Sample $\{\hat z_i\}_{i=1}^B \stackrel{iid}{\sim} \hat \PP$
    \STATE  \textbf{do} $K$ \textbf{steps of } $\omega \!\gets\!\omega + \frac{\eta}{B}\sum_{i=1}^{B} \nabla_\omega \ell(\theta,\omega,\hat{z}_i)$
    \STATE $z_i \!\gets\! T_{\omega}(\hat{z}_i) $ for all $i\in[B]$
    \STATE $\theta \!\gets\! \mathrm{Proj}_\Theta\!\bigl(\theta \!-\! \tfrac{\alpha}{B}\!\sum_{i=1}^B\! \nabla_\theta f(\theta, z_i)\bigr)$
\ENDFOR
\ENSURE $\theta$
\end{algorithmic}
\vspace{2pt}\hrule
\end{minipage}
\end{wrapfigure}

It is tempting to model the explicit adversarial maps as unconstrained neural networks (\textit{e.g.}, generative adversarial networks~\cite{goodfellow2014generative} or neural transport maps for worst-case generation~\cite{cheng2025worst}). However, such standard networks typically violate cyclical monotonicity, which is a known property of the globally optimal map. It is thus expedient to enforce cyclical monotonicity in the neural network architecture. By a classical result in convex analysis, a map~$T\in\mathcal T$ is cyclically monotone if and only if it lies in the subdifferential of a convex function. 

\begin{theorem}[{Convex-potential representation, \citep[Theorem~2.5.3]{figalli2021invitation}}]\label{prop:convex-potential-representation}
A transport map $T \in\mathcal T$ is cyclically monotone if and only if there exists a convex potential function $\psi: \RR^m \to \RR \cup \{+\infty\}$ such that $T(\hat{z}) \in \partial \psi(\hat{z})$ for every $\hat z\in \RR^m$. In particular, if $\psi$ is differentiable at $\hat{z}$, then $T(\hat{z}) = \nabla \psi(\hat{z})$.
\end{theorem}

\autoref{prop:convex-potential-representation} is the key to enforce cyclical monotonicity in explicit adversarial maps. Specifically, we can define $T_\omega \coloneqq \nabla_z \psi_\omega$, where the potential function $\psi_\omega \colon \mathbb{R}^m \to \mathbb{R}$ is modeled as an Input Convex Neural Network (ICNN)~\citep{amos2017input}. ICNNs are scalar-valued neural networks guaranteed to be convex in their input~$\hat{z}$ (see, {\em e.g.}, \citep[Proposition~1]{amos2017input}). Thus, the map~$T_\omega$ induced by~$\psi_\omega$ is cyclically monotone. We henceforth refer to any such~$T_\omega$ as an \emph{ICNN map}. Note that ICNNs represent universal approximators for convex functions~\citep[Theorem~1]{chen2019optimalcontrolneuralnetworks}. More precisely, any Lipschitz convex potential~$\psi$ on a compact set can be approximated to arbitrary absolute precision by an ICNN. Below we describe the exact architecture and initialization used to train ICNN maps in our experiments.
\looseness=-1

\myparagraph{Architecture.}
Among existing implementations of ICNN-based Brenier maps (see, {\em e.g.}, \citep{makkuva2020optimal,korotin2021neural,vesseron2025neuralimplementationbrenierspolar}), we adopt the convex-potential construction proposed in~\cite{vesseron2025neuralimplementationbrenierspolar} and adapt it to our setting. We model only the scalar potential function $\psi_\omega \colon \mathbb{R}^m \to \mathbb{R}$, and evaluate the ICNN map $T_\omega = \nabla_z \psi_\omega$ on the fly via automatic differentiation. With $L$ hidden layers, $\psi_\omega$ is defined through the recursions \looseness=-1
\begin{equation}
\tag{ICNN}
\label{ICNN Architecture}
\begin{aligned}
    y_1 &= \sigma\!\left(W^z_0\, z + b_0 \right), \\
    y_{\ell+1} &= \sigma\!\left( \exp(W^y_\ell)\, y_\ell + W^z_{\ell}\, z + b_{\ell} \right) \qquad \ell = 1,\dots,L-1, \\
    \psi_\omega(z) &= \exp(w^y_L)^{\!\top} y_L \;+\; \tfrac12\, z^{\top}\!\left( \operatorname{diag}(\delta_L^2) + A_L^{\top} A_L \right) z \;+\; (w^z_L)^{\top} z \;+\; b_L,
\end{aligned}
\end{equation}
where the exponential function $\exp(\cdot)$ is applied component-wise. The full parameter vector is
\[
\omega \;=\; \Big(
\;\{W^z_\ell\}_{\ell=0}^{L-1},\, w^z_L,
\delta_L, A_L,
\;\{W^y_\ell\}_{\ell=1}^{L-1},\, w^y_L,
\;\{b_\ell\}_{\ell=0}^{L}
\Big)
\]
with $W^z_\ell \in \mathbb{R}^{q_\ell \times m}$, $w^z_L, \delta_L \in \mathbb{R}^m$, $A_L \in \mathbb{R}^{r_{\mathrm{out}}\times m}$, $W^y_\ell \in \mathbb{R}^{q_\ell \times q_{\ell-1}}$, $w^y_L \in \mathbb{R}^{q_{L-1}}$, and $b_\ell \in \mathbb{R}^{q_\ell}$, where $q_\ell$ denotes the number of neurons in the $\ell$-th layer ($q_L=1$). By construction, $\psi_\omega$ is convex in $z$ for all~$\omega$. Convexity is preserved across layers through affine transformations $W^z_\ell z + b_\ell$, convex non-decreasing activation functions $\sigma$ (\textit{e.g.}, Softplus), and elementwise positive weights connecting each layer to the next. The quadratic term in the last layer is also convex because $\operatorname{diag}(\delta_L^2) + A_L^\top A_L\succeq 0$. Thus, the induced map $T_\omega = \nabla_z \psi_\omega$ is cyclically monotone for every~$\omega$. The original formulation proposed in~\cite{vesseron2025neuralimplementationbrenierspolar} incorporates positive semidefinite quadratic forms from the input to every layer. However, incorporating such quadratic forms at every layer is computationally expensive. We therefore restrict the quadratic terms to the final readout layer, keeping the hidden-layer passthroughs purely affine. \looseness=-1

Our architecture's capacity is primarily controlled by the last-layer readout rank $r_{\mathrm{out}}$. As we show in \autoref{sec:design_axes}, restricting the quadratic term to the last layer while increasing $r_{\mathrm{out}}$ achieves the best robust performance and substantially reduces runtime compared with incorporating a quadratic term at every layer. Furthermore, ICNNs are notoriously difficult to train because their nonnegative weights can cause standard zero-mean initializations to propagate unstably across layers. We address this challenge with a two-part initialization strategy. First, we use a principled moment-matching scheme to stabilize the variance of the forward and backward signals across depth. Second, our single-quadratic architecture allows us to initialize the overall gradient map $T_\omega$ near the identity, providing an effective warm start. Finally, to maintain training stability, we use an adaptive inner step size that dynamically adjusts to the local curvature of the network. We detail these training mechanisms in \autoref{app:icnn-training}. \looseness=-1

\vspace{-0.4cm}
\section{Experiments}
\vspace{-0.25cm}
\label{sec:experiment}
We benchmark our two new methods, namely the implicit multi-start particle ascent approach (MPA) and the explicit ICNN map approach (ICNN-DRO), against the following six baselines. Empirical Risk Minimization (ERM) solves problem~\eqref{ERM Problem} using SGD. Robust Optimization~(RO) solves~\eqref{prob:robust-opt} as in~\cite{madry2018towards}. Particle Ascent (PA) solves~\eqref{prob:regularized_dro} using Algorithm~\ref{alg:implicit_mpa} with~$R=1$. Neural Network DRO (NN-DRO) solves~\eqref{prob:regularized_dro} using Algorithm~\ref{alg:explicit_maps} with the explicit transport map modeled as a Multi-Layer Perceptron (MLP)~\cite{cheng2025worst} {\em without} enforcing cyclical monotonicity.\footnote{NN-DRO is related to \cite{cheng2025worst}, which models the transport map as an MLP trained {\em indirectly} via least-squares regression on adversarial samples generated by PA. In contrast, NN-DRO trains the MLP {\em directly} using Algorithm~\ref{alg:explicit_maps}.}  Finally, Sinkhorn DRO (SDRO) \cite{wang2021sinkhorn} and Wasserstein-Fisher-Rao DRO (WFR) \cite{xu2025gradientflowsamplerbaseddistributionally} solve an entropy-regularized optimal transport variant of problem~\eqref{prob:regularized_dro} (recovering~\eqref{prob:regularized_dro} as the regularization weight vanishes) using sampling methods. The two methods differ in how they sample from the worst-case distribution, which is known to be a Gibbs distribution \cite[Remark~4]{wang2021sinkhorn}. 
\looseness=-1

Our experiments span three domains of increasing complexity: feature-space logistic regression in~\S~\ref{sec:Adversarial_Logistic Regression}; image-space logistic regression under adversarial attacks, domain shifts, and image corruptions in \S~\ref{sec:Robustness_AA_Full}; and robust control in \S~\ref{sec:exp_rl_main}. Additional results for uncertain least squares and cross-model universality are provided in \S~\ref{sec:exp_robust_ls} and \S~\ref{sec:cross-model-universality}, respectively.\looseness=-1

\vspace{-0.3cm}
\subsection{Adversarial Multi-class Logistic Regression}\label{sec:Adversarial_Logistic Regression}
\vspace{-0.2cm}

\begin{wrapfigure}[17]{r}{0.4\textwidth}
    \vspace{-0.3cm}
    \includegraphics[width=\linewidth]{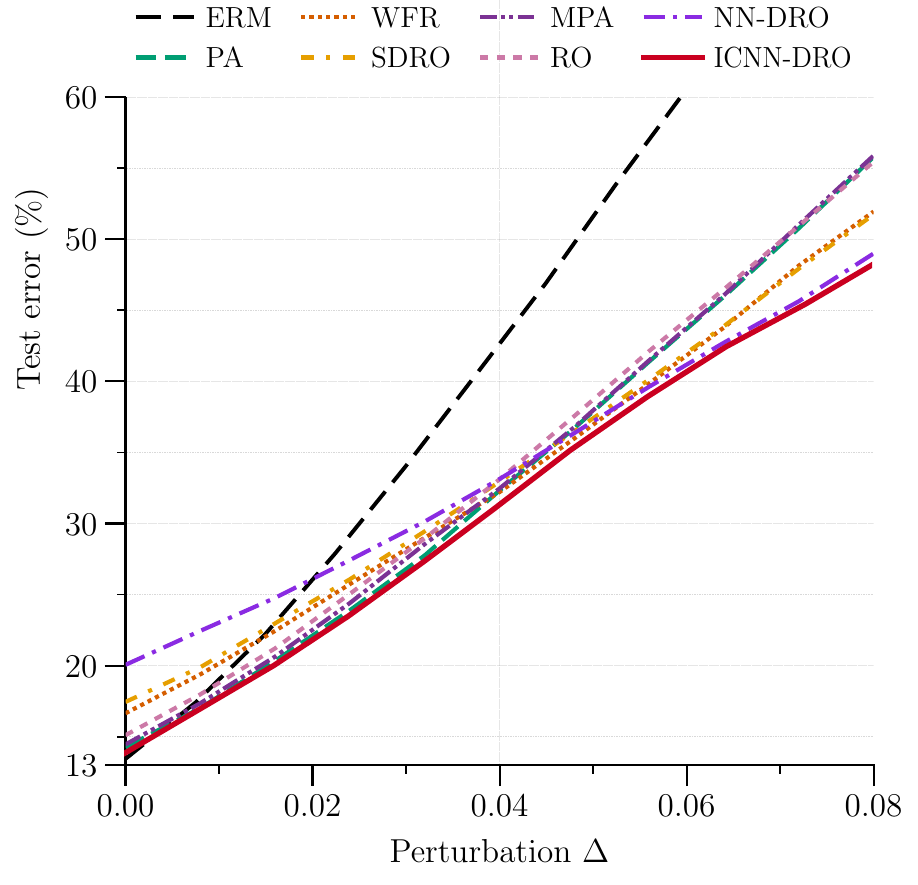}
    \caption{Test error vs.\ perturbation level $\Delta$ on multi-class logistic regression. \looseness=-1}
    \label{fig:Combined_Logistic_regression}
\end{wrapfigure}%

In the first experiment, we follow the setup of~\cite[\S~6.3]{xu2025gradientflowsamplerbaseddistributionally} to assess the adversarial robustness of a multi-class logistic regression model trained on features extracted from the CIFAR-10 dataset~\cite{CIFAR-10}. Specifically, each image is mapped to a feature vector~$x \in \mathbb{R}^{m-1}$ with $m-1=512$ using a frozen ResNet-50 architecture~\cite{he2015deepresiduallearningimage} pre-trained on ImageNet~\cite{deng2009imagenet}. We then train a linear classifier on these features using the negative log-likelihood loss. For a data point $z=(x,y)$, the loss is given by $f(\theta,(x,y)) \coloneqq - y^\top \theta^\top x + \log(\mathbf{1}^\top \exp(\theta^\top x) )$, where $\theta = [\theta_1,\dots,\theta_C] \in \RR^{(m-1)\times C}$ comprises all model parameters, $y \in \{0,1\}^C$ is the one-hot label vector, and~$C$ is the number of classes. The adversary perturbs only the features~$x$, while the labels~$y$ remain unchanged. \looseness=-1

To evaluate robustness under adversarial distribution shift, we subject the test features to an $\ell_2$-constrained Projected Gradient Descent (PGD) attack. We sweep the relative attack budget $\Delta \in [0, 0.08]$, where the maximum perturbation radius is scaled by the average $\ell_2$-norm of the test features. Under this configuration (full training hyperparameters are provided in~\autoref{app:training_details}), \autoref{fig:Combined_Logistic_regression} reports the test error as a function of $\Delta$. Both explicit map approaches perform well, with \methodname slightly outperforming NN-DRO across most perturbation budgets. At extreme perturbation levels, their test errors coincide. Under such large attack budgets, the adversary shifts probability mass so extensively that the optimal transport geometry no longer acts as a limiting factor, making the ICNN's structural advantages less critical. The key difference emerges on clean samples ($\Delta=0$), where NN-DRO degrades significantly. In contrast, \methodname preserves clean accuracy. \autoref{app:monge_audit} additionally reports the Monge gaps of the trained adversarial maps.
\looseness=-1

\vspace{-0.3cm}
\subsection{Robustness to Adversarial and Natural Distribution Shifts}
\vspace{-0.17cm}
\label{sec:Robustness_AA_Full}
Building on the feature-space analysis of~\autoref{sec:Adversarial_Logistic Regression}, we now transition to full image-space adversarial training on the CIFAR-10 dataset. We train a deep neural network classifier $h_\theta: \RR^{m-1} \to \mathbb{R}^C$ against perturbations injected directly into the raw pixel space ($m-1=32\times32\times3$). For a given data point $z = (x, y)$, where $y \in \{0,1\}^C$ is the one-hot encoded ground-truth label, we optimize the network parameters $\theta$ utilizing the multi-class cross-entropy loss. We formulate this objective as
$f(\theta, (x, y)) \coloneqq- y^\top h_\theta(x) + \log(\mathbf{1}^\top \exp(h_\theta(x)))$.
As in the previous experiment, the adversary exclusively perturbs the covariates~$x$ while the labels~$y$ remain fixed. 
We subject the trained models to a comprehensive evaluation spanning worst-case adversarial attacks, image corruptions, and natural domain shifts to assess their generalizability, and additionally report each map's empirical Monge gap and training runtime (implementation details in \autoref{app:Robustness_AA_Full}).
\looseness=-1

\begin{table*}[!t]
  \centering
  \small
  \renewcommand{\arraystretch}{1.15}
  \setlength{\tabcolsep}{3.5pt}
    \caption{%
    Accuracy (\%) under adversarial attacks, natural domain shifts and image corruptions. The left section reports results for the original CIFAR-10 test set (Clean, PGD, and AA), along with the Monge gap. The center columns evaluate accuracy on the unseen domains of CIFAR-10.1 and CIFAR-10.2. The right section (CIFAR-10-C) details performance across five increasing levels of corruption severity, followed by the mean accuracy across all five levels and the runtime. Bold values indicate the best results, and underlined values denote the second best.}
  \vspace{-0.15cm}
  \resizebox{\linewidth}{!}{%
    \begin{tabular}{@{}l@{\hskip 10pt} cccc @{\hskip 4pt}@{\hskip 4pt} cc @{\hskip 6pt} cccccc @{\hskip 6pt} c@{}}
      \toprule
      \multirow{2.5}{*}{Method}
        & \multicolumn{4}{c}{CIFAR-10}
        & \multicolumn{1}{c}{CIFAR-10.1}
        & \multicolumn{1}{c}{CIFAR-10.2}
        & \multicolumn{6}{c}{CIFAR-10-C (Common Corruption)}
        & \multicolumn{1}{c}{Runtime} \\
      \cmidrule(lr){2-5} \cmidrule(lr){6-6} \cmidrule(lr){7-7}
      \cmidrule(lr){8-13} \cmidrule(lr){14-14}
        & Clean
        & PGD
        & AA
        & Monge Gap
        & Clean
        & Clean
        & 1 & 2 & 3 & 4 & 5
        & Avg
        & (s) \\
      \midrule

      ERM
        & \textbf{93.12} & 14.76 & 00.20 & --- & \textbf{87.12} & \textbf{83.14}
        & \textbf{91.50} & \textbf{87.95} & 83.84 & 77.89 & 68.41 & 81.92 & --- \\
      PA
        & 90.08 & 52.41 & 48.86 & 8.01e-4 & 80.27 & 76.16
        & 87.11 & 84.35 & 80.21 & 74.62 & 69.74 & 79.20 & 1320 \\
      RO
        & 90.05 & 53.37 & 50.26 & 8.06e-4 & 80.28 & 76.81
        & 87.15 & 85.06 & 81.75 & 75.35 & 70.12 & 79.89 & \textbf{659} \\
      WFR
        & 91.03 & 53.28 & 50.37 & 5.63e-3 & 81.24 & 77.09
        & 88.60 & 85.34 & 82.93 & 77.56 & 71.09 & 81.10 & 6396 \\
      SDRO
        & 90.00 & 53.12 & 49.61 & 1.37e-3 & 80.20 & 76.48
        & 87.12 & 84.82 & 80.33 & 74.70 & 70.09 & 79.41 & 6241 \\
      NN-DRO
        & 89.46 & 50.34 & 44.27 & 3.25e-3 & 80.01 & 75.71
        & 85.90 & 82.23 & 78.55 & 72.20 & 67.42 & 77.26 & 11051 \\

      \midrule

      MPA
        & 91.39 & \underline{55.31} & \underline{54.47} & \underline{4.71e-4} & 82.23 & 78.45
        & \underline{89.34} & \underline{87.18} & \textbf{84.47} & \textbf{80.64} & \textbf{74.51}
        & \textbf{83.23} & 3612 \\
      ICNN-DRO
        & \underline{91.41} & \textbf{58.11} & \textbf{57.71} & \textbf{1.01e-7} & \underline{82.56}
        & \underline{80.12}
        & 89.20 & 86.79 & \underline{84.01} & \underline{79.92} & \underline{73.95}
        & \underline{82.77} & 4037 \\

      \bottomrule
    \end{tabular}%
  }
  \label{tab:CIFAR-10-CIFAR101-CIFAR102}
\end{table*}

\myparagraph{Robustness to adversarial attacks.} We evaluate on CIFAR-10 under $\ell_2$ attacks with $\varepsilon{=}0.5$ (standard budget~\citep{croce2020reliableevaluationadversarialrobustness, croce2021robustbench}) using PGD and AutoAttack~(AA)~\cite{croce2020reliableevaluationadversarialrobustness}. AA is a suite of diverse attacks designed to expose the gradient obfuscation failure mode~\cite{athalye2018obfuscatedgradientsfalsesense}, in which a model appears robust to PGD only because its gradients are uninformative. The left block of \autoref{tab:CIFAR-10-CIFAR101-CIFAR102} reports the results. \methodname preserves clean accuracy at $91.41\%$ (within $1.71\%$ of ERM) and attains $58.11\%$ against PGD and $57.71\%$ against AA, substantially outperforming PA and WFR. Crucially, the PGD--AA gap is only $0.4\%$ for \methodname compared to $3.55\%$ for PA. This tight agreement certifies that the robustness gain reflects genuine geometric regularity induced by the Brenier-structured map, rather than an artifact of gradient masking~\cite{athalye2018obfuscatedgradientsfalsesense}~(see~\autoref{app:Robustness_AA_Full}). \looseness=-1

\myparagraph{Robustness to unseen domains.} CIFAR-10.1~\cite{CIFAR101} and CIFAR-10.2~\cite{CIFAR102} provide independent, natural covariate-shift benchmarks free of algorithmic corruption. \autoref{tab:CIFAR-10-CIFAR101-CIFAR102} shows that \methodname achieves $82.56\%$ on CIFAR-10.1 and $80.12\%$ on CIFAR-10.2. Our MPA follows very closely behind at $82.23\%$ and $78.45\%$, respectively. Both of our proposed methods consistently outperform the full suite of established robust baselines. Because these datasets were unseen during training, these gains demonstrate genuine generalization rather than overfitting to specific shift patterns.
\looseness=-1

\myparagraph{Robustness to image corruptions.} We evaluate on CIFAR-10-C~\cite{hendrycks2019benchmarkingneuralnetworkrobustness}, which features $15$ algorithmic corruptions~(\eg, noise, blur, weather effects) at five severity levels. As shown in~\autoref{tab:CIFAR-10-CIFAR101-CIFAR102}, \methodname attains $82.77\%$ average accuracy, outperforming PA~($79.20\%$) and WFR~($81.10\%$). Notably, it exceeds ERM~($81.92\%$), showing that robustification via a global transport map improves resilience to corruptions without sacrificing in-distribution accuracy. The gap widens at the highest severity (level~$5$), where \methodname reaches $73.95\%$ against $69.74\%$ for PA. \looseness=-1

\myparagraph{Monge gap and cost.} \methodname attains a Monge gap of $1.01\mathrm{e}\text{-}7$, nearly four orders of magnitude below every baseline, since it is cyclically monotone by construction, while reassignment alone lowers PA's gap by $38\%$. The two smallest gaps belong to the two most robust methods, as predicted by \autoref{prop:wasted_transport}. Enforcing cyclical monotonicity is also inexpensive: \methodname trains in $4{,}037$s, faster than every other DRO baseline while surpassing all of them on every robustness metric, and only the first-order RO is faster, at a cost of $7.45$ AA points.

\vspace{-0.3cm}
\subsection{Robust Control}
\vspace{-0.28cm}
\label{sec:exp_rl_main}
\begin{wrapfigure}[18]{r}{0.4\textwidth}
    \vspace{-0.6em}
    \centering
    \includegraphics[width=\linewidth]{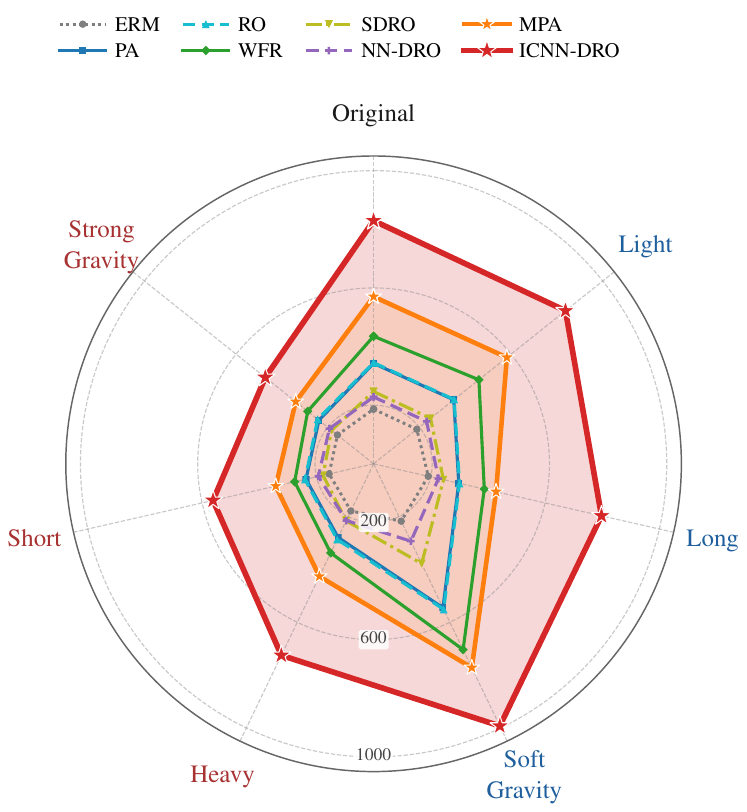}
    \caption{Mean episode length  for $\lambda{=}1$.} 
    \label{fig:radar_RL_main}
\end{wrapfigure}
We evaluate our methods on the cart-pole control task, where the goal is to balance a pole on a cart by applying horizontal actuation. We have parametric physical uncertainty over the pole mass and length, denoted by $z = (m_p, \ell) \in \cZ\subseteq \mathbb{R}^2$, where $\cZ=[0.05, 0.2] \times [0.25, 0.75]$, and draw the training samples uniformly from~$\cZ$. At any time step~$t$, the agent observes a state $s_t\in \mathcal{S} \subseteq \mathbb{R}^4$, capturing the system's position, linear velocity, angle, and angular velocity, and samples an action $a_t \in \mathcal{A}=\{0,1\}$ (corresponding to a left or right impulse) from a parameterized policy $\pi_\theta(a_t \mid s_t)$. The environment then evolves according to known physical dynamics $s_{t+1} = F(s_t, a_t, z)$~\cite{freeman2021brax}. We define the loss function as the negative expected reward $f(\theta, z) \coloneqq  - \mathbb{E}_{\pi_\theta} [ \sum_{t=0}^{H-1} r(s_t, a_t) ]$ over trajectories generated by the dynamics parameterized by $z$.  The adversary perturbs the nominal uncertainty distribution under the regularization parameter $\lambda =1$. We optimize $\theta$ using proximal policy optimization (PPO)~\cite{schulman2017ppo}. \looseness=-1

We evaluate each policy over 1{,}000 deterministic rollouts, each with maximum horizon $H=1{,}000$, and report the mean episode length across these trials in~\autoref{fig:radar_RL_main}. The figure reports mean episode lengths only; a more detailed table including standard errors, together with a sensitivity analysis over
$\lambda \in \{0.1, 0.5\}$ and additional training details, is provided in~\autoref{app:exp_neural_transport_rl}. The evaluation grid spans three categories to test both interpolation and extrapolation. The original configuration captures the nominal physics at the center of the training box. The evaluation also includes Easier Environments (Light, Long, Soft Gravity) and Harder Environments (Heavy, Short, Strong Gravity). Crucially, the harder variants lie outside the training uncertainty box, so they probe extrapolation rather than interpolation of the learned robustness. As shown in \autoref{fig:radar_RL_main},  \methodname achieves the best episode length on every variant, with especially large gains on the extrapolative Heavy, Short, and Strong gravity shifts. MPA is the consistent runner-up and substantially improves over PA, while WFR is the strongest of the remaining baselines but degrades more noticeably on the harder variants. These results suggest that improving the transport geometry in the inner maximization already yields a substantial benefit, and that learning a single global Brenier-structured adversarial map provides an additional robustness gain beyond batch-wise reassignment alone. \looseness=-1

\vspace{-0.3cm}
\section{Conclusions}
\vspace{-0.2cm}
\label{sec:conclusion}
We studied regularized Wasserstein DRO through the geometry of optimal transport and showed that cyclical monotonicity is the key structural property missing from standard per-sample adversarial training. This yields two remedies: MPA, which restores the transport structure implicitly through ascent and reassignment, and ICNN-based adversaries, which enforce it explicitly by parameterizing a Brenier map. Across regression, adversarial training, and robust control, both approaches improve robustness and out-of-distribution generalization over standard baselines.

\myparagraph{Limitations.} The main limitation is computational cost. MPA requires an $\cO(B^2)$ reassignment step per round, while the ICNN adversary introduces an auxiliary network, making each training iteration more expensive than standard single-step adversarial updates~\citep{goodfellow2014explaining}. However, this increased cost is offset by significant robustness benefits.

\newpage
\paragraph{Acknowledgments}
This work was supported as a part of NCCR Automation, a National Centre of Competence in Research, funded by the Swiss National Science Foundation (grant number 51NF40\_225155). We thank Yves Rychener for insightful discussions in the early stages of this work.

\bibliographystyle{plainnat}
\bibliography{ref}

\newpage
\clearpage
\appendix
\onecolumn
\addtocontents{toc}{\protect\setcounter{tocdepth}{2}}
{\centering\LARGE\bfseries Appendix\par}
\vspace{1.5em}
\tableofcontents

\section{Related Work}
The literature most closely related to our work falls naturally into three groups.

\myparagraph{Optimal transport and ICNNs.} 
For any~$\PP,\hat\PP\in\cP$, one can use Kantorovich duality to show that
\begin{equation}
    \label{eq:Kantorovich-dual}
    \frac{1}{2} \wasserstein^2(\mathbb{P}, \hat{\mathbb{P}}) = \frac{1}{2}\EE_{\hat{\mathbb P}}\bigl[\|\hat Z\|^2_2\bigr] + \frac{1}{2}\EE_{\mathbb P}\bigl[\|Z\|^2_2\bigr] - \min_{\psi\ \mathrm{convex}} \left\{ \EE_{\hat{\mathbb P}}\bigl[\psi(\hat Z)\bigr] + \EE_{\mathbb P}\bigl[\psi^*(Z)\bigr] \right\},
\end{equation}
where $\psi^*$ denotes the convex conjugate of~$\psi$. The minimum in~\eqref{eq:Kantorovich-dual} is always attained. Moreover, if $\hat\PP$ is absolutely continuous with respect to Lebesgue measure, then for every minimizer~$\psi$, $\nabla\psi$ coincides $\hat\PP$-almost everywhere with the unique Brenier map from~$\hat\PP$ to~$\PP$ \cite{Brenier1991}. These observations underlie a rich line of work that parameterizes the convex potential~$\psi$ using an ICNN~\citep{amos2017input} and optimizes its parameters, with the conjugate~$\psi^*$ evaluated through an inner maximization~\citep{makkuva2020optimal, korotin2019wasserstein, korotin2021neural, bunne2023supervisedtrainingconditionalmonge, vesseron2025neuralimplementationbrenierspolar}. Our methods in \autoref{sec:explicit_maps} differ from the ICNN-based solution methods for optimal transport problems in several important respects. (i) The target distribution~$\PP$ is not fixed as it solves the inner problem in~\eqref{prob:regularized_dro} and hence varies with~$\theta$. (ii) Our reference distribution~$\hat\PP$ is not required to be absolutely continuous (such as a discrete empirical distribution), and thus Brenier's theorem does not apply. (iii)~Our problem~\eqref{eq:maximize-over-maps} is nonconvex in~$T$ and thus also in the potential~$\psi$, in contrast to the minimization problem in~\eqref{eq:Kantorovich-dual}. We show that restricting the adversary to gradients of convex potentials acts as a structural safeguard against selecting poor local optima---a role absent in the context of solving classical optimal transport problems. (iv) We use ICNNs directly to solve a primal optimization problem over transport maps (see~\eqref{eq:maximize-over-maps}) instead of a dual problem over potential functions. (v) For the optimal transport problem underlying~\eqref{eq:Kantorovich-dual}, cyclical monotonicity of a map is both necessary and sufficient for optimality~\cite[Theorem~2.12]{villani2021topics}, whereas in our setting it is only necessary.
\looseness=-1

\myparagraph{Adversarial training with DRO.} For adversarial training based on problem~\eqref{prob:regularized_dro}, the closest approach to ours is~\citep{sinha2020certifying}, which evaluates the adversarial risk by running independent gradient ascent from each training sample. Other approaches quantify the mismatch between the nominal and target distributions using an entropic optimal transport discrepancy. In this setting, the worst-case distribution is a Gibbs measure~\citep{wang2021sinkhorn}, allowing the adversarial risk to be evaluated by sampling from this measure. An alternative particle-based flow for sampling from the Gibbs measure is proposed in~\citep{xu2025gradientflowsamplerbaseddistributionally}. In all of these approaches, however, adversarial samples are generated independently for each training sample, without enforcing cyclical monotonicity. MPA takes a fundamentally different approach in that it reassigns particles across the batch to enforce this structural property.

\myparagraph{Adversarial transport maps.} The idea of recasting the adversary's problem as an optimization problem over transport maps was first explored in~\citep{sinha2020certifying, xu2024flow, cheng2025worst}. When $\hat{\mathbb P}$ admits a density, \citep{xu2024flow, cheng2025worst} justify this formulation via Brenier's theorem for Monge's optimal transport problem. In~\citep{xu2024flow}, the optimal adversarial map is constructed by training and composing a sequence of neural networks that push the nominal distribution forward to the worst-case distribution. In~\citep{cheng2025worst}, by contrast, the transport map is represented by a single neural network trained via least-squares regression on adversarial samples generated by particle ascent~\citep{sinha2020certifying}. Both approaches model the adversarial map using an \emph{unconstrained} neural network and require $\hat\PP$ to admit a density. We instead constrain the map to be the gradient of an ICNN, making it cyclically monotone by construction, while imposing no density assumption that would exclude an empirical nominal distribution. Moreover, none of these prior works recognize that enforcing cyclical monotonicity can steer the adversarial map away from local maxima, thereby improving the quality of the adversarial samples and, ultimately, the effectiveness of adversarial training.
\looseness=-1

\section{Technical Background Results}
\label{app:Technical_Background}
\paragraph{Blanket assumptions.}
Throughout~\autoref{app:Technical_Background}, we fix $\theta\in\Theta$ and $\lambda>0$, and we assume that $f(\theta,z)$ is $L_z$-Lipschitz continuous in~$z$. We use $\mathcal{P}_0(\RR^m)$ to denote the family of all Borel probability measures on $\RR^m$, which contains the family $\mathcal{P}(\RR^m)$ of all Borel probability measures with finite second moments as a subset. Furthermore, we fix an arbitrary reference measure $\hat{\PP}\in\mathcal{P}(\RR^m)$.

In the remainder of this section, we show that the adversarial map $T^\star$ is well-defined and that the maximum in the definition of $F(\theta)$ is attained by $T^\star_{\#}\hat{\PP}$, which is guaranteed to have finite second moments. The main text defines $F(\theta)$ as a maximum over $\mathcal{P}(\RR^m)$, which presupposes the existence of a maximizer therein. As establishing this is the goal of the present appendix, we initially replace the maximum with a supremum and reintroduce the maximum only once attainment is shown.

\begin{proposition}[Measurability and well-definedness of \texorpdfstring{$T^\star$}{T*}]\label{prop:measurability_Tstar}
The set-valued mapping $S(\hat{z})$ in~\eqref{eq:T_star_argmax} is non-empty and closed-valued for all $\hat{z} \in \mathbb{R}^m$. Furthermore, there exists a Borel-measurable map $T^\star : \mathbb{R}^m \to \mathbb{R}^m$ such that $T^\star(\hat{z}) \in S(\hat{z})$ for all $\hat{z} \in \mathbb{R}^m$.
\end{proposition} 
\begin{proof}
    Define $g(\hat{z}, z) \coloneqq f(\theta, z) - \lambda \|z - \hat{z}\|_2^2$ to simplify notation. Since $f(\theta, z)$ is Lipschitz continuous in $z$, the quadratic penalty dominates as $\|z\|_2 \to \infty$, and thus the continuous function~$-g(\hat{z}, z)$ is coercive in~$z$. Therefore, $g(\hat{z}, z)$ attains its maximum over $z$, which in turn implies that the set of maximizers $S(\hat{z})$ is non-empty and closed for every $\hat{z}\in\RR^m$. To establish the measurability of the set-valued mapping~$S(\hat z)$, note first that $g(\hat{z}, z)$ is continuous throughout $\mathbb{R}^m \times \mathbb{R}^m$. This implies via \cite[Examples~14.15~\&~14.29]{rockafellar1998variational} that $g(\hat{z}, z)$ is a normal integrand. By~\cite[Theorem~14.37]{rockafellar1998variational}, the argmax mapping $S(\hat{z})=\arg\max_{z \in \RR^m} g(\hat{z}, z)$ thus constitutes a Borel-measurable set-valued mapping. We can therefore invoke~\cite[Corollary~14.6]{rockafellar1998variational} to conclude that $S$ admits a measurable selection. That is, there exists a Borel-measurable function $T^\star : \RR^m \to \RR^m$ with $T^\star(\hat{z}) \in S(\hat{z})$ for all $\hat{z} \in \RR^m$.
\end{proof}

We now show that the adversarial risk $F(\theta)$ can be expressed as in~\eqref{eq:regularized_dro-dual}. This reformulation is reminiscent of duality results in Wasserstein DRO (see, {\em e.g.}, \citep{blanchet2017quantifyingdistributionalmodelrisk, mohajerin2018data, zhang2022simple}) and appeared first in \cite[Proposition~1]{sinha2020certifying}. We provide a short alternative proof to keep this paper self-contained.

\begin{proposition}[Adversarial risk]
\label{prop:inner_deterministic_optimizer}
The adversarial risk satisfies~\eqref{eq:regularized_dro-dual}, the measure $\PP^{\star}{\coloneqq} T^{\star}_\#\hat{\PP}$ solves the inner maximization problem in~\eqref{prob:regularized_dro}, and~$\gamma^\star {\coloneq} (T^\star, \mathrm{id})_{\#} \hat{\mathbb{P}}$ is an optimal coupling of~$\hat{\PP}$ and~$\PP^{\star}$.
\end{proposition}

\begin{proof}
By rewriting the squared Wasserstein distance appearing in the definition of~$F(\theta)$ as an optimal transport problem over couplings $\gamma$, and by eliminating the candidate measure $\PP$ through its representation as the first marginal of~$\gamma$, we obtain
\begin{equation}
\label{eq:V_coupling}
F(\theta)
\;=\;
\sup_{\substack{\gamma\in\mathcal{P}(\RR^m\times\RR^m)\\ (\pi_2)_\#\gamma=\hat{\PP}}}
\int_{\RR^m\times\RR^m}\!\!
\Bigl(f(\theta,z)-\lambda\|z-\hat z\|_2^2\Bigr)\,\diff\gamma(z,\hat z).
\end{equation}
Note that the constraint $\PP \in \mathcal{P}(\RR^m)$ in the definition of~$F(\theta)$, which requires~$\mathbb P$ to have finite second moments, translates to the requirement $\gamma \in \mathcal{P}(\RR^m \times \RR^m)$ in~\eqref{eq:V_coupling}, which requires the coupling~$\gamma$ to have finite second moments.  Indeed, whenever $\hat{\PP}$ has finite second moments, $\PP$ has finite second moments if and only if every coupling of $\PP$ and $\hat{\PP}$ does. This follows from the relation
\[
    \EE_{\gamma} \left[\|(Z,\hat{Z})\|_2^2 \right] = \EE_{(\pi_1)_\#\gamma}\left[\|Z\|_2^2 \right] + \EE_{\hat{\PP}}\left[\|\hat{Z}\|_2^2 \right],
\]
where $\pi_1:\RR^m\times\RR^m\to\RR^m$  is the coordinate projection defined by $\pi_1(z,\hat z)=z$. Next, we show that couplings with infinite second moments are strictly suboptimal in~\eqref{eq:V_coupling} and that $\mathcal{P}(\mathbb{R}^m \times \mathbb{R}^m)$ can be relaxed to $\mathcal{P}_0(\mathbb{R}^m \times \mathbb{R}^m)$. To this end, we first prove
that the optimal value of~\eqref{eq:V_coupling} is finite. 

We can start by bounding
the integrand in~\eqref{eq:V_coupling}. As $f(\theta,z)$ is $L_z$-Lipschitz continuous in~$z$, we have
\begin{align*}
    f(\theta,z) - \lambda\|z - \hat z\|_2^2 
    &\leq f(\theta,\hat z) + L_z\|z - \hat z\|_2  - \lambda\|z - \hat z\|_2^2 \\
    &= f(\theta,\hat z) + \frac{L_z^2}{4\lambda} -\lambda\left(\|z - \hat z\|_2 - \frac{L_z}{2\lambda}\right)^2
\end{align*}
for all $z, \hat z \in \RR^m$. Using the elementary inequality $(a-b)^2 \geq \frac{1}{2}a^2 - b^2$ for any $a,b\in\mathbb R$, we can further loosen this bound to isolate the quadratic transport cost. We thus obtain
\begin{align*}
    f(\theta,z) - \lambda\|z - \hat z\|_2^2 
    &\leq f(\theta,\hat z) + \frac{L_z^2}{2\lambda} -\frac{\lambda}{2}\|z - \hat z\|_2^2 \\
    &\leq f(\theta,0) + L_z\|\hat z\|_2 + \frac{L_z^2}{2\lambda} -\frac{\lambda}{2}\|z - \hat z\|_2^2.
\end{align*}
Integrating both sides with respect to~$\gamma$ and recalling that the second marginal of~$\gamma$ equals~$\hat\PP$ yields
\begin{subequations}
\begin{align}
    \int\!\bigl(f(\theta,z) - \lambda\|z - \hat z\|_2^2\bigr)\,\diff\gamma (z,\hat z) 
    &\leq 
    f(\theta,0) + \frac{L_z^2}{2\lambda} + L_z\EE_{\hat\PP}\bigl[\|\hat Z\|_2\bigr] - \frac{\lambda}{2}\!\int\!\|z - \hat z\|_2^2\,\diff\gamma(z,\hat z) \label{eq:upper_bound_1} \\
    &\leq f(\theta,0) + \frac{L_z^2}{2\lambda} + L_z\EE_{\hat\PP}\bigl[\|\hat Z\|_2\bigr]. \label{eq:upper_bound_2}
\end{align}
\end{subequations}

As~$\hat{\PP}$ has finite second moments, the inequality~\eqref{eq:upper_bound_2} shows that  the integral is bounded above by a finite constant uniformly across all couplings~$\gamma$. Thus, we find~$F(\theta) < \infty$. Conversely, the trivial coupling $\gamma = (\mathrm{id},\mathrm{id})_\#\hat\PP$ is feasible in~\eqref{eq:V_coupling} and attains the finite objective value $\EE_{\hat{\PP}}[f(\theta,\hat{Z})]$. Consequently, we have~$F(\theta)>-\infty$. In summary, we may conclude that~$F(\theta)$ is indeed finite.

By~\eqref{eq:upper_bound_1}, the objective function in~\eqref{eq:V_coupling} evaluates to~$-\infty$ for any coupling~$\gamma$ with infinite transport cost. Moreover, $\gamma$ has infinite transport cost if and only if it has infinite second moments because
\[
	\int \frac{1}{2}\|z-\hat{z}\|_2^2 \,\diff\gamma (z,\hat z) \leq  \int\left( \|z\|_2^2 + \|\hat{z}\|_2^2 \right)\diff\gamma (z,\hat z)  \leq \int\left( 2\|z-\hat{z}\|_2^2 + 4\|\hat{z}\|_2^2\right) \diff\gamma (z,\hat z) .
\]
We may therefore relax $\mathcal{P}(\mathbb{R}^m \times \mathbb{R}^m)$ to $\mathcal{P}_0(\mathbb{R}^m \times \mathbb{R}^m)$ without increasing the supremum in~\eqref{eq:V_coupling}.

Next, we can decompose any feasible coupling $\gamma\in \mathcal{P}_0(\mathbb{R}^m \times \mathbb{R}^m)$ with respect to its second marginal. Indeed, by the disintegration theorem on Polish spaces \citep[Theorem~5.3.1] {ambrosio2005gradient}, $\gamma$ decomposes as
\begin{equation}
\label{eq:disintegration}
\gamma(\diff z,\diff\hat z)
\;=\;
\nu_{\hat z}(\diff z)\otimes \hat{\PP}(\diff\hat z)
\end{equation}
for a $\hat{\PP}$-a.e.\ uniquely defined Borel family $(\nu_{\hat z})_{\hat z\in\RR^m}\subseteq\mathcal{P}_0(\RR^m)$. Conversely, any such family defines a feasible $\gamma$ via~\eqref{eq:disintegration}. Substituting \eqref{eq:disintegration} into the exact relaxation of~\eqref{eq:V_coupling} derived above and applying the Fubini-Tonelli theorem, which is justified in view of~\eqref{eq:upper_bound_2}, we obtain
\begin{equation*}
F(\theta)
\;=\;
\sup_{(\nu_{\hat z})}
\int_{\RR^m}\!\!\left[\int_{\RR^m}\!\bigl(f(\theta,z)-\lambda\|z-\hat z\|_2^2\bigr)\,\diff\nu_{\hat z}(z)\right]\diff\hat{\PP}(\hat z).
\end{equation*}
For each fixed $\hat z$, the inner integral admits the upper bound
\[
\int_{\RR^m}\bigl(f(\theta,z)-\lambda\|z-\hat z\|_2^2\bigr)\,\diff\nu_{\hat z}(z) \;\le\; \max_{z\in\RR^m}\;\bigl\{f(\theta,z)-\lambda\|z-\hat z\|_2^2\bigr\},
\]
where equality holds whenever $\nu_{\hat z}$ is supported on $S(\hat z)$ defined in~\eqref{eq:T_star_argmax}. Note that the right-hand side of the above inequality is a pointwise maximum of continuous functions and thus Borel-measurable. By \autoref{prop:measurability_Tstar}, there exists a Borel selector $T^\star$ such that $T^\star(\hat z)$ is optimal on the right-hand side for every~$\hat z\in\mathbb R^m$. Clearly, the choice $\nu_{\hat z}^\star\coloneqq\delta_{T^\star(\hat z)}$ defines a feasible disintegration. Thus, we find
\[
F(\theta)
\;=\;
\mathbb{E}_{\hat{\mathbb{P}}} \Big[\max_{z \in \RR^m} \bigl\{ f(\theta, z) - \lambda \|z - \hat{Z}\|_2^2 \bigr\}\Big],
\]
which establishes~\eqref{eq:regularized_dro-dual}. This shows that $\gamma^\star \coloneqq (\mathrm{id},T^\star)_\#\hat{\PP}$ solves~\eqref{eq:V_coupling} and that~$\PP^\star = (\pi_1)_\#\gamma^\star = T^\star_\#\hat{\PP}$ solves the inner maximization problem in~\eqref{prob:regularized_dro}. We emphasize that $\PP^\star \in \mathcal{P}(\RR^m)$ has finite second moments for otherwise~$\gamma^\star$ would have infinite second moments and an objective value of~$-\infty$, thus contradicting its optimality. It remains to be shown that $\gamma^\star$ is an optimal coupling between~$\PP^\star$ and~$\hat{\PP}$. Suppose for contradiction that there is $\tilde\gamma\in\Gamma(\PP^\star,\hat{\PP})$ with
\[
\int_{\RR^m\times\RR^m}\|z- \hat z\|_2^2\,\diff\tilde \gamma(z,\hat z)
\;<\;
\int_{\RR^m\times\RR^m}\|z-\hat z\|_2^2\,\diff\gamma^\star(z, \hat z).
\]
Since $\tilde\gamma$ and $\gamma^\star$ share the first marginal $\PP^\star$, they yield identical loss integrals, that is,
\[
\int f(\theta,z)\,\diff\tilde\gamma(z, \hat z)
\;=\;
\int f(\theta,z)\,\diff\gamma^\star(z, \hat z)
\;=\;
\EE_{\PP^\star}[f(\theta,Z)].
\]
Consequently, we find
\[
\int\bigl(f(\theta,z)-\lambda\|z-\hat z\|_2^2\bigr)\,\diff\tilde\gamma(z,\hat z)
\;>\;
\int\bigl(f(\theta,z)-\lambda\|z-\hat z\|_2^2\bigr)\,\diff\gamma^\star(z,\hat z)
\;=\;
F(\theta),
\]
which contradicts~\eqref{eq:V_coupling}. We may thus conclude that
\[
\wasserstein^2(\PP^\star,\hat{\PP})
\;=\;
\int_{\RR^m\times\RR^m}\|z-\hat z\|_2^2\,\diff\gamma^\star(z, \hat z)
\;=\;
\EE_{\hat{\PP}}\!\bigl[\|T^\star(\hat Z)-\hat Z\|_2^2\bigr],
\]
which confirms that $\gamma^\star$ is indeed an optimal coupling between~$\PP^\star$ and~$\hat{\PP}$.
\end{proof}

The proof of \autoref{prop:inner_deterministic_optimizer} shows that $F(\theta)$ is guaranteed to be finite.

\begin{remark}[Empirical measure]
\label{rem:empirical_case}
When $\hat{\PP}=\frac{1}{N}\sum_{i=1}^N\delta_{\hat z_i}$ is the empirical measure corresponding to the training samples $\hat z_1,\dots,\hat z_N\in\RR^m$, the disintegration~\eqref{eq:disintegration} reduces to the elementary decomposition
\[
\gamma(\diff z,\diff\hat z)
\;=\;
\frac{1}{N}\sum_{i=1}^N \nu_i(\diff z) \otimes \delta_{\hat z_i}(\diff\hat z)
\quad \text{with} \quad \nu_i\in\mathcal{P}_0(\RR^m) \quad \forall i\in[N],
\]
in which case no measurable selection is required. Specifically, it suffices to pick a maximizer
\[
z_i^\star
\;\in\;
\arg\max_{z\in\RR^m}\bigl\{f(\theta,z)-\lambda\|z-\hat z_i\|_2^2\bigr\}
\]
for each $i\in[N]$. \autoref{prop:inner_deterministic_optimizer} then reduces to the simpler statement
\[
\sup_{\PP\in\mathcal{P}(\RR^m)}
\Bigl\{\EE_{\PP}[f(\theta,Z)] - \lambda\,\wasserstein^2(\PP,\hat{\PP})\Bigr\}
\;=\;
\frac{1}{N}\sum_{i=1}^N\sup_{z\in\RR^m}\bigl\{f(\theta,z)-\lambda\|z-\hat z_i\|_2^2\bigr\},
\]
and an optimizer of the inner maximization problem in~\eqref{prob:regularized_dro} is given by $\PP^\star=\frac{1}{N}\sum_{i=1}^N\delta_{z_i^\star}$.
\end{remark}

\section{Proofs}
\label{app:ot_structure}

\begin{proof}[Proof of~\autoref{prop:convex_bound_inexact}]
We first show that $g_t$ is an $\varepsilon(\theta_t, T_t)$-subgradient of $F$ at $\theta_t$. As $f(\theta,z)$ is convex in $\theta$, we may use the first-order condition of convexity to obtain
\[
f(\theta,T_t(\hat z))-\lambda\|T_t(\hat z)-\hat z\|_2^2
\;\ge\;
f(\theta_t,T_t(\hat z))-\lambda\|T_t(\hat z)-\hat z\|_2^2
+\langle \nabla_{\theta}f(\theta_t,T_t(\hat z)),\,\theta-\theta_t\rangle
\]
for any fixed $\theta\in\Theta$ and all $\hat z \in \RR^m$. Taking expectations with respect to $\hat{\mathbb{P}}$ on both sides implies that
\[
\EE_{\hat{\PP}}\!\bigl[f(\theta,T_t(\hat Z))-\lambda\|T_t(\hat Z)-\hat Z\|_2^2\bigr]
\;\ge\;
\EE_{\hat{\PP}}\!\bigl[f(\theta_t,T_t(\hat Z))-\lambda\|T_t(\hat Z)-\hat Z\|_2^2\bigr]
+\langle g_t,\,\theta-\theta_t\rangle.
\]
By the definition of the adversarial risk, the left-hand side is no larger than $F(\theta)$, 
and by the definition of the adversarial error, the first term on the right-hand side equals $F(\theta_t)-\varepsilon(\theta_t,T_t)$. Hence, we find
\begin{equation}
\label{eq:eps_subgrad_proof}
F(\theta)\;\ge\;F(\theta_t)+\langle g_t,\theta-\theta_t\rangle-\varepsilon(\theta_t,T_t).
\end{equation}
As $\theta\in\Theta$ was chosen freely, this shows that $g_t$ is indeed an $\varepsilon(\theta_t, T_t)$-subgradient of $F$ at $\theta_t$.

Next, we use the nonexpansiveness of the Euclidean projection operator `$\operatorname{Proj}_\Theta(\cdot)$' and the update rule $\theta_{t+1}=\operatorname{Proj}_\Theta (\theta_t - \alpha\, g_t)$ of the projected gradient descent algorithm to conclude that
\begin{equation}
\label{eq:improvement}
\|\theta_{t+1}-\theta^\star\|_2^2
\;\le\;\|\theta_{t}-\alpha\, g_t-\theta^\star\|_2^2
\;\le\;
\|\theta_t-\theta^\star\|_2^2
-2\alpha\langle g_t,\theta_t-\theta^\star\rangle
+\alpha^2\|g_t\|_2^2.
\end{equation}
As $g_t$ is defined at the expected value of a random variable whose norm is uniformly bounded by $L_\theta$, it is clear that
$\|g_t\|_2\le L_\theta$. Combining~\eqref{eq:eps_subgrad_proof} 
for $\theta =\theta^\star$ with~\eqref{eq:improvement} and rearranging terms we then find
\[
F(\theta_t)-F(\theta^\star)
\;\le\;
\frac{1}{2\alpha} \left( \|\theta_t-\theta^\star\|_2^2-\|\theta_{t+1}-\theta^\star\|_2^2\right)
+\frac{\alpha L_\theta^2}{2}
+ \varepsilon(\theta_t,T_t).
\]
Summing over $t=0,\dots,H-1$ and discarding the last term $-\|\theta_H-\theta^\star\|_2^2/(2\alpha)\leq 0$ in the telescoping sum then yields the estimate
\[
\frac{1}{H}\sum_{t=0}^{H-1}\bigl(F(\theta_t)-F(\theta^\star)\bigr)
\;\le\;
\frac{\|\theta_0-\theta^\star\|_2^2}{2\alpha H}
+\frac{\alpha L_\theta^2}{2}
+\frac{1}{H}\sum_{t=0}^{H-1}\varepsilon(\theta_t,T_t).
\]
The step size $\alpha=\|\theta_0-\theta^\star\|_2/(L_\theta\sqrt{H})$ 
balances the first two terms, ensuring that their sum reduces to 
$L_\theta\|\theta_0-\theta^\star\|_2/\sqrt{H}$. Since $F(\theta)$ is convex
(as the pointwise supremum of convex functions), Jensen's inequality implies that $F(\bar\theta_H)\le \frac{1}{H}\sum_{t=0}^{H-1}F(\theta_t)$, 
which yields the claimed bound.
\end{proof}

\begin{proof}[Proof of~\autoref{prop:wasted_transport}]
    Since $F(\theta)$ is finite by the proof of~\autoref{prop:inner_deterministic_optimizer}, we may assume without loss of generality that $T_{\#}\hat{\PP}$ has finite second moments. Otherwise, the same proof yields $\varepsilon(\theta,T)=\infty$, in which case the claim holds trivially. Hence, $T_{\#}\hat{\PP}$ is feasible in the inner maximization problem in~\eqref{prob:regularized_dro}, and we have
    \[
        F(\theta) \;\geq\; \EE_{T_{\#}\hat{\PP}}[f(\theta,Z)] - \lambda\, \wasserstein^2(T_{\#}\hat{\PP},\, \hat{\PP}).
    \]
    Subtracting $\EE_{\hat{\PP}}\bigl[f(\theta, T(\hat Z)) - \lambda\|T(\hat Z) - \hat Z\|_2^2\bigr]$ from both sides and using the measure-theoretic change of variables formula $\EE_{T_{\#}\hat{\PP}}[f(\theta, Z)] = \EE_{\hat{\PP}}[f(\theta, T(\hat Z))]$, we obtain
    \begin{align*}
        \varepsilon(\theta, T) 
        &\geq \lambda\,\EE_{\hat{\PP}}\bigl[\|T(\hat Z) - \hat Z\|_2^2\bigr] - \lambda\, \wasserstein^2(T_{\#}\hat{\PP},\, \hat{\PP}) = \lambda\, \cM_{\hat{\PP}}(T),
    \end{align*}
    where the equality follows from the definition of the Monge gap. This concludes the proof.
\end{proof}

\begin{remark}[Monge gap of the empirical distribution]
    When $\hat{\PP} = \tfrac{1}{N}\sum_{i=1}^{N} \delta_{\hat z_i}$ is the empirical distribution over $N$ training samples, the Monge gap admits the more explicit representation
    \begin{equation*}
        \cM_{\hat\PP}(T)
        \;=\;
        \max_{\sigma \in \mathcal{S}_N}\,
        \frac{2}{N}\sum_{i=1}^{N}
        \bigl\langle
            T(\hat z_{\sigma(i)}) - T(\hat z_i),\, \hat z_i
        \bigr\rangle,
    \end{equation*}
    where $\mathcal{S}_N$ denotes the family of all permutations of $N$ elements~\citep{uscidda2023monge}. By \autoref{prop:wasted_transport}, any lack of cyclical monotonicity of $T$ on the support of $\hat{\PP}$ thus leads to a strictly positive lower bound on~$\varepsilon(\theta, T)$.
\end{remark}

\begin{proof}[Proof of~\autoref{prop:optimal_map_cm}]
    Select any $k \in \NN$ and any collection of points $\{\hat{z}_i\}_{i=1}^k$ in~$\RR^m$. By the definition of the optimal adversarial map~$T^\star$, we have
    \[
        T^\star(\hat z_i) \in \argmax_{z \in \RR^m} \left\{ f(\theta, z) - \lambda \|z - \hat{z}_i\|_2^2 \right\}\qquad \forall i\in[k].
    \]
    Therefore, the following inequality holds for any permutation $\sigma \in \cS_k$:
    \[
        f(\theta, T^\star(\hat z_i)) - \lambda \|T^\star(\hat z_i) - \hat z_i\|_2^2
        \;\geq\;
        f(\theta, T^\star(\hat z_{\sigma(i)})) - \lambda \|T^\star(\hat z_{\sigma(i)}) - \hat z_i\|_2^2
        \qquad \forall i \in [k].
    \]
    Summing over $i\in[k]$, the terms involving $f$ on both sides of the inequality coincide and thus cancel. As the points $\{\hat z_i\}_{i=1}^{k}$ were chosen arbitrary, this implies that $T^\star$ is cyclically monotone on~$\RR^m$. In particular, $T^\star$ is cyclically monotone on the support of~$\hat \PP$, which implies that $\cM_{\hat \PP}(T^\star) = 0$.
\end{proof}

\begin{proof}[Proof of~\autoref{prop:pa_monotonicity}]
    Since $m = 1$, the map $T_{\mathrm{pa}}(t,\cdot)$ is cyclical monotone if and only if it is monotonically non-decreasing~\citep[Chapter~2]{santambrogio2015optimal}. We therefore prove that the PA map $T_{\mathrm{pa}}(t,\cdot)$ is monotonically non-decreasing both in the continuous-time and discrete-time settings. 

    \textbf{Continuous-time PA.} By the Picard-Lindel\"of theorem, the defining ODE of the continuous-time PA map has a unique solution that depends continuously on the initial conditions. This implies that $T_{\mathrm{pa}}(t,\hat z)$ is continuously differentiable in~$t$ and continuous in~$\hat z$. In the remainder, we fix $t>0$.
    
    Suppose now for the sake of contradiction that $T_{\mathrm{pa}}(t,\cdot)$ fails to be monotonically non-decreasing. Thus, there exist $\hat z_1 < \hat z_2$ such that $T_{\mathrm{pa}}(t,\hat z_1)>T_{\mathrm{pa}}(t,\hat z_2)$. For ease of notation, we define the gap function $\delta(t') \coloneqq T_{\mathrm{pa}}(t', \hat z_2) - T_{\mathrm{pa}}(t', \hat z_1)$, which is continuously differentiable. By assumption, we have $\delta(0) = \hat z_2 - \hat z_1 > 0$ and $\delta(t)<0$. The intermediate value theorem then implies that there exists $t^\star < t$ with $\delta(t^\star) = 0$ and $\delta'(t^\star) \leq 0$. Next, an elementary calculation reveals that
    \begin{align*}
        \delta'(t^\star)
        &= \tfrac{\mathrm{d}}{\mathrm{d}t} T_{\mathrm{pa}}(t^\star, \hat z_2) - \tfrac{\mathrm{d}}{\mathrm{d}t} T_{\mathrm{pa}}(t^\star, \hat z_1) \\
        &= \bigl[\nabla_z f(\theta, T_{\mathrm{pa}}(t^\star, \hat z_2)) - 2\lambda\bigl(T_{\mathrm{pa}}(t^\star, \hat z_2) - \hat z_2\bigr)\bigr] \\
        &\quad - \bigl[\nabla_z f(\theta, T_{\mathrm{pa}}(t^\star, \hat z_1)) - 2\lambda\bigl(T_{\mathrm{pa}}(t^\star, \hat z_1) - \hat z_1\bigr)\bigr].
    \end{align*}
    Since $T_{\mathrm{pa}}(t^\star, \hat z_1) = T_{\mathrm{pa}}(t^\star, \hat z_2)$, the terms involving $\nabla_z f$ and $2\lambda\, T_{\mathrm{pa}}(t^\star, \cdot)$ cancel. Thus, we obtain
    \[
        \delta'(t^\star) = 2\lambda(\hat z_2 - \hat z_1)>0,
    \]
    which contradicts the defining property $\delta'(t^\star) \leq 0$ of the point~$t^\star$. Consequently, we may conclude that the continuous-time PA map $T_{\mathrm{pa}}(t,\cdot)$ must be monotonically non-decreasing.

    \textbf{Discrete-time PA.} Fix $\hat z_1 < \hat z_2$ and define the gap $\delta_t \coloneqq T_{\mathrm{pa}}(t, \hat z_2) - T_{\mathrm{pa}}(t, \hat z_1)$. We use induction on~$t$ to show that~$\delta_t\geq 0$ for every~$t\in\mathbb N$. As~$\hat z_1$ and~$\hat z_2$ with~$\hat z_1 < \hat z_2$ were chosen arbitrarily, this readily implies that $T_{\mathrm{pa}}(t,\hat z)$ is monotonically non-decreasing in~$\hat z$. The base case corresponding to~$t = 0$ holds trivially because $T_{\mathrm{pa}}(0, \cdot)$ is the identity map, which implies that $\delta_0 = \hat z_2 - \hat z_1 > 0$. As for the induction step, assume now that $\delta_t \geq 0$, and note that the gap obeys the recursion
    \begin{align*}
        \delta_{t+1}
        &= \delta_t + \eta\bigl[\nabla_z f(\theta, T_{\mathrm{pa}}(t, \hat z_2)) - \nabla_z f(\theta, T_{\mathrm{pa}}(t, \hat z_1))\bigr] - 2\lambda \eta\, \delta_t + 2\lambda \eta (\hat z_2 - \hat z_1) \\
        &= (1 - 2\lambda \eta)\, \delta_t + \eta\bigl[\nabla_z f(\theta, T_{\mathrm{pa}}(t, \hat z_2)) - \nabla_z f(\theta, T_{\mathrm{pa}}(t, \hat z_1))\bigr] + 2\lambda \eta (\hat z_2 - \hat z_1).
    \end{align*}
    We have $\nabla_z f(\theta, T_{\mathrm{pa}}(t, \hat z_2)) - \nabla_z f(\theta, T_{\mathrm{pa}}(t, \hat z_1)) \geq -L_{zz}(T_{\mathrm{pa}}(t, \hat z_2) - T_{\mathrm{pa}}(t, \hat z_1)) = -L_{zz}\, \delta_t$ because $\nabla_z f(\theta, z)$ is $L_{zz}$-Lipschitz continuous in~$z$. Substituting this bound into the recursion yields
    \[
        \delta_{t+1} \;\geq\; \bigl(1 - \eta(L_{zz} + 2\lambda)\bigr)\, \delta_t + 2\lambda \eta (\hat z_2 - \hat z_1).
    \]
    The step size condition $\eta \leq (L_{zz} + 2\lambda)^{-1}$ and the induction hypothesis ensure that the first term is nonnegative, while the second term is positive by construction. We may thus conclude that $\delta_{t+1} \geq 0$, which completes the induction step. Thus, the claim follows.
\end{proof}

\begin{proof}[Proof of~\autoref{lem:assign_stat_implies_cm}]
    The proof widely parallels the second part of the proof of~\autoref{prop:optimal_map_cm}, with the optimality of $T^\star$ replaced by the assignment-stationarity inequality~\eqref{eq:assign-stat}.
\end{proof}

\section{Properties of MPA}
\label{app:MPA_properties}
In this appendix, we establish two key properties of MPA. This discussion focuses on the inner loop of Algorithm~\ref{alg:implicit_mpa}, which evaluates the adversarial risk of a fixed model parameter~$\theta$. First, we prove that an ideal variant of MPA with $K=\infty$ is guaranteed to terminate in finitely many rounds (\autoref{sec:mpa_analysis}). Second, we prove that if the batch size~$B$ is sufficiently large, then MPA is guaranteed to find a near-optimal adversarial map with high probability (\autoref{sec:mpa_global_optimality}). To showcase the benefits of MPA {\em vis-\`a-vis} PA, we also demonstrate that the outer loop of Algorithm~\ref{alg:implicit_mpa} may get trapped at the initial iterate unless the current adversarial map is cyclically monotone (\autoref{sec:toy_separation}).

\subsection{Finite-time Termination of MPA}
\label{sec:mpa_analysis}
Recall our standing assumption that $f(\theta,z)$ is differentiable and $L_z$-Lipschitz continuous with respect to $z$. Some of the results in this section require one or both of the following additional assumptions.

\begin{assumption}
\label{ass:loss_function_smoothness}
The gradient $\nabla_{z}f(\theta,z)$ is $L_{zz}$-Lipschitz continuous in $z$ uniformly over all~$\theta \in \Theta$.
\end{assumption}

\begin{assumption}
    \label{ass:analytic_loss_function}
    The loss function $f(\theta,z)$ is real analytic in~$z$ uniformly across all~$\theta \in \Theta$. 
\end{assumption}
For example, \autoref{ass:analytic_loss_function} is satisfied when $f$ is the cross-entropy loss of a multilayer perceptron, which consists of compositions of linear transformations and analytic activation functions, such as sigmoid, GELU, or Softplus. Throughout this section we focus on one epoch of Algorithm~\ref{alg:implicit_mpa} and thus consider a fixed batch of $B$ training samples $\{\hat z_i\}_{i=1}^B$ generated from~$\hat\PP$. 

Focusing on one epoch also means that~$\lambda > 0$ and~$\theta \in \Theta$ are kept fixed. For ease of notation, we then define 
\begin{equation}
    f_i(z) \coloneqq f(\theta, z) - \lambda \|z - \hat z_i\|_2^2
    \label{eq:per_sample_objective_function}
\end{equation}
as the penalized objective function for each sample~$i \in [B]$.

In the following we use boldface symbols to denote stacked vectors of per-sample quantities, such as $\mathbf{z} = (z_1, \dots, z_B)$. Accordingly, we assemble the per-sample penalized objectives into a vector field $\mathbf{f} \coloneqq (f_1, \dots, f_B)$, such that $\mathbf{f}(\mathbf{z}) = (f_1(z_1), \dots, f_B(z_B))$. 

Since the objective function in~\eqref{eq:maximize-over-maps} depends on the transport map $T \in \mathcal{T}$ only through its evaluations at the empirical samples, we introduce the shorthand $z_i \coloneqq T(\hat{z}_i)$. This allows us to identify the map $T$ restricted to the empirical samples with the tuple $\mathbf{z} \in \RR^{Bm}$. Finally, for each $i \in [B]$, let $\cM_i \subseteq \RR^m$ denote the set of stationary points of $f_i$, and define the joint set $\cM \coloneqq \cM_1 \times \cdots \times \cM_B$.

\begin{lemma}       
    \label{lem:compact_analytic_subset}
    If~\autoref{ass:analytic_loss_function} holds, then $\cM_i$ is a compact analytic subset of $\RR^m$ for each $i \in [B]$.
\end{lemma}

\begin{proof}
    We first show that $\cM_i$ is compact. By definition, any $z \in \cM_i$ satisfies 
    \begin{equation*}
        0=\nabla_z f_i(z) = \nabla_z f(\theta, z) - 2\lambda(z - \hat{z}_i),
    \end{equation*}
    which implies $2\lambda \|z - \hat{z}_i\|_2 = \|\nabla_z f(\theta, z)\|_2$. Lipschitz continuity of $f(\theta,\cdot)$ ensures that $\|\nabla_zf(\theta, \cdot)\|_2\leq L_z$. As~$\lambda>0$, we thus find $\|z - \hat{z}_i\|_2 \le \frac{L_z}{2\lambda}$. Consequently, $\cM_i$ is bounded. Because $\nabla_z f_i(z)$ is continuous, its zero-set $\cM_i$ is also closed, making $\cM_i$ a compact subset of $\RR^m$.

    Next, note that the penalized objective~$f_i(z)$ inherits real analyticity from~$f(\theta, \cdot)$. Therefore, the squared norm of its gradient, $\|\nabla_z f_i(z)\|_2^2$, is also a real analytic function on $\RR^m$. The set~$\cM_i$ of stationary points can be uniquely identified as the zero locus of this single analytic function, that is,
    \begin{equation*}
        \cM_i = \left\{z \in \RR^m : \|\nabla_z f_i(z)\|_2^2 = 0 \right\}.
    \end{equation*}
    By definition, this makes $\cM_i$ an analytic subset of $\RR^m$~\cite[Chapter~6]{krantz2002primer}.
\end{proof}

\begin{corollary}
    \label{cor:finite_critical_values}
    If~\autoref{ass:analytic_loss_function} holds, then~$\cM_i$ has finitely many connected components, and the set of critical values $f_i(\cM_i)$ is finite for each~$i \in [B]$.
\end{corollary}

\begin{proof}
    By \autoref{lem:compact_analytic_subset}, $\cM_i$ is a compact analytic subset of $\RR^m$, which inherently belongs to the broader class of semianalytic sets. A fundamental topological property of semianalytic sets is that they are locally connected and that the family of their connected components is locally finite \cite[Theorem~6.5.12]{krantz2002primer}. Local finiteness means that for every point~$z \in \cM_i$ there exists an open neighborhood $U_z$ that intersects only a finite number of connected components. The collection of all such neighborhoods, $\{U_z\}_{z \in \cM_i}$, naturally forms an open cover of $\cM_i$. Because $\cM_i$ is compact, this open cover admits a finite subcover, say $\{U_{z_1}, \dots, U_{z_K}\}$. Since the entirety of $\cM_i$ lies within the union of these $K$ finite neighborhoods, and each neighborhood intersects only finitely many components, the total number of connected components in $\cM_i$ is finite.

    Let $C \subseteq \cM_i$ be one such connected component. By \cite[Theorem~6.5.12]{krantz2002primer}, the connected component $C$ is itself a semianalytic set. A standard result in semianalytic geometry is that any two points of a connected semianalytic set can be joined by a semianalytic arc, which is real analytic except at finitely many points~\citep{gabrielov1968projections}. Therefore, for any two points $z, z' \in C$, there exists a piecewise smooth path $\gamma: [0,1] \to C$ connecting them. Because every point on $\gamma$ belongs to $\cM_i$, we have $\nabla_z f_i(\gamma(t)) = 0$ for all $t \in [0,1]$. By the fundamental theorem of calculus along curves, we thus have 
    \begin{equation*}
        f_i(z') - f_i(z) = \int_0^1 \langle \nabla_z f_i(\gamma(t)), \tfrac{\diff}{\diff t}\gamma(t) \rangle \, \diff t = 0.
    \end{equation*}
    Therefore, $f_i$ is constant on $C$. Because $\cM_i$ consists of a finite number of connected components, it follows that the set of critical values $f_i(\cM_i)$ is finite.
\end{proof}

We now show that an ideal instance of MPA (Algorithm~\ref{alg:implicit_mpa} with~$K=\infty$) terminates after finitely many rounds. The argument relies on the properties of certain ascent and reassignment operators.

\begin{definition}[Ascent Operators]
    \label{def:ascent-operator}
    For a fixed step size $0 < \eta < 2/(L_{zz}+ 2\lambda)$, we define the
    \begin{itemize}[leftmargin=6mm]
        \item[(i)] \textbf{$i$-th one-step ascent operator:} $\cA_i:\RR^m\to\RR^m$,  $z\mapsto z + \eta \nabla f_i(z)$;
        \item[(ii)] \textbf{$i$-th $t$-step ascent operator:} $\cA^t_i \coloneqq \cA_i \circ \cA_i^{t-1}$ for all $t\in \NN$, where $\cA^0_i \coloneqq \text{id}$;
        \item[(iii)] \textbf{$i$-th asymptotic ascent operator:} $\cA_i^\infty:\RR^m\to\RR^m$,  $z\mapsto \lim_{t \to \infty} \cA_i^t(z)$;
        \item[(iv)] \textbf{$t$-step ascent operator:} $\cA^t:\RR^{Bm}\to\RR^{Bm}$,  ${\bf z}\mapsto (\mathcal{A}_1^t(z_1), \dots, \mathcal{A}_B^t(z_B))$ for all $t \in \NN \cup \{\infty\}$.
    \end{itemize}
\end{definition}

The $t$-step ascent operator implements one full round of $t$ ascent steps in parallel across all $B$ samples. Note that the discrete-time PA map satisfies $T_{\mathrm{pa}}(t, \hat{z}_i)= \mathcal{A}_i^t(\hat z_i)$ and is thus fully determined by the $t$-step ascent operator~$\mathcal{A}_i^t$. However, $\mathcal{A}_i^t$ itself is more general because its input~$z_i$ may differ from the empirical sample~$\hat z_i$, which is hard-coded in the definition of the penalized objective~$f_i$. The next lemma shows that all asymptotic ascent operators are well-defined.

\begin{lemma}
    \label{lem:ascent_convergence}
    If~\Autorefs{ass:loss_function_smoothness}{ass:analytic_loss_function} hold and $0<\eta<2/(L_{zz}+2\lambda)$,
    then the sequence $z^t=\cA_i^t(z^0)$ converges to a stationary point
    $z^\star\in\cM_i$ for every~$i\in[N]$ and initialization $z^0\in\RR^m$. 
\end{lemma}

\begin{proof}
    Fix $i\in[N]$, and set $L \coloneqq L_{zz}+2\lambda$. Since $\nabla f_i$ is $L$-Lipschitz by \autoref{ass:loss_function_smoothness}, the ascent lemma \cite[Proposition~6.1.2]{bertsekas2015convex} and the update rule $z^{t+1}=\cA_i(z^t) = z^t+\eta\nabla f_i(z^t)$ yield
    \begin{equation}
        \label{eq:strong-descent-contidion}
        \begin{aligned}
        f_i(z^{t+1})-f_i(z^t)
        &\geq
        \left(\frac{1}{\eta}-\frac{L}{2}\right)
        \|z^{t+1}-z^t\|_2^2 \\
        &= \left(1-\frac{\eta L}{2}\right)
        \|\nabla f_i(z^t)\|_2
        \|z^{t+1}-z^t\|_2. 
    \end{aligned}
    \end{equation}
    As $\eta < 2/L$, a standard telescoping sum argument implies that $f_i(z^t)\ge f_i(z^0)$ for all~$t \ge 0$. Moreover, the $L_z$-Lipschitz continuity of $f(\theta,\cdot)$ implies that $f_i(z)$ is coercive, \ie, $f_i(z)\to-\infty$ as $\|z\|_2\to\infty$. Thus, the superlevel set $\{z\in\RR^m \mid f_i(z)\ge f_i(z^0)\}$ is compact and contains the entire sequence $\{z^t\}_{t\in\NN}$. To establish convergence, we note that the sequence satisfies the strong ascent conditions of \cite[Definition~3.1]{absil2005convergence}. This is an immediate consequence of~\eqref{eq:strong-descent-contidion}. In particular, the inequality in the first line of~\eqref{eq:strong-descent-contidion} ensures that $f_i(z^{t+1})=f_i(z^t)$  implies $z^{t+1}=z^t$. Because the sequence $\{z^t\}_{t\in\NN}$ is bounded and $f_i$ is real analytic by \autoref{ass:analytic_loss_function}, \cite[Theorem~3.2]{absil2005convergence} guarantees that the sequence converges and thus has a limit~$z^\star$. 
    Finally, the update rule ensures that
    \begin{align*}
        \nabla f_i(z^t) = \frac{z^{t+1}-z^t}{\eta} \longrightarrow 0 \quad \text{as} \quad t \to \infty.
    \end{align*}
    By the continuity of $\nabla f_i$, we thus have $\nabla f_i(z^\star)=0$. This implies that $z^\star\in\cM_i$. 
\end{proof}

\autoref{lem:ascent_convergence} shows that the asymptotic ascent operator $\cA_i^\infty$ maps any input~$z^0$ to a unique stationary point~$z^\star$ and is therefore indeed well-defined for every~$i\in[B]$. Hence, $\cA^\infty$ is well-defined, too.

To mathematically analyze the MPA algorithm, we also define a reassignment operator.

\begin{definition}[Reassignment Operator]
    \label{def:reassignment_operator}
    The reassignment operator $\mathcal{R} : \RR^{Bm} \to \RR^{Bm}$ is defined through $\cR(\mathbf{z})=\mathbf{z}'$, where $z'_i=z_{\sigma(i)}$ and
    $\sigma(i) \in  \argmax_{j \in [B]} f_i(z_j)$ for all $i\in[B]$. Ties among the maximizers are resolved using a self-favoring rule, \textit{i.e.} if $i \in \argmax_{j} f_i(z_j)$, we set $\sigma(i) = i$. Otherwise, ties are broken by choosing the maximizer with the smallest index.
\end{definition}

Given a list $\mathbf{z}=(z_1,\ldots,z_B)$ of $B$ samples, the reassignment operator $\mathcal{R}$ produces a new list $\mathbf{z}'=(z_1',\ldots,z_B')$, where, for each $i\in[B]$, the sample $z_i'$ is a maximizer of $f_i$ over the set $\{z_1,\ldots,z_B\}$. Equivalently, for every $i\in[B]$, $z_i'$ solves the restriction of problem~\eqref{eq:T_star_argmax} with $\hat z=z_i$, in which the uncountable original feasible set $\mathbb{R}^m$ is replaced by the finite set $\{z_1,\ldots,z_B\}$.

By the construction of the penalized objective functions $f_i$, $i\in[B]$, a map~$T$ is assignment-stationary on~$\{\hat{z}_i\}_{i=1}^{B}$ in the sense of \autoref{def:assignment_stationary} if and only if its outputs $z_i\coloneqq T(\hat z_i)$, $i\in[B]$, satisfy $f_i(z_i) \geq f_i(z_j)$ for all $i,j\in[B]$. Thus, each output $z_i$ maximizes the corresponding objective function $f_i$ over the set of outputs $\{z_1,\ldots,z_B\}$. By a slight abuse of terminology, we call the list $\mathbf{z}=(z_1,\ldots,z_B)$ of output samples \emph{assignment-stationary}. We call $\mathbf{z}$ simply {\em stationary} if~$\mathbf{z}\in\cM$.

\begin{lemma}[Properties of the Ascent and Reassignment Operators]
    \label{lem:mpa_operators}
    If~\Autorefs{ass:loss_function_smoothness}{ass:analytic_loss_function} hold and $0<\eta<2/(L_{zz}+2\lambda)$,  the ascent and reassignment operators have the following properties.
    \begin{itemize}[leftmargin=6mm]
        \item[(i)]\textbf{Stationarity:} $\cA^\infty(\mathbf{z})$ is stationary, and $\cR(\mathbf{z})$ is assignment-stationary for all $\mathbf{z} \in \RR^{Bm}$.
        \item[(ii)] \textbf{Idempotence:} $\cA^\infty \circ \cA^\infty = \cA^\infty$ and $\cR \circ \cR = \cR$. 
        \item[(iii)] \textbf{Strict monotonicity:} $\mathbf{f}(\cA^t(\mathbf{z})) \geq \mathbf{f}(\mathbf{z})$ and $\mathbf{f}(\cR(\mathbf{z})) \geq \mathbf{f}(\mathbf{z})$ for all $\mathbf{z}\in\RR^{Bm}$ and $t \in \NN \cup \{\infty\}$, and these inequalities are strict in at least one coordinate if $\cA^t(\mathbf{z}) \neq \mathbf{z}$ or $\cR(\mathbf{z}) \neq \mathbf{z}$, respectively.
    \end{itemize}
\end{lemma}

\begin{proof}
    \textbf{Stationarity:} 
    Fix an arbitrary $\mathbf{z}\in\RR^{Bm}$. \autoref{lem:ascent_convergence} readily implies that $\cA^\infty(\mathbf{z})\in\cM$ is stationary. Similarly, the definition of the reassignment operator readily implies that $\mathbf{z}' \coloneqq \cR(\mathbf{z})$ is such that $z'_i$ maximizes $f_i$ over the candidate set $\{z_1, \dots, z_B\}$. As $\{z'_1, \dots, z'_B\} \subseteq \{z_1, \dots, z_B\}$, it immediately follows that $f_i(z'_i) \ge f_i(z'_j)$ for all $i,j \in [B]$, meaning that $\mathbf{z}'$ is assignment-stationary.

    \textbf{Idempotence:} 
    Set $z^\star_i \coloneqq \cA_i^\infty(z_i)$. As $\nabla f_i(z^\star_i) = 0$ by virtue of \autoref{lem:ascent_convergence}, the point $z^\star_i$ is a fixed point of the one-step operator $\cA_i$, meaning that $\cA_i^\infty(z^\star_i) = z^\star_i$. As this is true for every $i\in[B]$, we have $\cA^\infty \circ \cA^\infty = \cA^\infty$. 
    as for the reassignment operator, set $\mathbf{z}' \coloneqq \cR(\mathbf{z})$. As $\mathbf{z}'$ is assignment-stationary, we have $i \in \argmax_{j\in[B]} f_i(z'_j)$ for all $i\in[B]$. The self-favoring tie-breaking rule in \autoref{def:reassignment_operator} thus dictates that $\cR(\mathbf{z}')_i = z'_i$ for every $i\in[B]$. Consequently, we have $\cR \circ \cR = \cR$.

    \textbf{Strict monotonicity:} 
    Equation~\eqref{eq:strong-descent-contidion} in the proof of \autoref{lem:ascent_convergence} and the assumptions on~$\eta$ imply that $f_i(\cA_i^t(z_i)) \geq f_i(z_i)$ and that this inequality is strict if $\cA_i^t(z_i) \neq z_i$, for every~$i\in[B]$. Hence, the $t$-step ascent operator~$\cA^t$ is strictly monotone for every~$t\in\NN\cup\{\infty\}$. The reassignment operator~$\cR$ is strictly monotone by construction. In particular, if $\mathbf{z}' \coloneqq \cR(\mathbf{z})$ and there exists $i\in[B]$ with $z'_i\neq z_i$, then the self-favoring tie-breaking rule ensures that $f_i(z'_i) > f_i(z_i)$.
\end{proof}

The ascent and reassignment operators enable us to provide an abstract mathematical description of the MPA algorithm. Each epoch is initialized at the empirical samples $\mathbf{z}^{(0)} \coloneqq (\hat{z}_1, \dots, \hat{z}_B)$, and the iterates $\mathbf{z}^{(r)} \coloneqq \mathcal{A}^K (\mathcal{R}(\mathbf{z}^{(r-1)}))$ are constructed recursively for all rounds~$r\in[R]$. Recall that this approach simultaneously solves $B$ coupled nonconvex optimization problems---one for each empirical sample~$\hat z_i$, $i\in[B]$. The strict monotonicity of the ascent and reassignment operators implies that, unless the MPA algorithm has converged to a fixed point of $\cA^K\circ\cR$, at least one of the $B$ objective functions is improved and {\em no} objective function deteriorates in each round.

A fundamental question is whether the inner iterates produced by Algorithm~\ref{alg:implicit_mpa} eventually stabilize when the algorithm is run over~$R=\infty$ rounds. Specifically, it is of interest to investigate whether there exists a finite round~$R$ after which the iterates remain unchanged. If so, we say that the algorithm \emph{terminates in finite time}. For finite~$K$, finite-time termination can occur only in exceptional cases, since performing additional gradient ascent steps generically leads to a strict improvement. For $K=\infty$, by contrast, the algorithm is guaranteed to terminate after finitely many rounds. This follows from the idempotence of both the asymptotic ascent operator~$\mathcal{A}^\infty$ and the reassignment operator~$\mathcal{R}$ (see \autoref{thm:mpa_finite} below). Although performing infinitely many gradient ascent steps is not computationally feasible, the resulting \emph{exact MPA algorithm} is nevertheless of theoretical interest. Besides being more amenable to mathematical analysis, it provides an accurate approximation for the practical algorithm with a finite (sufficiently large) number of gradient ascent steps per round.

In any epoch, the exact MPA algorithm terminates at the first round $r\in\NN$ with $\mathcal{A}^\infty (\mathcal{R}(\mathbf{z}^{(r)})) = \mathbf{z}^{(r)}$. Since every iterate $\mathbf{z}^{(r)}$ constitutes a fixed point of~$\cA^\infty$, this termination condition can be expressed more concisely as $\mathcal{R}(\mathbf{z}^{(r)}) = \mathbf{z}^{(r)}$. The output of the exact MPA algorithm is thus guaranteed to be stationary as well as assignment-stationary. These insights are formalized in the following theorem.

\begin{theorem}[Finite termination of exact MPA]
    \label{thm:mpa_finite}
    If~\Autorefs{ass:loss_function_smoothness}{ass:analytic_loss_function} hold, $0<\eta<2/(L_{zz}+2\lambda)$ and~$K=\infty$, then the full-batch MPA algorithm terminates after $R\leq 1 + \sum_{i=1}^N (|f_i(\cM_i)| - 1)$ rounds in each epoch. The returned output $\mathbf{z}^{(R)}$ is both stationary and assignment-stationary.
\end{theorem}

\begin{proof}
    The first iterate $\mathbf{z}^{(1)} = \cA^\infty(\cR(\mathbf{z}^{(0)}))$ is stationary by \autoref{lem:mpa_operators}\,{\em (i)}, and thus \autoref{cor:finite_critical_values} implies that~$f_i(z^{(1)}_i)$ belongs to the finite set~$f_i(\cM_i)$ for every~$i\in[B]$. For any subsequent non-terminal round $r \geq 2$, we have $\mathbf{z}^{(r)} \neq \mathbf{z}^{(r-1)}$. The definition of the exact MPA algorithm then~yields
    \begin{equation*}
        \mathbf{f}(\mathbf{z}^{(r)}) =  \mathbf{f}(\cA^\infty(\cR(\mathbf{z}^{(r-1)}))) \ge \mathbf{f}(\mathbf{z}^{(r-1)}),
    \end{equation*}
    where the inequality is strict for at least one coordinate~$i\in[B]$ by virtue of \autoref{lem:mpa_operators}\,{\em (iii)}. Because both~$\mathbf{z}^{(r)}$ and~$\mathbf{z}^{(r-1)}$ belong to~$\cM$, this represents a strict jump between two values inside $f_i(\cM_i)$. 
    For each coordinate, $f_i$ can jump to a strictly higher critical value at most $|f_i(\cM_i)| - 1$ times. Summing this bound over all~$N$ coordinates, and adding~$1$ to account for the first round, yields the desired upper bound on the total number of rounds~$R$ until termination: $1 + \sum_{i=1}^N (|f_i(\cM_i)| - 1)$. 
    
    The terminal iterate~$\mathbf{z}^{(R)}$ is generated by the exact ascent operator~$\cA^\infty$, which guarantees stationarity by \autoref{lem:mpa_operators}\,{\em (i)}. The termination condition thus simplifies to $\mathcal{R}(\mathbf{z}^{(R)}) = \mathbf{z}^{(R)}$, which ensures via \autoref{lem:mpa_operators}\,{\em (i)} that $\mathbf{z}^{(R)}$ is also assignment-stationary. This observation completes the proof.
\end{proof}

\begin{remark}[Cyclical monotonicity]
    \label{rem:mpa_caveats}
    In practice, an epoch of Algorithm~\ref{alg:implicit_mpa} may be interrupted before a fixed point of~$\cA^K\circ\cR$ is reached. In this case, the last iterate may fail to be assignment-stationary. Indeed, while~$\cR$ induces assignment-stationarity, $\cA^K$ can fail to preserve assignment-stationarity even if~$K=\infty$. In practice, it is therefore expedient to conclude the MPA algorithm with an extra reassignment step to enforce assignment-stationarity and cyclical monotonicity. We also emphasize that assignment-stationarity is enforced only locally on the given mini-batch. 
\end{remark}

\subsection{Global Optimality Guarantees for MPA}
\label{sec:mpa_global_optimality}

The goal of this section is to show that, for any fixed model parameter~$\theta\in\Theta$ and tolerance~$\epsilon>0$, the inner loop of Algorithm~\ref{alg:implicit_mpa} computes the adversarial risk~$F(\theta)$ with absolute error at most~$\epsilon$ with high probability, provided that the batch size~$B$ is sufficiently large. Thus, MPA essentially solves the nonconvex functional optimization problem~\eqref{eq:maximize-over-maps} to global optimality irrespective of the penalty parameter~$\lambda>0$. In contrast, recall that the state-of-the-art PA algorithm by \citet{sinha2020certifying} provides such a global optimality guarantee only for unrealistically large penalty parameters~$\lambda>L_{zz}/2$.

Since our analysis of MPA is probabilistic, we model the empirical samples explicitly as independent and identically distributed (i.i.d.) random vectors denoted by $\hat Z_1,\ldots,\hat Z_B$.
We study a generic instance of MPA with an arbitrary number of local ascent steps~$K\in\NN\cup\{\infty\}$ and rounds~$R \in \NN$ and with any step size $\eta\in (0,2/(L_{zz}+2\lambda))$; see Algorithm~\ref{alg:implicit_mpa}. Recall from \autoref{sec:implicit} that MPA generates an {\em implicit} adversarial map~$T_{\mathrm{mpa}}$, which is only defined on the given batch of empirical samples. Using the notation from \autoref{sec:mpa_analysis}, this map can be constructed as follows. Initialized at $\mathbf{Z}^{(0)} \coloneqq (\hat{Z}_1, \dots, \hat{Z}_B)$, MPA computes iterates $\mathbf{Z}^{(r)} \coloneqq \mathcal{A}^K \bigl(\mathcal{R}(\mathbf{Z}^{(r-1)})\bigr)$ for all~$r\in[R]$ and outputs~$\mathcal R(\mathbf{Z}^{(R)})$, where~$\mathcal{A}^K$ and~$\mathcal{R}$ represent the ascent and reassignment operators of Definitions~\ref{def:ascent-operator} and~\ref{def:reassignment_operator}. Thus, we can define~$T_{\mathrm{mpa}}(\hat{Z}_i)$ as the $i$-th sample in the last batch of adversarial samples~$\mathcal R(\mathbf{Z}^{(R)})$. 

Throughout this section, we continue to denote the penalized objective associated with the $i$-th empirical sample by~$f_i$ as in \eqref{eq:per_sample_objective_function}, bearing in mind that this sample now constitutes a random vector. We further assume that the empirical samples are governed by a sufficiently regular distribution.

\begin{assumption}
    \label{ass:data_distribution}
    The empirical samples $\hat{Z}_1, \dots, \hat{Z}_B$ are drawn independently from a distribution $\PP_0 \in \mathcal{P}(\RR^m)$ that admits an everywhere positive density with respect to Lebesgue measure on~$\RR^m$.
\end{assumption}

In the following, we set $\rho^2 \coloneqq \EE_{\PP_0}[\|Z\|_2^2]$ and note that $\rho^2<\infty$ because $\PP_0 \in \mathcal P(\RR^m)$. \autoref{ass:data_distribution} implies that any optimal adversarial map is square-integrable with respect to~$\PP_0$.

\begin{lemma}[Square integrability of the optimal adversarial map]
    \label{lem:optimal_adversary_bounded_second_moment}
    If~\autoref{ass:data_distribution} holds and $M^2 \coloneqq 2\rho^2 + L_z^2/(2\lambda^2)$, any optimal adversarial map $T^\star$ in the sense of \autoref{prop:measurability_Tstar} satisfies
    \begin{equation*}
        \EE_{\PP_0}\bigl[\|T^\star(\hat{Z}_i)\|_2^2\bigr] \leq M^2\quad \forall i\in[B].
    \end{equation*}
\end{lemma}

\begin{proof}
    By construction, the random vector $Z^\star\coloneq T^{\star}(\hat{Z}_i)$ is a global maximizer of the random function~$f_i$ pointwise for every uncertainty realization. The stationarity condition $\nabla_z f(\theta,Z^{\star}) =2\lambda(Z^{\star} - \hat{Z}_i)$ thus implies via the proof of~\autoref{lem:compact_analytic_subset} that $\|Z^{\star} - \hat{Z}_i\|_2 \leq \frac{L_z}{2\lambda}$. Hence, we obtain
    \begin{equation*}
        \|Z^{\star}\|_2^2 
        \leq \left(\|\hat{Z}_i\|_2 + \|Z^{\star} - \hat{Z}_i\|_2\right)^2 
        \leq \left(\|\hat{Z}_i\|_2 + \frac{L_z}{2\lambda}\right)^2 
        \leq 2\|\hat{Z}_i\|_2^2 + \frac{L_z^2}{2\lambda^2}.
    \end{equation*}
    The claim then follows by replacing~$Z^\star$ with its definition~$T^\star(\hat Z_i)$, taking expectations with respect to~$\PP_0$ on both sides of the resulting inequality, and invoking the definition of~$M^2$.
\end{proof}

To establish the claimed optimality guarantee of MPA, we fix any tolerance~$\epsilon>0$ and optimal adversarial map~$T^\star$ in the sense of \autoref{prop:measurability_Tstar}, and we introduce two positive radii
\begin{equation}
    \label{def:bulk_radius}
    r_\epsilon \coloneqq \sqrt{\frac{\epsilon}{L_{zz}+2\lambda}} \quad \text{and} \quad R_\epsilon \coloneqq \frac{2L_z M}{\sqrt{2\lambda\epsilon}},
\end{equation}
where $M^2$ stands for the second-order moment bound from \autoref{lem:optimal_adversary_bounded_second_moment}. Using this notation, we partition the index set~$[B]$ of the empirical samples into three disjoint random sets
\begin{subequations}
\label{eq:batch-partition}
\begin{align}
    \label{def:outliers_set}
    O_\epsilon & \coloneqq \big\{i\in[B] : \|T^\star(\hat{Z}_i)\|_2 > R_\epsilon \big\}, \\
    \label{def:uncovered_inliers_set}
    U_\epsilon & \coloneqq \big\{i\in[B] : i \not\in O_\epsilon,~ \|\hat Z_j - T^\star(\hat Z_i)\|_2 > r_{\epsilon} ~ \forall j \neq i\big\} ,\\
    \label{def:covered_inliers_set}
    C_\epsilon & \coloneqq \big\{i\in[B] : i \not\in O_\epsilon,~ i\not\in U_\epsilon \big\}.
\end{align}
\end{subequations}
Samples with indices in~$O_\epsilon$ are referred to as `outliers,' whereas samples with indices in~$U_\epsilon$ and~$C_\epsilon$ are termed `uncovered inliers' and `covered inliers', respectively. We will see below that the numbers of outliers and uncovered inliers are small with high probability and that if~$\hat z$ is any covered inlier, then there exists an empirical sample that is near-optimal in~\eqref{eq:T_star_argmax}. Together, these insights imply that~$T_{\mathrm{mpa}}$ is near-optimal in problem~\eqref{eq:maximize-over-maps} if $\hat\PP=\tfrac{1}{B} \sum_{i=1}^B \delta_{\hat Z_i}$ is the empirical distribution on the given batch.

We first prove that if $B$ is large, then the fraction of outliers is unlikely to exceed $\cO(\epsilon)$.

\begin{lemma}[Outliers are rare]
    \label{lem:outlier_fraction}
    If~\autoref{ass:data_distribution} holds, then we have
    \begin{equation*}
        \PP_0^B\left(\frac{|O_\epsilon|}{B} > \frac{\lambda\epsilon}{L_z^2}\right) \leq \exp\left(-B\lambda\epsilon/6L_z^2\right).
    \end{equation*}
\end{lemma}

\begin{proof}
The indicator variables $Y_i \coloneqq \mathbf{1}_{\{\|T^\star(\hat{Z}_i)\|_2 > R_\epsilon\}}$, $i \in [B]$, form a sequence of i.i.d.\ Bernoulli trials. We denote their common success probability by~$p$ and assume temporarily that $p>0$. We also set $\bar{p}\coloneq \lambda\epsilon /(2L_z^2)$. Markov's inequality together with \autoref{lem:optimal_adversary_bounded_second_moment} then implies that 
    \begin{equation*}
        p = \EE_{\PP_0}[Y_i] = \PP_0\left(\|T^\star(\hat{Z}_i)\|^2_2 > R_\epsilon^2\right) \leq \frac{M^2}{R_\epsilon^2} = \bar p,
    \end{equation*}
    where the last equality exploits the definition of~$R_\epsilon$ in~\eqref{def:bulk_radius}. Re-expressing the number of outliers~$|O_\epsilon|$ as the sum of all Bernoulli trials and setting $\delta \coloneqq 2\bar{p}/p-1\geq 1$ therefore yields 
    \begin{align*}
    \PP_0^B\left(\frac{|O_\epsilon|}{B} > \frac{\lambda\epsilon}{L_z^2}\right)
    &= \PP_0^B\left(\sum_{i=1}^{B}Y_i > 2B\bar{p}\right)
    = \PP_0^B\left(\sum_{i=1}^{B}Y_i > (1+\delta)Bp\right) \\
    &\leq \left(\frac{e^\delta}{(1+\delta)^{1+\delta}}\right)^{Bp} \leq e^{-Bp\delta/3} \leq e^{-B\bar{p}/3}.
    \end{align*}
    Here, the first two equalities follow from the definitions of~$\bar{p}$ and~$\delta$, respectively. The first inequality exploits Chernoff’s inequality for sums of independent Bernoulli random variables \cite[Theorem~2.3.1]{vershynin2018high}, and the second inequality holds because $(1+\delta)\ln(1+\delta)\geq 4\delta/3$ for all~$\delta \geq 1$. The last inequality holds because $p\delta = 2\bar{p} - p \ge \bar{p}$. Replacing $\bar{p}$ with its definition then yields the claim for $p>0$. If $p=0$, on the other hand, $\PP_0^B(\sum_{i=1}^B Y_i>0)=0$, and the claim trivially holds.
\end{proof}

In order to show that the fraction of uncovered inliers is unlikely to exceed $\cO(\epsilon)$, it is expedient to define the minimum probability mass of a ball of radius $r_{\epsilon}$ centered within the closed ball $B_{R_\epsilon}(0)$ as
\begin{equation*}
    \underline{p}(\epsilon) \coloneqq \inf_{z \in B_{R_\epsilon}(0)} \PP_0\left(B_{r_{\epsilon}}(z)\right).
\end{equation*}

\begin{lemma}[Strict positivity of $\underline{p}(\epsilon)$]
    \label{lem:positive_covering_mass}
    If~\autoref{ass:data_distribution} holds, then we have~$\underline{p}(\epsilon) > 0$.  
\end{lemma}

\begin{proof}
    Introduce an auxiliary function~$\kappa:\RR^m\to [0,1]$ defined via $\kappa(z)=\PP_0(B_{r_\epsilon}(z))$. As $\epsilon>0$ and~$\PP_0$ admits an everywhere positive density with respect to Lebesgue measure, we have~$\kappa(z)>0$ for every~$z\in\RR^m$. To prove the lemma, we must show that~$\kappa(z)$ is bounded away from~$0$ on~$B_{R_\epsilon}(0)$.
    
    By \autoref{ass:data_distribution}, the data-generating distribution $\PP_0$ is absolutely continuous with respect to Lebesgue measure on~$\RR^m$. Thus, for any fixed~$z\in\RR^m$, the boundary of~$B_{r_\epsilon}(z)$ has $\PP_0$-measure zero, and the function $h_{\hat z}(z)\coloneq \mathbf{1}_{\{\hat z\in B_{r_\epsilon}(z)\}}$ is continuous (in fact, constant) at~$z$ for $\PP_0$-almost every~$\hat z\in\RR^m$. Consequently, $\kappa(z)$ is continuous thanks to the dominated convergence theorem \cite[Theorem~3.31]{axler2020measure}. As the closed ball~$B_{R_\epsilon}(0)$ is compact, there exists a minimizer~$z^\star\in B_{R_\epsilon}(0)$ with $\underline{p}(\epsilon)=\kappa(z^\star)>0$ by Weierstrass' extreme value theorem. Thus, the claim follows.
\end{proof}

Armed with \autoref{lem:positive_covering_mass}, we can prove that the fraction of uncovered inliers is unlikely to exceed~$\cO(\epsilon)$.

\begin{lemma}[Uncovered inliers are rare]
    \label{lem:uncovered_inlier_fraction}
    If~\autoref{ass:data_distribution} holds, then we have
    \begin{equation}
    \label{eq:uncovered_inlier_bound}
        \PP_0^B\left(\frac{|U_\epsilon|}{B} > \frac{\lambda\epsilon}{L_z^2}\right) \leq \frac{L_z^2}{\lambda\epsilon} \exp\left(-(B-1)\underline p(\epsilon)\right).
    \end{equation}
\end{lemma}

\begin{proof}
    Fix any index $i \in [B]$, and note that
    \begin{align*}
        \PP_0^B(i \in U_\epsilon) & = \PP_0^B(i\in O_\epsilon) \,\PP_0^B(i \in U_\epsilon | i\in O_\epsilon) + \PP_0^B(i \not\in O_\epsilon) \,\PP_0^B(i \in U_\epsilon | i\not\in O_\epsilon) \\
        & \leq \PP_0^B(i \in U_\epsilon | i\not\in O_\epsilon) = \EE_{\PP_0^B} \big[\PP_0^B(i \in U_\epsilon | \hat Z_i,\, i\not\in O_\epsilon) \big| i\not\in O_\epsilon\big] ,
    \end{align*}
    where the two equalities follow from the law of total probability, and the inequality holds because~$O_\epsilon\cap U_\epsilon=\emptyset$, which implies that $\PP_0^B(i \in U_\epsilon | i\in O_\epsilon)=0$. The event $i\not\in O_\epsilon$ is completely determined by~$\hat Z_i$ (that is, it belongs to the $\sigma$-algebra generated by~$\hat Z_i$). As the empirical samples $\hat Z_j$, $j\neq i$, are mutually independent as well as independent of~$\hat Z_i$, the definition of~$U_\epsilon$ in~\eqref{def:uncovered_inliers_set} then implies that
    \begin{align*}
        \PP_0^B \big(i \in U_\epsilon \big| \hat Z_i,\, i\not\in O_\epsilon \big) &= \prod_{j\in[B]:\, j\neq i}\PP_0^B \Big(\|\hat Z_j - T^\star(\hat Z_i)\|_2 > r_{\epsilon} \Big| \hat Z_i,\, i\not\in O_\epsilon \Big) \\
        & \leq (1-\underline p(\epsilon))^{B-1} \leq \exp\bigl(-(B-1)\underline p(\epsilon)\bigr).
    \end{align*}
    Here, the first inequality follows from the definition of~$\underline p(\epsilon)$, which ensures that
    \[
        \PP_0\Big( \|\hat Z_j- T^\star(\hat Z_i)\|_2\leq r_{\epsilon} \Big| \hat Z_i,\, i\not\in O_\epsilon \Big) \ge \underline{p}(\epsilon).
    \]
    Combining the above estimates finally yields
    \begin{align*}
        \PP_0^B(i \in U_\epsilon) & \leq \exp\bigl(-(B-1)\underline p(\epsilon)\bigr) \quad \forall i\in[B].
    \end{align*}
    As $|U_\epsilon|=\sum_{i=1}^B \mathbf 1_{\{i\in U_\epsilon\}}$, the expected number of uncovered inliers satisfies
    \begin{equation*}
        \EE_{\PP_0^B}[|U_\epsilon|] = \sum_{i=1}^B \PP_0^B(i \in U_\epsilon) \le B \exp\bigl(-(B-1)\underline p(\epsilon)\bigr).
    \end{equation*}
    Note that $|U_\epsilon|$ is a non-negative random variable. By Markov's inequality, we thus obtain
    \begin{equation*}
        \PP_0^B\left( \frac{|U_\epsilon|}{B} > \frac{\lambda\epsilon}{L_z^2} \right) \leq \frac{\EE_{\PP_0^B}[|U_\epsilon| / B]}{\lambda\epsilon / L_z^2} \leq \frac{L_z^2}{\lambda\epsilon} \exp\left(-(B-1)\underline p(\epsilon)\right).
    \end{equation*}
    This observation completes the proof.
\end{proof}

Next, we prove that every point in the $r_\epsilon$-neighborhood of a global maximizer of~$f_i$ constitutes an $\epsilon/2$-optimal point. For ease of notation, we define $f_i^\star \coloneqq \max_{z\in\RR^m}f_i(z)$ for every~$i\in[N]$.

\begin{lemma}[Superlevel sets]
    \label{lem:suboptimality_level_sets}
    Suppose that~\autoref{ass:loss_function_smoothness} holds, and define $r_\epsilon$ as in~\eqref{def:bulk_radius}. If~$z^\star,z\in\RR^m$ with~$f_i(z^\star)=f_i^\star$ and~$\|z-z^\star\|_2\leq r_\epsilon$, then $f_i(z)\geq f_i^\star-\epsilon/2$ for all~$i \in [B]$.
\end{lemma}

\begin{proof}
    Note that both~$f_i$ as well as~$-f_i$ are  $(L_{zz}+2\lambda)$-smooth by virtue of \autoref{ass:loss_function_smoothness}. Hence, the standard quadratic upper bound for the smooth function $-f_i$ \cite[Proposition~6.1.2]{bertsekas2015convex} yields
    \begin{align*}
        f_i(z) &\geq f_i(z^\star) + \langle\nabla f_i(z^\star),z-z^\star \rangle - \frac{L_{zz}+2\lambda}{2}\|z-z^\star\|_2^2 \\
        &\geq f_i^{\star} - \frac{L_{zz}+2\lambda}{2} r_\epsilon^2=f_i^{\star} -\frac{\epsilon}{2},
    \end{align*}
    where the second inequality holds because~$z^\star$ is a global maximizer of~$f_i$ satisfying~$\nabla f_i(z^\star) = 0$ and because~$\|z-z^\star\|_2\leq r_\epsilon$, and the equality follows from the definition of~$r_\epsilon$. 
\end{proof}

If $\hat z=\hat Z_i$ for $i\in C_\epsilon$ is a covered inlier and $z^\star=T^\star(\hat z)$, then there exists~$j\neq i$ such that~$z=\hat Z_j$ satisfies $\|z - z^\star\|_2 \leq r_{\epsilon}$; see~\eqref{def:covered_inliers_set}. \autoref{lem:suboptimality_level_sets} thus guarantees that~$z$ constitutes an $\epsilon/2$-optimal solution for problem~\eqref{eq:T_star_argmax}. This reasoning implies that, for every covered inlier, the adversary can find a near-optimal adversarial sample among the finite set of empirical samples. This insight and the earlier results of this section culminate in the following main theorem, which proves that the MPA map $T_{\mathrm{mpa}}$ is near-optimal in problem~\eqref{eq:maximize-over-maps} with high probability. To formalize this result, we recall that the adversarial error $\varepsilon(\theta, T)$ of a map~$T$ is defined as the suboptimality of~$T$ in~\eqref{eq:maximize-over-maps}; see~\eqref{eq:oracle_error}.

\begin{theorem}[Global optimality guarantee of MPA]
    \label{thm:mpa_global_optimality}
    If \Autorefs{ass:loss_function_smoothness}{ass:data_distribution} hold and $0<\eta<2/(L_{zz}+2\lambda)$, then for any $K \in \NN\cup\{\infty\}$, $R \in\NN$, and tolerance $\epsilon > 0$, we have
    \begin{equation*}
        \PP_0^B\bigl( \varepsilon(\theta, T_{\mathrm{mpa}}) \leq \epsilon \bigr) \geq 1 - \beta(B, \epsilon),
    \end{equation*}
    where $\varepsilon(\theta, T)$ is the adversarial error of a map $T$ with respect to $\hat\PP=\tfrac{1}{B} \sum_{i=1}^B \delta_{\hat Z_i}$, and where
    \begin{equation*}
        \beta(B, \epsilon) \coloneqq \exp\left(-B\lambda\epsilon/6L_z^2\right) + \frac{L_z^2}{\lambda\epsilon} \exp\left(-(B-1)\underline p(\epsilon)\right).
    \end{equation*}
\end{theorem}

\begin{proof}
    We first rewrite the adversarial error of $T_{\mathrm{mpa}}$ by expanding the empirical expectation. Using the definition~\eqref{eq:oracle_error} of the adversarial error and the per-sample formulation~\eqref{eq:regularized_dro-dual} of $F(\theta)$, we can write
    \begin{align*}
        \varepsilon(\theta, T_{\mathrm{mpa}}) 
        &= \frac{1}{B}\sum_{i=1}^{B} \left( \max_{z \in \RR^m} \bigl\{ f(\theta, z)\! -\! \lambda \|z - \hat{Z}_i\|_2^2 \bigr\} \! - \! \left( f(\theta, T_{\mathrm{mpa}}(\hat{Z}_i))\! -\! \lambda \|T_{\mathrm{mpa}}(\hat{Z}_i) - \hat{Z}_i\|_2^2 \right) \right) \\
        &= \frac{1}{B}\sum_{i=1}^{B} \left(f_i^\star - f_i(T_{\mathrm{mpa}}(\hat{Z}_i)) \! \right).
    \end{align*}
    Next, we introduce $G_i \coloneqq f_i^\star - f_i(T_{\mathrm{mpa}}(\hat{Z}_i))$ as shorthand for the per-sample suboptimality gap of~$T_{\mathrm{mpa}}$ for every~$i \in [B]$. Using the batch partition~\eqref{eq:batch-partition}, the adversarial error can thus be recast as
    \begin{align}
        \label{eq:error-decomposition}
        \varepsilon(\theta, T_{\mathrm{mpa}}) = \frac{1}{B}\sum_{i \in O_\epsilon} G_i + \frac{1}{B}\sum_{i \in U_\epsilon} G_i + \frac{1}{B}\sum_{i \in C_\epsilon} G_i.
    \end{align}
    By~\autoref{lem:outlier_fraction}, the probability of the event $|O_\epsilon|/B>\lambda\epsilon/L_z^2$ is bounded by $\exp(-B\lambda\epsilon/6L_z^2)$. Similarly, \autoref{lem:uncovered_inlier_fraction} ensures that the probability of the event $|U_\epsilon|/B>\lambda\epsilon/L_z^2$ is bounded by $L_z^2/(\lambda\epsilon)\exp(-(B-1)\underline p(\epsilon))$. Using the union bound, one thus verifies that the desirable event in which both $|O_\epsilon|/B\leq \lambda\epsilon/L_z^2$ and $|U_\epsilon|/B\leq \lambda\epsilon/L_z^2$ occurs with probability at least $1-\beta(B,\epsilon)$.

    It remains to be shown that the adversarial error is bounded by~$\epsilon$ in this desirable event. To this end, recall first that the ascent and reassignment operators are strictly monotone thanks to \autoref{lem:mpa_operators}. Hence, the objective value of $T_{\mathrm{mpa}}(\hat{Z}_i)$ dominates that of $\hat Z_i$, that is, $f_i(T_{\mathrm{mpa}}(\hat{Z}_i)) \geq f_i(\hat{Z}_i)$ for all $i\in[B]$. As $f(\theta, z)$ is $L_z$-Lipschitz continuous in~$z$, we may thus conclude that
    \begin{align*}
        G_i &\leq f_i^\star - f_i(\hat{Z}_i) = \max_{z \in \RR^m} \bigl\{ f(\theta,z) - f(\theta,\hat{Z}_i) - \lambda\|z-\hat{Z}_i\|_2^2 \bigr\} \\
            &\leq \max_{z \in \RR^m} \bigl\{ L_z\|z-\hat{Z}_i\|_2  - \lambda\|z-\hat{Z}_i\|_2^2 \bigr\} = \frac{L_z^2}{4\lambda}.
    \end{align*}
    Using this coarse estimate, we obtain
    \begin{subequations}
    \begin{equation}
        \label{eq:O_epsilon_U_epsilon-bound}
        \frac{1}{B}\sum_{i \in O_\epsilon} G_i \le \frac{|O_\epsilon|}{B} \frac{L_z^2}{4\lambda} \le \frac{\epsilon}{4} \quad \text{and} \quad \frac{1}{B}\sum_{i \in U_\epsilon} G_i \le \frac{|U_\epsilon|}{B} \frac{L_z^2}{4\lambda} \le \frac{\epsilon}{4}.
    \end{equation}
    If~$i\in C_\epsilon$, then we have $\|T^\star(\hat{Z}_i)\|_2\leq R_\epsilon$, and there is $j\neq i$ with $\|\hat Z_j - T^\star(\hat Z_i)\|_2 \leq r_{\epsilon}$; see~\eqref{def:covered_inliers_set}. By \autoref{lem:suboptimality_level_sets}, we thus find $f_i(\hat{Z}_j) \ge f^\star_i - \epsilon/2$ and~$G_i \le \epsilon/2$ for all $i\in C_\epsilon$, which implies that
    \begin{equation}
        \label{eq:C_epsilon-bound}
        \frac{1}{B}\sum_{i \in C_\epsilon} G_i \le \frac{|C_\epsilon|}{B} \frac{\epsilon}{2} \le \frac{\epsilon}{2}.
    \end{equation}
    \end{subequations}
    Substituting~\eqref{eq:O_epsilon_U_epsilon-bound} and~\eqref{eq:C_epsilon-bound} into~\eqref{eq:error-decomposition} yields $\varepsilon(\theta, T_{\mathrm{mpa}}) \leq \epsilon$, and thus the claim follows.
\end{proof}

\begin{corollary}[Sample complexity of MPA]
    \label{cor:mpa_sample_complexity}
    If all assumptions of \autoref{thm:mpa_global_optimality} hold, $\beta\in(0,1)$, and the batch size satisfies
    \begin{equation*}
        B \geq \max\left\{ \frac{6L_z^2}{\lambda\epsilon}\log\left(\frac{2}{\beta}\right), \; 1 + \frac{1}{\underline{p}(\epsilon)}\log\left(\frac{2L_z^2}{\lambda\epsilon\beta}\right) \right\},
    \end{equation*}
    then we have $\varepsilon(\theta, T_{\mathrm{mpa}}) \leq \epsilon$ with probability at least $1-\beta$.
\end{corollary}

\begin{proof}
    To guarantee that the failure probability $\beta(B, \epsilon)$ from \autoref{thm:mpa_global_optimality} falls below $\beta$, it suffices to bound each of its two additive terms by $\beta/2$. As for the first term, we thus require
    \begin{equation*}
        \exp\left(-\frac{B\lambda\epsilon}{6L_z^2}\right) \leq \frac{\beta}{2} \quad\implies \quad B \geq \frac{6L_z^2}{\lambda\epsilon}\log\left(\frac{2}{\beta}\right).
    \end{equation*}
    Similarly, we require the second term to satisfy
    \begin{equation*}
        \frac{L_z^2}{\lambda\epsilon}\exp\bigl(-(B-1)\underline p(\epsilon)\bigr) \leq \frac{\beta}{2} \quad \implies \quad B \geq 1 + \frac{1}{\underline p(\epsilon)}\log\left(\frac{2L_z^2}{\lambda\epsilon\beta}\right).
    \end{equation*}
    
    Taking $B$ to be the maximum of these two lower bounds ensures both conditions hold simultaneously. This observation completes the proof.
\end{proof}

As is typical of global optimality guarantees for nonconvex optimization, the guarantees established in \autoref{thm:mpa_global_optimality} and \autoref{cor:mpa_sample_complexity} come at the price of a curse of dimensionality.

\begin{remark}[Curse of dimensionality]
\label{rem:covering_mass}
The sample complexity bound in \autoref{cor:mpa_sample_complexity} is governed primarily by the inverse covering mass $1/\underline p(\epsilon)$, which conceals an exponential dependence on the ambient dimension~$m$. Indeed, one can construct at least $(R_\epsilon/2r_\epsilon)^m$ disjoint balls of radius~$r_\epsilon$ with centers in $B_{R_\epsilon}(0)$. As their masses sum to at most one, at least one of them has mass at most $(2r_\epsilon/R_\epsilon)^m$, and since $R_\epsilon$ grows when $\epsilon$ decreases, we conclude that
$p(\epsilon)=\cO(r_\epsilon^m)=\cO(\epsilon^{m/2})$ irrespective of the data-generating distribution~$\PP_0 \in \cP(\RR^m)$. The factor $\epsilon^{m/2}$ thus reflects the volume of the covering ball rather than any tail behavior of~$\PP_0$. One can show that if~$\PP_0$ is a standard normal distribution, then $p(\epsilon)=\Omega(\epsilon^{m/2})$. Even in this best possible case, \autoref{cor:mpa_sample_complexity} requires a batch of size $B=\widetilde\cO(\max\{\epsilon^{-1},\epsilon^{-m/2}\})$, which grows exponentially with the ambient dimension~$m$.
\end{remark}

\begin{remark}[MPA versus PA]
\label{rem:mpa_vs_pa}
The global optimality guarantee of \autoref{cor:mpa_sample_complexity} requires a large batch of sufficiently diverse empirical samples. In the absence of such a batch, MPA may not output a globally optimal adversarial map. In fact, as each round of MPA starts with a reassignment step, the adversarial error of the resulting MPA map can even exceed that of the corresponding PA map. To see this, suppose that $f_i(\theta,\hat z_i)<f_i(\theta,\hat z_j)$, while the stationary point in the basin containing~$\hat z_j$ has a lower objective value than the stationary point in the basin containing~$\hat z_i$. To construct the $i$-th adversarial sample, PA applies gradient ascent initialized at~$\hat z_i$, whereas MPA moves from~$\hat z_i$ to~$\hat z_j$ due to the first reassignment step and then applies gradient ascent initialized at~$\hat z_j$. In this case, PA ends up in the better stationary point and thus finds a better adversarial example than MPA. 

To ensure that MPA dominates PA, one could reverse the order of the reassignment and gradient ascent steps in the inner loop of Algorithm~\ref{alg:implicit_mpa}. The first gradient ascent of this alternative variant of MPA replicates the PA map exactly. Moreover, since both the reassignment operator and the $t$-step ascent operator are monotone by \autoref{lem:mpa_operators}\,{\em (iii)}, every subsequent MPA update can only improve upon PA, regardless of the batch size. The price of this guarantee, however, is the loss of a global optimality result analogous to \autoref{thm:mpa_global_optimality}. In particular, one can construct problem instances for which the alternative variant of MPA produces a strictly suboptimal adversarial map for every possible batch of empirical samples. We tested both variants of MPA in our experiments and found that the original version of Algorithm~\ref{alg:implicit_mpa} consistently attains the lower adversarial error. Hence, we do not pursue the ascent-first variant of MPA further.
\end{remark}

\subsection{An Illustrative Example where PA Stalls and MPA Progresses}
\label{sec:toy_separation}

\autoref{ex:counterexample} presents an instance of problem~\eqref{eq:regularized_dro-dual} in which~$\hat\PP$ is a discrete distribution supported on only two atoms and the discrete PA map fails to satisfy cyclical monotonicity after a fixed (albeit large) number of iterations. We now extend this example by embedding it into an outer minimization problem over~$\theta$ and demonstrate that the loss of cyclical monotonicity leads to a critical failure of the outer optimization dynamics. Specifically, the algorithm becomes trapped at the initial outer iterate, while a single MPA reassignment produces a nonzero and informative outer gradient, thereby enabling further progress of the optimization algorithm.

Note that the loss function in~\autoref{ex:counterexample} is constant in~$\theta$. In the following, we construct a closely related loss function that displays a nontrivial dependence on~$\theta$. Specifically, we set
\begin{equation*}\label{eq:toy_outer_loss}
f(\theta, z) \coloneqq \theta\bigl(f_0(z) - b\bigr),
\quad\text{where}
    \quad
f_0(z) \coloneqq \sum_{k=1}^2 a_k \exp\left(-(z-\mu_k)^\top \Sigma_k^{-1} (z-\mu_k)\right).
\end{equation*}
The constants $a_1$, $a_2$, $\mu_1$, $\mu_2$, $\Sigma_1$ and~$\Sigma_2$ are chosen as in \autoref{ex:counterexample}, and $\theta \in \Theta=[0, 1]$ is a scalar model parameter. The value of~$b \in \RR$ will be specified later. Note that $f(\theta , z)$ is affine in~$\theta$ with uniformly bounded gradients, and thus \autoref{prop:convex_bound_inexact} applies. If we define~$\hat\PP$ and~$\lambda$ as in~\autoref{ex:counterexample},
then \autoref{prop:inner_deterministic_optimizer} implies that the adversarial risk can be represented as
\begin{equation}
\label{example-adversary}
F(\theta) = \max_{T \in \mathcal T}\;
\frac{1}{2}\sum_{i=1}^2 \Bigl(f(\theta, T(\hat z_i)) - \lambda\,\|T(\hat z_i) - \hat z_i\|_2^2\Bigr).
\end{equation}
Next, we approximately minimize $F(\theta)$ over $\Theta$ by means of projected gradient descent with step size $\alpha\in(0,1)$, using inexact gradient estimates. The iterates of this algorithm satisfy
\begin{equation}
\label{eq:toy_outer_update}
\theta_{e+1} = \mathrm{Proj}_{[0,1]}\bigl(\theta_e - \alpha\, g_e\bigr),
\quad\text{where}
    \quad
g_e = \frac{1}{2}\sum_{i=1}^2 \bigl(f_0(T_e(\hat z_i)) - b\bigr)
\end{equation}
is an inexact gradient of~$F(\theta)$ at~$\theta=\theta_e$, and $T_e$ is a near-optimal solution of~\eqref{example-adversary} for each epoch~$e\in [E]$. Note that if $T_e$ were a globally optimal adversarial map, then~$g_e$ would coincide with the exact subgradient of~$F(\theta)$ at~$\theta=\theta_e$ by Danskin's theorem. We initialize~$\theta_1 = 1$.

\paragraph{PA stalls.}
Assume first that the adversary's problem in~\eqref{example-adversary} is addressed with discrete PA such that $T_1(\cdot)=T_{\mathrm{pa}}(t,\cdot)$ for~$t=8,000$ steps. As $f(\theta_1,z)$ coincides with the loss function of \autoref{ex:counterexample} up to the additive constant~$-b$, discrete PA initialized at the empirical samples of~$\hat\PP$ yields
\[
T_1(\hat z_1) = (2.17, 1.98)^\top \quad \text{and} \quad
T_1(\hat z_2) = (-3.96, -2.00)^\top.
\]
Hence, $T_1$ fails to be cyclically monotone as shown in \autoref{ex:counterexample}, which implies via \autoref{lem:assign_stat_implies_cm} that~$T_1$ also fails to be assignment-stationary. Next, set
\begin{equation*}\label{eq:b_choice}
b \coloneqq \frac{1}{2}\bigl(f_0(T_1(\hat z_1)) + f_0(T_1(\hat z_2))\bigr)
\end{equation*}
such that the inexact gradient~$g_1$ vanishes. One can then prove by induction that~$\theta_e=1$ for every~$e\in[E]$. This shows that PA-based adversarial training gets trapped at the initial outer iterate~$\theta_1 = 1$ even though~$\theta_1$ is neither a minimizer nor a stationary point of~$F(\theta)$.

\paragraph{MPA progresses.}
Assume next that the adversary's problem in~\eqref{example-adversary} is addressed with a full-batch ($B=2$) version of MPA over~$R=1$ rounds; see Algorithm~\ref{alg:implicit_mpa}. 

To obtain a near-optimal adversarial map for~\eqref{example-adversary} with~$\theta=\theta_1$, MPA applies one reassignment step followed by one discrete PA step, and ends with a final reassignment step. The near-optimal adversarial samples~$\{z_i\}_{i=1}^2$ corresponding to the empirical samples~$\{\hat z_i\}_{i=1}^n$ are thus constructed iteratively as follows. In the first round, MPA sets $\mathbf z\leftarrow \cR(\mathbf{\hat z})$, where $\cR$ is the reassignment operator from~\autoref{def:reassignment_operator}. One readily verifies that
\begin{align*}
    & f(\theta_1, \hat z_2) - \lambda\,\|\hat z_2 - \hat z_1\|_2^2 > f(\theta_1, z_1) \quad \text{and} \quad 
    & f(\theta_1, \hat z_2) > f(\theta_1, \hat z_1) - \lambda\,\|\hat z_1 - \hat z_2\|_2^2,
\end{align*}
suggesting that the adversary is better off mapping {\em both} empirical samples to~$\hat z_2$ instead of keeping the identity initialization. Hence, in the first reassignment step, the adversary sets~$z_1\leftarrow \hat z_2$ and~$z_2\leftarrow \hat z_2$. The first PA step is then warm-started at the same point $z_1 = z_2 = \hat z_2$ for {\em both} empirical samples. In this step MPA sets $z_i\leftarrow \mathcal{A}_i^t(z_i)$ where $\cA^t_i$ is the $i$-th $t$-step ascent operator from~\autoref{def:ascent-operator}. As in \autoref{ex:counterexample}, we set $t=8{,}000$, which ensures that~$z_i$ closely approximates a stationary point. After this PA step, both samples have almost reached the corresponding global maximizers, and one can readily verify that the adversary would not benefit from reassigning the adversarial samples in the final step, that is, the trivial reassignment is optimal. At this stage, MPA has essentially converged. Subsequent PA iterations would merely refine the values of~$\{z_i\}_{i=1}^2$ by adding insignificant digits, while subsequent reassignment steps would have no effect. See \autoref{fig:2D_counterexample} for the trajectory that each sample takes in the MPA algorithm. The adversarial map~$T_{\mathrm{mpa}}\in\mathcal T$ corresponding to two MPA rounds can thus be defined by setting $T_{\mathrm{mpa}}(\hat z_i)=z_i$ for~$i=1,2$ and arbitrarily setting $T_{\mathrm{mpa}}(z)=0$ for all~$z\not\in\{\hat z_i\}_{i=1}^2$, which ensures Borel measurability. By construction, $T_{\mathrm{mpa}}$ is assignment-stationary and cyclically monotone on~$\{\hat z_i\}_{i=1}^2$. A direct calculation reveals that
\[
    \sum_{i=1}^2 f_0(T_{\rm mpa}(\hat z_i)) > \sum_{i=1}^2 f_0(T_{\rm pa}(\hat z_i)) \quad \implies \quad
    g_1 = \frac{1}{2}\bigl(f_0(T_{\rm mpa}(\hat z_1)) + f_0(T_{\rm mpa}(\hat z_2))\bigr) - b > 0.
\]
This shows that MPA-based adversarial training using the iterative scheme~\eqref{eq:toy_outer_update} escapes from the spurious fixed point~$\theta_1=1$ that traps PA. That is, MPA generates a strictly better iterate~$\theta_2<1$.

This example showcases that the reassignment step in MPA is essential. Its purpose is to warm-start the next PA phase at the best adversarial sample output by the previous phase. In particular, if the output of a PA phase happens to lie within the basin of attraction of the global maximizer associated with an empirical sample~$\hat z_i$, then, after the subsequent PA phase, the corresponding adversarial sample~$z_i$ will be driven arbitrarily close to a global maximizer, provided that a sufficiently large number of PA iterations is performed. Through this mechanism, MPA exploits information about the global optimization landscape that remains inaccessible to PA alone and incorporates it into the outer gradient. As a consequence, the outer minimization over~$\theta$ can continue to make progress under MPA even in situations where PA yields a vanishing gradient.

\section{ICNN Training Details}\label{app:icnn-training}
To guarantee convexity of the parametric potential $\psi_\omega(z)$ in $z$, the hidden-to-hidden and output weights of the ICNN must remain non-negative throughout training. We enforce this constraint by reparametrizing these weights as exponentials, $\exp(W^{y}_{\ell})$ and $\exp(w^{y}_{L})$, so that non-negativity holds by construction regardless of the values of the underlying unconstrained parameters $W^{y}_{\ell}$ and $w^{y}_{L}$. A natural starting point would be to initialize $W^{y}_{\ell}$ and $w^{y}_{L}$ according to a standard zero-mean scheme such as Xavier initialization \citep{xavier}. However, even when the pre-exponential weights are zero-mean, the exponentiated weights $\exp(W^{y}_{\ell})$ and $\exp(w^{y}_{L})$ possess a
strictly positive mean. This positive-mean bias accumulates across layers, distorting both the forward signal and the backpropagated gradient, and has been identified as a source of the optimization instabilities and representational bottlenecks reported in prior work \citep{makkuva2020optimal, korotin2021neural, sivaprasad2021curious}. To mitigate this effect, we adopt the principled initialization scheme of \citet{Principled}, which extends the signal-propagation analysis underlying Xavier initialization to the setting of weights with non-zero means, yielding initialization scales that are explicitly corrected for the positive-mean offset introduced by the exponential reparametrization.

Following~\cite{Principled}, for a layer with $n_{\mathrm{in}}$ inputs, we set the target mean and variance of the exponentials to
\begin{equation}\label{eq:principled-moments}
\mu_w = \sqrt{\frac{6\pi}{n_{\mathrm{in}}\bigl[\,6(\pi-1)+(n_{\mathrm{in}}-1)\bigl(3\sqrt{3}+2\pi-6\bigr)\bigr]}}
\quad \text{and} \quad
\sigma_w^{2} = \frac{1}{n_{\mathrm{in}}},
\end{equation}
respectively. To achieve these target values, we sample the original, pre-exponential weights independently from a normal distribution $W^y_{\ell,ij} \sim \mathcal{N}(\mu, \sigma^2)$. This makes the resulting positive weights independent and follow a lognormal distribution $\exp(W^y_{\ell,ij}) \sim \mathrm{LogNormal}(\mu, \sigma^2)$. To find the exact $\mu$ and $\sigma^2$, we match the lognormal moments $\mathbb{E}[\exp W^y_{\ell,ij}] = \exp(\mu + \tfrac12 \sigma^2)$ and $\operatorname{Var}(\exp W^y_{\ell,ij}) = (\exp \sigma^2 - 1)\exp(2\mu + \sigma^2)$ to the target moments $(\mu_w, \sigma_w^2)$ in~\eqref{eq:principled-moments}. This yields
\begin{equation}\label{eq:lognormal-params}
\mu = \ln\bigl(\mu_w^{2}\bigr) - \tfrac{1}{2}\ln\bigl(\sigma_w^{2} + \mu_w^{2}\bigr)
\quad \text{and} \quad
\sigma^{2} = \ln\bigl(\sigma_w^{2} + \mu_w^{2}\bigr) -\ln\bigl(\mu_w^{2}\bigr).
\end{equation}

\paragraph{Cancelling the positive-mean shift.}
The target mean $\mu_w > 0$ still contributes a systematic bias $\mu_w\, \mathbf{1}^{\!\top} y_{\ell-1}$ to each pre-activation. Two designs cancel this contribution at initialization. The original recipe in~\cite{Principled}  attaches a bias term to each exp-reparameterized layer and initializes it at $-\mu_b\, \mathbf{1}$ with
\begin{equation}\label{eq:mu-b}
\mu_b = \sqrt{\frac{3\,n_{\mathrm{in}}}{6(\pi-1)+(n_{\mathrm{in}}-1)\bigl(3\sqrt{3}+2\pi-6\bigr)}}.
\end{equation}
The construction used throughout our experiments takes the second design: the exp-reparameterized layers carry no bias, and the positive-mean contribution is instead absorbed by the unconstrained passthrough $W^z_\ell z + b_\ell$ at each hidden layer and by $(w^z_L)^{\!\top} z + b_L$ at the readout. The two designs are equivalent at initialization, but the second has a structural advantage during training. The passthrough is $z$-dependent, so the offset it applies adapts to the input, whereas a fixed scalar bias on $\exp(W^y_\ell)$ cannot. This cleanly separates roles, so that convexity-preserving parameters are confined to the exponential reparameterization~\eqref{eq:lognormal-params}, while all expressive, sign-free degrees of freedom are concentrated in the well-initialized passthrough. Simpler ICNN variants without a passthrough still require the explicit offset~\eqref{eq:mu-b}; the codebase supports both conventions for ablation.

\subsection{Adaptive Inner Step Size Selection}
\label{app:inner_stepsize}

\begin{algorithm}[t!]
\caption{\textsc{Adaptive Inner Step Size (BB + Armijo)}}
\label{alg:adaptive_inner_stepsize}
\begin{algorithmic}
\REQUIRE Initial iterate $z^{-1}=z^0$, step size bounds $[\eta_{\min}, \eta_{\max}]$, initial step size $\eta_0\in [\eta_{\min}, \eta_{\max}]$, Armijo sufficient ascent $c \in (0,1)$, shrinkage $\tau \in (0,1)$,  \#~max backtracking steps $J$, \#~steps $K$
\FOR{$k = 0,\ldots,K-1$}
    \STATE $s \gets z^k - z^{k-1}$ and  $y \gets \nabla_z f_{\hat z}(z^{k-1}) - \nabla_z f_{\hat z}(z^{k})$ 
    \IF{$\langle s, y \rangle > 0$ }
    \STATE $\eta_k \gets \mathrm{Proj}_{[\eta_{\min},\eta_{\max}]}\eta_k^\mathrm{BB}$, with $\eta_k^\mathrm{BB}$ as in~\eqref{eq:eta_BB}
    \ELSE 
    \STATE $\eta_k \gets \eta_0$
    \ENDIF
    \STATE $j \gets 0$ 
    \WHILE{$j < J$ and $f_{\hat z}(z^k + \eta_k \nabla_z f_{\hat z}(z^k)) -     f_{\hat z}(z^k) < c \eta_k \|\nabla_z f_{\hat z}(z^k)\|^2_2$}
        \STATE $\eta_k \gets \tau \eta_k,\quad j \gets j+1$
    \ENDWHILE
    \STATE $z^{k+1} \gets z^k + \eta_k \nabla_z f_{\hat z}(z^k)$

\ENDFOR
\STATE \textbf{return} $z^{K}$
\end{algorithmic}
\end{algorithm}

The training dynamics of the adversarial map depend heavily on the inner step size $\eta$. For instance, preserving the monotonicity of the PA map in the one-dimensional case requires $\eta \le (L_{zz}+2\lambda)^{-1}$ (\autoref{prop:pa_monotonicity}). Since a neural network $f(\theta, z)$ is highly non-linear with respect to its inputs, the local curvature $\nabla_{zz}^2 f$ varies significantly across the domain. Consequently, a single fixed step size makes the training process brittle. A conservative small step size safely avoids instability but wastes the inner-loop computational budget. Conversely, a large step size can cause the optimization to diverge. To dynamically adapt to the varying local geometry, we propose to use the Barzilai-Borwein (BB) method~\citep{Barzilai-Borwein} paired with an Armijo backtracking~\cite{armijo1966minimization} safeguard (\autoref{alg:adaptive_inner_stepsize}). The BB method acts as a computationally cheap quasi-Newton step to estimate the effective inverse curvature along the search direction, while the Armijo condition serves as a safety net to guarantee sufficient objective improvement. We detail both components next. To this end, let $f_{\hat{z}}(z) = f(\theta,z)-\lambda\|z-\hat{z}\|_2^2$ denote the local objective to be maximized for a fixed sample $\hat{z}$.

\myparagraph{Barzilai-Borwein method.} Quasi-Newton methods, which motivate the Barzilai-Borwein approach, accelerate gradient ascent on $\max_z f_{\hat{z}}(z)$ by replacing the step size with a local curvature estimate. The update takes the form 
\begin{align} \label{eq:zk_update_quasinewton}
    z^{k+1}= z^k + (B^k)^{-1} \nabla_z f_{\hat{z}}(z^k) ,
\end{align} 
where $B^k$ serves as an approximation of the exact Hessian matrix of $-f_{\hat{z}}$ evaluated at the current iterate $z^k$. After taking a step from $z^{k-1}$ to $z^k$, the algorithm observes the secant pair $s = z^k - z^{k-1}$ and $y = \nabla_z f_{\hat{z}}(z^{k-1})- \nabla_z f_{\hat{z}}(z^k)$, which encodes local curvature along the most recent step. A natural requirement is that $B^k$ satisfies the so-called secant equation $B^k s = y$, ensuring the approximation is consistent with the observed curvature. In high dimensions, however, maintaining a dense matrix $B^k$ that satisfies this condition and solving the resulting linear system becomes computationally prohibitive. The BB method alleviates this by approximating the Hessian with a simple scalar multiple of the identity matrix, $B^k \approx (\eta_k)^{-1} \mathrm{id}$, where the scalar $\eta_k$ is chosen to minimize the squared norm of the secant equation residual. Solving this least-squares problem, $\min_{\eta} \|(\eta)^{-1} s - y\|_2^2$, yields the BB step size
\begin{align}\label{eq:eta_BB}
     \eta_k^\mathrm{BB} = \frac{\|s\|_2^2}{\langle s, y \rangle}.
\end{align}
Substituting $(B^k)^{-1}$ in~\eqref{eq:zk_update_quasinewton} with $\eta_k^\mathrm{BB}$ bypasses the matrix operations and yields an efficient scalar update that scales with the effective inverse curvature. Since neural network landscapes are highly non-concave, the secant product $\langle s, y \rangle$ may occasionally be non-positive. In such cases where the curvature estimate is invalid, we fall back to a default initial step size $\eta_0$ (\autoref{alg:adaptive_inner_stepsize}).

\myparagraph{Armijo backtracking.} While the BB rule provides an inexpensive and often accurate estimate of the local inverse curvature, it does not by itself guarantee that the resulting step yields an improvement in the objective. This issue is particularly pronounced in our setting, where $f_{\hat{z}}$ is generally non-concave, and the BB estimate may occasionally be overly aggressive. To ensure stability of the inner maximization, we complement the BB proposal with an Armijo backtracking line search, which enforces a sufficient ascent condition at each iteration.

Given the current iterate $z^k$ and ascent direction $\nabla_z f_{\hat{z}}(z^k)$, we initialize a trial step size with the BB proposal, \textit{i.e.}, $\eta_k\gets\eta_k^\mathrm{BB}$. The Armijo condition requires that the achieved improvement in the objective is at least a fixed fraction of the linearized gain predicted by the gradient, namely
{%
\setlength{\abovedisplayskip}{2pt plus 1pt minus 1pt}%
\setlength{\belowdisplayskip}{2pt plus 1pt minus 1pt}%
\setlength{\abovedisplayshortskip}{0pt plus 1pt}%
\setlength{\belowdisplayshortskip}{1pt plus 1pt minus 1pt}%
\begin{equation*}
    f_{\hat{z}}(z^k + \eta_k \nabla_z f_{\hat{z}}(z^k))-f_{\hat{z}}(z^k) \geq c\eta_k\|\nabla_z f_{\hat{z}}(z^k)\|_2^2,
\end{equation*}
}%
where $c \in (0,1)$ controls the required level of sufficient ascent. This condition compares the actual improvement against the increase predicted by a first-order approximation of $f_{\hat z}$ at $z^k$. When the trial step size is too large, curvature effects cause this approximation to overestimate the gain, and the condition fails. In that case, we reduce the step size to $\eta_k \leftarrow \tau \eta_k$ with $\tau \in (0,1)$, and repeat the test until the condition is satisfied or the maximum number of backtracking steps $J$ is reached. This procedure prevents overly aggressive updates while preserving the efficiency of the BB initialization.

In practice, we operate on mini-batches $\{\hat{z}_i\}_{i=1}^B$ rather than a single sample. Accordingly, the single-sample objective $f_{\hat{z}}(z)=f(\theta,z)-\lambda\|z-\hat z\|_2^2$ is replaced by its batch counterpart
\[
    f_{\hat{\mathbf{z}}}(\mathbf{z})=\frac{1}{B}\sum_{i=1}^B( f(\theta,z_i)-\lambda\|z_i-\hat z_i\|_2^2)\quad \text{with} \quad \mathbf{z}=(z_1,\ldots,z_B).
\]
All steps are applied to this averaged objective, and the same analysis carries over directly. For notational simplicity, we present the algorithm and derivations in the single-sample form above.

\section{Supplementary Details for~\autoref{sec:experiment}} \label{app:experiments-appendix}
Appendices~\ref{app:training_details}, \ref{app:Robustness_AA_Full}, and~\ref{app:exp_neural_transport_rl} provide full hyperparameter settings and
implementation details for the experiments of Sections~\ref{sec:Adversarial_Logistic Regression}, \ref{sec:Robustness_AA_Full}, and~\ref{sec:exp_rl_main}, respectively. The latter two appendices additionally report gradient-obfuscation checks and a sensitivity analysis in the penalty parameter~$\lambda$. For all methods requiring an inner maximization step---namely PA, RO, WFR, NN-DRO, MPA, and ICNN-DRO---we conduct a comprehensive search over
the inner step size~$\eta$, evaluating both constant step sizes and adaptive Barzilai--Borwein step sizes coupled with an Armijo line search (Algorithm~\ref{alg:adaptive_inner_stepsize}), and report the best-performing configuration for each method. All experiments were run on an \texttt{NVIDIA A100-SXM4-80GB} GPU with dual \texttt{AMD EPYC 7543} 32-core CPUs. \looseness=-1

We briefly outline the configurations of our baselines, with exact hyperparameter values deferred to the subsequent experimental sections. To ensure a fair comparison among parametrized-map methods, the network architectures of NN-DRO and ICNN-DRO are matched in depth and width. For the entropy-regularized baselines, we introduce a regularization parameter $\varepsilon_{\mathrm{ent}} > 0$. SDRO~\cite{wang2021sinkhorn} exploits the fact that the entropy-regularized worst-case distribution admits a closed-form Gibbs representation \cite[Remark~4]{wang2021sinkhorn} to reformulate the inner maximization as a nested log-expectation \cite[Theorem~1]{wang2021sinkhorn}, whose gradient is then estimated via randomized multilevel Monte Carlo with importance samples drawn from a truncated geometric distribution~\citep{blanchet2015unbiased}, parametrized by a success probability $p\in(0,1)$ and truncation level $M\in\mathbb{N}_+$. WFR~\cite{xu2025gradientflowsamplerbaseddistributionally}, by contrast, targets the adversarial distribution directly by evolving a cloud of $n\in\mathbb{N}_+$ auxiliary particles per data point under an interacting particle dynamics. Unlike PA, which perturbs the empirical samples themselves, WFR keeps the training data fixed and uses these auxiliary particles to approximate the local adversarial distribution.
The parameters $\varepsilon_{\mathrm{ent}}$, $p$, $M$, and $n$ are specified in each experimental section.

\subsection{Supplementary Details for~\autoref{sec:Adversarial_Logistic Regression}}
\label{app:training_details}
This section provides additional details for the adversarial multi-class logistic regression experiment of
\autoref{sec:Adversarial_Logistic Regression}. All methods are trained for $10$ epochs with batch size $B=128$ and penalty parameter
$\lambda=10$. The outer minimization uses step size
$\alpha=5\times 10^{-3}$ for all robust benchmarks and
$\alpha=10^{-3}$ for the ERM baseline. The inner maximization runs for $K=20$ ascent iterations for ICNN-DRO and for $K=100$ iterations for every other method, with step sizes selected by the BB+Armijo procedure with $(\eta_0, \eta_{\min}, \eta_{\max}) = (5 \times 10^{-4}, 10^{-6}, 1)$
and $(c, \tau, J) = (0.1, 0.5, 10)$. For method-specific parameters, MPA performs $R=5$ rounds of $K=20$ iterations, hence $100$ ascent iterations in total and, counting the terminal reassignment, six reassignments. RO has an $\ell_2$-perturbation
budget of $\varepsilon=0.04$. The entropy-regularized baselines (SDRO and WFR) use a regularization strength of $\varepsilon_{\mathrm{ent}}=0.2$. Moreover,  WFR utilizes $n=16$ auxiliary particles per data point. Additionally, SDRO employs a maximum truncation level of $M=5$ with success probability $p=0.5$ for its multilevel Monte Carlo estimator. The $\ell_2$-constrained PGD attack budget ($\Delta \in [0, 0.08]$) applied to the covariates and the feature embedding process remains exactly as detailed in the main text. \looseness=-1

\subsection{Supplementary Details for~\autoref{sec:Robustness_AA_Full}}
\label{app:Robustness_AA_Full}

Across all benchmarks, the classifier $h_\theta$ is a ResNet-18~\cite{he2015deepresiduallearningimage}, pre-trained on clean CIFAR-10 for $200$ epochs and then adversarially trained under each method for a further $50$ epochs with batch size $B=256$ and penalty parameter $\lambda=30$. The outer minimization uses SGD with Nesterov momentum parameter $0.9$ \cite{Nesterov1983AMF} under a one-cycle schedule with cosine annealing~\cite{cycliclearningrate} and peak learning rate $0.1$.
\looseness=-1

For all gradient-based inner solvers, the inner maximization runs for $K=20$ iterations with the step size set by the Barzilai-Borwein rule with Armijo backtracking~(\autoref{alg:adaptive_inner_stepsize}): initial step $\eta_0 = 2{\times}10^{-4}$, clipping interval $[\eta_{\min},\eta_{\max}]=[10^{-7},\,2.5{\times}10^{-1}]$, sufficient-ascent constant $c=10^{-4}$, shrinkage factor $\tau=0.5$, and at most $J=10$ backtracking steps. The RO inner attack is $\ell_2$-constrained PGD with budget $\varepsilon=0.5$ under the same BB+Armijo schedule, so RO and PA differ only in the $\ell_2$ feasibility projection, not in step-size selection. 
\looseness=-1

MPA performs $R=3$ rounds of $K=10$ ascent iterations, amounting to $30$
ascent iterations and, counting the terminal reassignment, four reassignments in
total. Its step size is constant rather than adaptive, set to $0.1$ in the first
round and to $0.05$ thereafter, which consistently outperformed BB+Armijo in our
experiments. 

For SDRO, we replace the multilevel Monte Carlo estimator used in the lower-dimensional experiments of~\S~\ref{sec:exp_robust_ls}
and~\S~\ref{sec:Adversarial_Logistic Regression} by direct sampling from the Gibbs worst-case distribution, as that estimator showed high-variance in this high-dimensional regime. We use the Metropolis-adjusted Langevin algorithm~\citep{roberts1996exponential}, which repeatedly displaces each particle by a step of size $h$ along the gradient of the log-density, perturbs it by Gaussian noise, and accepts or rejects the resulting move to correct for discretization bias. We run $20$ such steps with
$h=5{\times}10^{-3}$ on $n=8$ particles per data point and set $\varepsilon_{\mathrm{ent}}=0.05$. WFR utilizes $n=32$ auxiliary particles per data point with the same $\varepsilon_{\mathrm{ent}}=0.05$. 

\myparagraph{Per-epoch runtime comparison.}
We record the per-epoch wall-clock time of the full min-max training loop across all methods in~\autoref{tab:CIFAR-10-CIFAR101-CIFAR102}. All runs use PyTorch Distributed Data Parallel (DDP) across $4\times$ A100 GPUs. The DDP communication backend, the large-shared memory dataloader configuration, the random seed, and the per-GPU local batch size are held fixed across methods, so that the reported per-epoch times are directly comparable wall-clock measurements at matched effective batch size and matched all-reduce overhead.  In particular, the per epoch times in~\autoref{tab:CIFAR-10-CIFAR101-CIFAR102} are full DDP wall-clock times (forward,
backward, gradient all-reduce, optimizer step, and inner-loop ascent), not single-GPU compute times. \looseness=-1

\myparagraph{Certifying the absence of gradient obfuscation.}
A reliable adversarial training algorithm must avoid \emph{obfuscated gradients}~\cite{athalye2018obfuscatedgradientsfalsesense}. These are superficial defenses that resist first-order attacks not through true robustness, but by disrupting the gradient signal via shattered, stochastic, or exploding/vanishing gradients. Hence, robustness against PGD alone is therefore insufficient as a robustness certificate. We provide two complementary checks confirming that the robustness of \methodname reflects genuine geometric regularity rather than gradient masking.

AutoAttack (AA) is a parameter-free collection of attacks, each designed specifically to expose the problems described above. Its adaptive variants APGD-CE and APGD-DLR replace the fixed step size of standard PGD with a schedule adapted to the local curvature, and so navigate ill-conditioned or shattered landscapes on which fixed-step PGD stagnates. Square Attack, by contrast, is a score-based random-search procedure that never queries the gradient of the loss function, and is therefore manifestly immune to shattered, stochastic, and vanishing-gradient effects. Reporting robustness under AA thus provides one evidence for the absence of gradient obfuscation. The tight PGD-AA gap of $0.4\%$ for \methodname (vs.\ $3.55\%$ for PA) is the quantitative signature of this consistency (see~\autoref{tab:CIFAR-10-CIFAR101-CIFAR102}).\looseness=-1

We complement the AutoAttack (AA) evaluation with a qualitative inspection of the input-space loss landscape, following~\cite{athalye2018obfuscatedgradientsfalsesense,engstrom2019exploringlandscapespatialrobustness}.
  For a test sample $z=(x,y)$, let $x_{\mathrm{adv}}$ be the adversarial input obtained by $\ell_2$-PGD at budget $\varepsilon=0.5$, using 20 steps and five random restarts. Define the adversarial displacement
  $r=x_{\mathrm{adv}}-x\in\mathbb{R}^{m-1}$ and draw a random direction $u\in\mathbb{R}^{m-1}$ orthogonal to $r$. Writing $\hat{r}=r/\|r\|_2$ and $\hat{u}=u/\|u\|_2$, we probe the classifier on the affine plane
  {%
  \setlength{\abovedisplayskip}{1pt plus 1pt minus 1pt}%
  \setlength{\belowdisplayskip}{1pt plus 1pt minus 1pt}%
  \setlength{\abovedisplayshortskip}{0pt plus 1pt}%
  \setlength{\belowdisplayshortskip}{1pt plus 1pt minus 1pt}%
  \begin{equation*}
      x(\alpha,\beta)
      \;=\; x+\alpha\,\hat{r}+\beta\,\hat{u},
      \qquad
      (\alpha,\beta)\in[-\rho,\rho]^2,
      \qquad \rho=1.
  \end{equation*}
  }%
  Perturbations are measured in pixel-space $\ell_2$ units, with standard CIFAR-10 normalization applied before classifier evaluation. We evaluate the affine plane without additional pixel clipping, recording on a
  $101\times101$ grid both the cross-entropy $f(\theta,(x(\alpha,\beta),y))$ and the classification correctness
  $\mathbbm{1}\left\{\arg\max_k h_\theta(x(\alpha,\beta))_k=\arg\max_k y_k\right\}$.

  \autoref{fig:loss_landscape} presents the resulting surfaces and decision regions for two uniformly sampled test examples. The horizontal and vertical axes correspond to $\beta\,\hat{u}$ and $\alpha\,\hat{r}$,
  respectively; the dashed circle $\alpha^2+\beta^2=\varepsilon^2$ marks the attack budget. Orange and blue denote correct and incorrect classification, while the red curve traces the loss along $\beta=0$. For both
  samples, the loss rises smoothly and monotonically from the clean input to the PGD endpoint, with no visible oscillations or plateaus along this segment. Together with the AA results, these local observations provide
  complementary evidence against gradient masking.\looseness=-1

\begin{figure}[t!]
    \centering
    \includegraphics[width=0.85\textwidth]{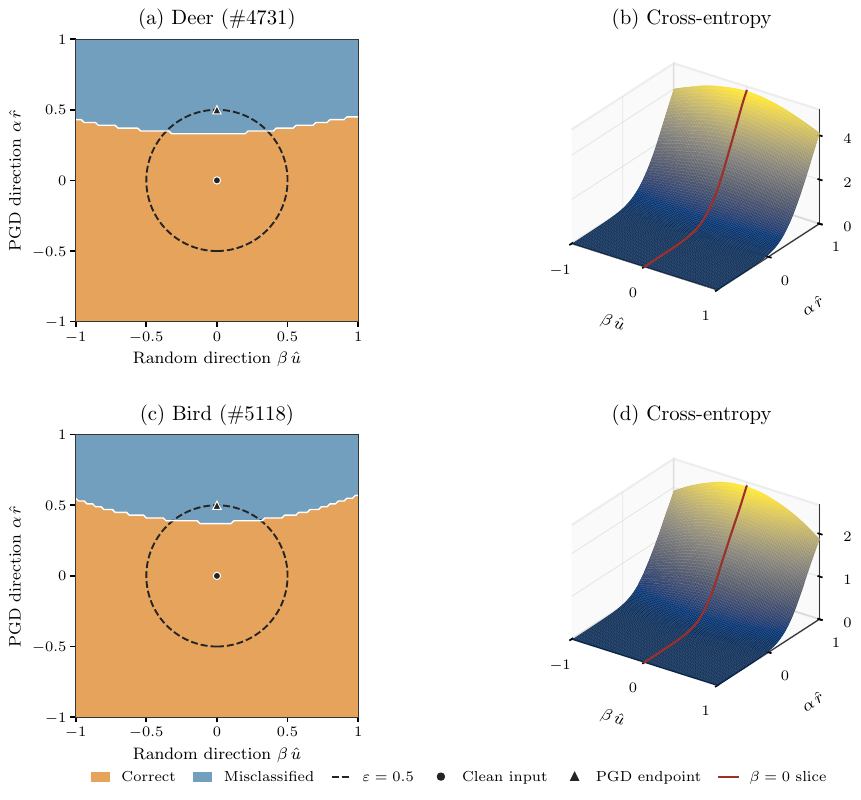}
    \caption{\textbf{Loss landscape of \methodname{} on CIFAR-10.} Panels (a) and (b) correspond to test sample \#4731 (Deer), while panels (c) and (d) detail test sample \#5118 (Bird). For each sample, the left panel (a, c) displays the 2D decision regions. The dashed circle marks the $\ell_2$ budget $\varepsilon=0.5$; orange and blue indicate correct and incorrect classification, respectively; and point markers specify the exact coordinates of the initial clean image (black circle) and the final adversarial image (black triangle). The right panel (b, d) plots the corresponding 3D cross-entropy surface where the red curve traces the loss along the adversarial direction.\looseness=-1}
    \label{fig:loss_landscape}
\end{figure}
\looseness=-1

\subsection{Supplementary Details for \autoref{sec:exp_rl_main}}
\label{app:exp_neural_transport_rl}
This section collects the experimental details of the robust control experiment:
the CartPole environment, hyperparameter configurations, differentiable surrogate
formulations, and additional results.
\looseness=-1

\myparagraph{Setting.}
We model the CartPole environment with a continuous state $s = (x, \dot x, \vartheta, \dot\vartheta) \in \RR^{4}$, capturing the cart's position and velocity alongside the pole's angle and angular velocity. The agent selects discrete actions $a \in \{0, 1\}$, which dictate a horizontal control force $u \in \{-u_{\max}, u_{\max}\}$ applied directly to the cart's center of mass. The system state then evolves according to known dynamics $s_{t+1} = F(s_t, a_t, z)$~\cite{barto1983neuronlike, freeman2021brax}  that are parameterized by the uncertainty vector $z = (m_p, \ell) \in \cZ\subseteq\mathbb{R}^2$, where $m_p$ and $\ell$ denote the mass and half-length of the pole and $\cZ=[0.05, 0.20] \times [0.25, 0.75]$. The agent receives a reward $r_t:=r(s_t,a_t) = 1$ at every non-terminal step. Episodes terminate early if the cart position or pole angle exceeds the physical bounds $x_{\max}$ or $\vartheta_{\max}$, and are otherwise truncated at a maximum horizon of $H=1{,}000$ steps. Hence, the total undiscounted reward corresponds exactly to the episode length censored at $1{,}000$. The loss function $f(\theta, z) \coloneqq  - \mathbb{E}_{\pi_\theta} [ \sum_{t=0}^{H-1} r(s_t, a_t) ]$ is the expected negative undiscounted return under policy $\pi_\theta$ and physics~$z$; see \S~\ref{sec:exp_rl_main}.

Since the physical parameters $z = (m_p, \ell)$ have incommensurable units (kilograms and meters), a standard Euclidean penalty would skew the optimization based on numerical scale. We resolve this by normalizing each coordinate by its domain width. Defining the lower and upper bounds as $\underline{z} = (0.05, 0.25)$ and $\overline{z} = (0.20, 0.75)$ respectively, we employ the width-normalized metric
{%
\setlength{\abovedisplayskip}{1pt plus 1pt minus 1pt}%
\setlength{\belowdisplayskip}{1pt plus 1pt minus 1pt}%
\setlength{\abovedisplayshortskip}{0pt plus 1pt}%
\setlength{\belowdisplayshortskip}{1pt plus 1pt minus 1pt}%
\begin{equation*}
    \|z - \hat{z}\|_w^2 \coloneqq \sum_{k=1}^2 \left(\frac{z_k - \hat{z}_k}{w_k}\right)^2 \quad \text{with} \quad w \coloneqq \overline{z} - \underline{z} = (0.15,\, 0.50).
\end{equation*}
}%
Replacing the Euclidean norm with $\|\cdot\|_w^2$ in the inner maximization objectives~\eqref{eq:regularized_dro-dual} and~\eqref{eq:maximize-over-maps} is equivalent to solving the standard Euclidean DRO problem in the rescaled coordinates $z_k \gets z_k / w_k$.

\myparagraph{Hyperparameters.} 
All methods in the robust control experiment are trained for $E=100$ epochs. For the outer minimization, we optimize the policy $\pi_\theta$ using Proximal Policy Optimization~(PPO)~\citep{schulman2017ppo}. We parameterize PPO's underlying actor (policy) and critic (value) networks as MLPs. Denoting by $g_\theta:\RR^4\to\RR^2$ the two action logits produced by the actor, the policy is  $\pi_\theta(a| s)=\exp(g_{\theta,a}(s))/\sum_{b=0}^{1}\exp(g_{\theta,b}(s))$
for $a\in\{0,1\}$. We update the weights of both networks via SGD with Nesterov momentum~\citep{Nesterov1983AMF} using a momentum parameter of 0.9, while keeping all other PPO settings at their standard defaults (including the outer step size $\alpha$). At each iteration, we draw a batch of $B=64$ nominal uncertainty anchors $\{\hat z_i\}_{i=1}^B$ independently from the uniform distribution on~$\cZ$ and generate one adversarial counterpart for each anchor, yielding a combined batch of $2B=128$ environments. Updates are performed on this mixture of nominal and adversarial conditions so that the agent remains capable on standard tasks while learning to handle worst-case scenarios. We report results for penalty parameters $\lambda \in \{0.1, 0.5, 1\}$. \looseness=-1

MPA performs $R=3$ rounds of $K=15$ ascent iterations, amounting to
$45$ ascent iterations in total and, including the terminal reassignment,
four reassignment steps. For RO, we tune the
$\ell_2$-perturbation budget to the regularization level, using
$\varepsilon=0.75$ for $\lambda=0.1$, $\varepsilon=0.5$ for
$\lambda=0.5$, and $\varepsilon=0.25$ for $\lambda=1.0$. Both
entropy-regularized baselines, SDRO and WFR, use a regularization strength
of $\varepsilon_{\mathrm{ent}}=0.01$. WFR approximates the adversarial
distribution using $n=8$ auxiliary particles per data point, whereas SDRO
uses a multilevel Monte Carlo gradient estimator based on a truncated
geometric distribution with success probability $p=0.5$ and 
truncation level $M=3$.

\myparagraph{Additional results.}
To evaluate the sensitivity of the learned robustness to the regularization
parameter $\lambda$, \autoref{tab:eval_variants} reports results for three
values, with all other training and evaluation settings kept fixed:
panel~(a) uses $\lambda=0.1$, panel~(b) uses $\lambda=0.5$, and panel~(c)
uses $\lambda=1.0$, corresponding to the experiment reported in
\autoref{fig:radar_RL_main}. ICNN-DRO attains the highest mean episode length on every evaluation variant at every penalty level, with MPA second throughout. MPA's advantage over the remaining baselines is substantial at $\lambda=0.5$ and~$\lambda=1.0$, but narrows at $\lambda=0.1$, where it stays ahead of WFR on all seven variants by only a small margin. WFR is consistently the
strongest remaining baseline, but falls substantially behind ICNN-DRO
on the harder extrapolative shifts, particularly Heavy and Short. Overall, performance varies non-monotonically with $\lambda$
for several methods, while the ordering of the two strongest approaches remains stable throughout.

\begin{table}[!t]
    \centering
    \renewcommand{\arraystretch}{1.15}
    \caption{Mean episode length $\pm$~standard error over $1{,}000$ deterministic rollouts across different $\lambda$. Panel~(a) reports $\lambda=0.1$; panel~(b) reports $\lambda=0.5$ and panel~(c) reports $\lambda=1$. The leftmost column corresponds to the original training configuration, the middle group to easier physics variants (Light, Long, Soft Gravity), and the right group to harder variants (Heavy, Short, Strong Gravity). Bold entries indicate the best result per column.\looseness=-1}
    \resizebox{\linewidth}{!}{
        \begin{tabular}{l c ccc ccc}
            \toprule
            \multirow{2.5}{*}{Method} & \multicolumn{1}{c}{Original} & \multicolumn{3}{c}{Easier Environments} & \multicolumn{3}{c}{Harder Environments} \\
            \cmidrule(lr){2-2} \cmidrule(lr){3-5} \cmidrule(lr){6-8}
            & Original & Light & Long & Soft Gravity & Heavy & Short & Strong Gravity \\
            \midrule
            \multicolumn{8}{l}{\emph{(a) $\lambda{=}0.1$}} \\
            \midrule
            ERM  & $186.4 \pm 1.9$  & $188.7 \pm 1.9$  & $191.4 \pm 1.2$  & $217.3 \pm 1.9$  & $178.6 \pm 1.5$  & $155.6 \pm 0.8$  & $157.6 \pm 1.1$ \\
            PA & $382.9 \pm 7.4$  & $395.1 \pm 7.4$  & $266.9 \pm 4.1$  & $543.3 \pm 7.1$  & $289.0 \pm 5.8$  & $205.1 \pm 1.7$  & $207.0 \pm 2.8$ \\
            RO & $380.1 \pm 5.1$ & $398.3 \pm 5.5$ & $270.5 \pm 3.4$ & $556.6 \pm 5.3$ & $295.1 \pm 2.1$ & $209.0 \pm 1.6$ & $212.4 \pm 1.2$\\
            WFR      & $507.3 \pm 7.6$  & $532.9 \pm 8.0$  & $434.5 \pm 5.8$  & $756.8 \pm 7.2$  & $387.0 \pm 4.6$  & $313.4 \pm 2.5$  & $323.9 \pm 3.0$ \\
            SDRO     & $239.5 \pm 3.3$  & $241.6 \pm 4.7$  & $238.2 \pm 2.9$  & $355.8 \pm 6.0$  & $213.5 \pm 2.4$  & $173.6 \pm 1.4$  & $182.7 \pm 1.6$ \\
            NN-DRO   & $225.3 \pm 1.8$  & $228.2 \pm 1.6$  & $221.0 \pm 1.3$  & $289.6 \pm 3.1$  & $210.3 \pm 1.7$  & $184.1 \pm 0.6$  & $185.3 \pm 1.4$ \\
            \cmidrule(lr){1-8}
            MPA     & $516.4 \pm 7.5$  & $549.3 \pm 7.0$  & $446.3 \pm 4.3$  & $763.3 \pm 8.1$  & $390.6 \pm 5.8$  & $323.3 \pm 2.7$  & $337.4 \pm 3.4$ \\
            ICNN-DRO & $\mathbf{771.3} \pm 7.3$  & $\mathbf{779.4} \pm 7.5$  & $\mathbf{616.3} \pm 6.2$  & $\mathbf{978.2} \pm 1.6$  & $\mathbf{665.2} \pm 7.8$  & $\mathbf{478.3} \pm 4.3$  & $\mathbf{415.3} \pm 5.5$ \\
            \midrule
            \multicolumn{8}{l}{\emph{(b) $\lambda{=}0.5$}} \\
            \midrule
            ERM  & $186.4 \pm 1.9$  & $188.7 \pm 1.9$  & $191.4 \pm 1.2$  & $217.3 \pm 1.9$  & $178.6 \pm 1.5$  & $155.6 \pm 0.8$  & $157.6 \pm 1.1$ \\
            PA & $255.1 \pm 2.0$  & $259.0 \pm 2.0$  & $252.9 \pm 1.1$  & $354.2 \pm 3.4$  & $239.1 \pm 1.3$  & $219.9 \pm 0.9$  & $217.6 \pm 1.0$ \\
            RO & $260.3 \pm 2.1$ & $267.4 \pm 3.2$ & $258.8 \pm 4.1$ & $366.6 \pm 7.5$ & $245.1 \pm 5.2$& $227.0 \pm 1.2$& $221.3 \pm 2.5$\\
            WFR      & $481.9 \pm 7.2$  & $491.7 \pm 7.3$  & $375.9 \pm 4.6$  & $656.3 \pm 7.2$  & $400.6 \pm 5.9$  & $301.8 \pm 2.2$  & $311.9 \pm 3.4$ \\
            SDRO     & $227.6 \pm 4.1$  & $231.8 \pm 3.8$  & $216.2 \pm 3.2$  & $303.5 \pm 6.0$  & $210.7 \pm 2.6$  & $181.8 \pm 1.3$  & $184.5 \pm 1.6$ \\
            NN-DRO   & $221.0 \pm 1.3$  & $224.4 \pm 1.7$  & $213.0 \pm 1.4$  & $285.4 \pm 2.5$  & $207.6 \pm 1.2$  & $187.2 \pm 0.6$  & $188.2 \pm 1.3$ \\
            \cmidrule(lr){1-8}
            MPA      & $641.1 \pm 7.7$  & $659.8 \pm 7.3$  & $481.4 \pm 6.4$  & $835.2 \pm 5.0$  & $488.6 \pm 7.0$  & $421.7 \pm 3.6$  & $381.6 \pm 3.5$ \\
            ICNN-DRO & $\mathbf{773.6} \pm 7.6$  & $\mathbf{781.6} \pm 7.8$  & $\mathbf{618.2} \pm 6.9$  & $\mathbf{984.3} \pm 1.8$  & $\mathbf{656.3} \pm 8.1$  & $\mathbf{468.3} \pm 4.1$  & $\mathbf{401.3} \pm 5.3$ \\
            \midrule
            \multicolumn{8}{l}{\emph{(c) $\lambda{=}1.0$}} \\
            \midrule
            ERM  & $186.4 \pm 1.9$  & $188.7 \pm 1.9$  & $191.4 \pm 1.2$  & $217.3 \pm 1.9$  & $178.6 \pm 1.5$  & $155.6 \pm 0.8$  & $157.6 \pm 1.1$ \\
            PA & $342.4 \pm 5.4$  & $349.9 \pm 5.4$  & $298.0 \pm 3.1$  & $545.2 \pm 7.4$  & $278.8 \pm 3.0$  & $236.4 \pm 1.3$  & $238.4 \pm 1.8$ \\
            RO & $345.0 \pm 4.2$ & $351.6 \pm 4.3$ & $300.6 \pm 2.7$  & $551.2 \pm 2.3$ & $286.3 \pm 1.2$ & $240.0 \pm 3.8$ & $243.7 \pm 1.6$ \\
            WFR      & $435.0 \pm 6.8$  & $459.6 \pm 7.3$  & $387.4 \pm 5.2$  & $703.8 \pm 7.7$  & $337.5 \pm 4.1$  & $276.2 \pm 2.0$  & $287.3 \pm 2.8$ \\
            SDRO     & $246.4 \pm 4.5$  & $248.7 \pm 4.5$  & $244.6 \pm 3.5$  & $378.1 \pm 7.0$  & $215.9 \pm 2.9$  & $176.7 \pm 1.1$  & $183.3 \pm 1.8$ \\
            NN-DRO   & $228.0 \pm 1.7$  & $231.4 \pm 1.9$  & $227.0 \pm 1.1$  & $293.3 \pm 2.7$  & $214.8 \pm 1.5$  & $193.4 \pm 0.9$  & $192.6 \pm 1.1$ \\
            \cmidrule(lr){1-8}
            MPA      & $570.3 \pm 8.4$  & $581.6 \pm 8.5$  & $428.6 \pm 5.9$  & $772.1 \pm 6.8$  & $426.3 \pm 7.0$  & $342.3 \pm 2.5$  & $339.8 \pm 3.6$ \\
            ICNN-DRO & $\mathbf{829.3} \pm 6.6$  & $\mathbf{837.4} \pm 6.3$  & $\mathbf{797.5} \pm 6.7$  & $\mathbf{992.7} \pm 0.6$  & $\mathbf{725.2} \pm 7.7$  & $\mathbf{561.7} \pm 6.4$  & $\mathbf{472.5} \pm 7.6$ \\
            \bottomrule
        \end{tabular}
    }
    \label{tab:eval_variants}
\end{table}

\myparagraph{Differentiable rollouts and pathwise gradients.}
To ascend $\mathcal{L}(\theta;\omega)$ with respect to $\omega$, we need $\nabla_\omega f(\theta, T_\omega(\hat{z}))$. Applying the chain rule, the gradient of the objective $f(\theta, T_\omega(\hat{z}))$ decomposes as
\[
    \nabla_\omega f(\theta, T_\omega(\hat{z})) = \left[ \frac{\partial T_\omega(\hat{z})}{\partial \omega} \right]^\top \nabla_z f(\theta, z) \Big|_{z = T_\omega(\hat{z})}.
\]
We must resolve three distinct non-differentiability challenges to ensure both components of this chain rule permit continuous gradient flow. Specifically, the environment gradient $\nabla_z f(\theta, z)$ presents two discontinuities within the standard MDP formulation, while the map Jacobian $\partial T_\omega(\hat{z}) / \partial \omega$ introduces a third related to physical boundaries. We address these in sequence. 

Evaluating $\nabla_z f(\theta, z)$, which is also required by particle-based baselines like PA and SDRO to optimize their adversarial perturbations, requires differentiating the episodic return directly through the unrolled environment dynamics; \textit{i.e.}, $\nabla_z r_t$. Following the decoupled formulation in~\cite[Appendix~A.1]{you2025acceleratingvisualpolicy}, the gradient of the per-step reward $r_t$ with respect to $z$ propagates recursively
\[
    \nabla_z r_t = (\nabla_z a_t)^\top \nabla_{a_t} r_t
    +
    (\nabla_z s_t)^\top \nabla_{s_t} r_t,
\]
where the state  and action Jacobians are given recursively by
\[  
\nabla_z s_t = \frac{\partial s_t}{\partial z} + \frac{\partial s_t}{\partial s_{t-1}} \nabla_z s_{t-1} + \frac{\partial s_t}{\partial a_{t-1}} \nabla_z a_{t-1},\quad \nabla_z a_{t-1}= \frac{\partial a_{t-1}}{\partial s_{t-1}} \nabla_z s_{t-1}.
\]
Two terms in this standard discrete MDP formulation break the gradient flow. First, the original action $a_t \in \{0, 1\}$ is discrete, rendering $\nabla_z a_t = 0$ almost everywhere and completely severing the gradient. We address this with a straight-through estimator~\cite{bengio2013estimating}. Specifically, we define a vector of available physical forces $\mathbf{u}\coloneqq(-u_{\max},u_{\max} )^\top$. In the forward pass, the simulator applies the hard action $a_t = \argmax_{a\in\{0,1\}} g_{\theta,a}(s_t)$, so that the realized force is
$-u_{\max}$ if $a_t=0$ and $u_{\max}$ otherwise. In the backward pass, derivatives are
instead taken through the relaxed control signal $u_t \coloneqq \mathrm{softmax}(g_\theta(s_t))^\top \mathbf{u}$, where 
\[  
    [\mathrm{softmax}(\mathbf{v})]_a = \frac{e^{v_a / \tau}}{\sum_{a'\in\{0,1\}}e^{v_{a'} / \tau}}
\]  
for a temperature $\tau > 0$, so that the degenerate term $\nabla_z a_t $ is replaced by $\nabla_z u_t $. Second, an episode terminates once $|x_t|>x_{\max}$ or
$|\vartheta_t|>\vartheta_{\max}$, at which point the per-step reward switches
from one to zero. As a function of~$s_t$, this reward has zero gradient away
from the termination boundaries and is nondifferentiable on the boundaries. We circumvent this pathology by replacing the hard indicator with a smooth survival function~$p_t(s_t)$, which captures the probability that the agent remains within physical bounds at time~$t$. Given the limits $x_{\max}$ and $\vartheta_{\max}$ and the logistic sigmoid function $\sigma(x) = (1 + e^{-\beta x})^{-1}$ with a sharpness parameter $\beta>0$, we define
\[    
    p_t(s_t) = \sigma\big(x_{\max} - |x_t|\big) \cdot \sigma\big(\vartheta_{\max} - |\vartheta_t|\big).
\]
The loss $f(\theta, z)$ is then approximated by the sum of cumulative survival probabilities over the horizon
\[
    f(\theta, z) \approx -\sum_{t=0}^{H-1} \prod_{k=0}^{t} p_k(s_k).
\]
As the parameters $\beta \to \infty$ and $\tau \to 0^+$, this formulation exactly recovers the true undiscounted return and the discrete action logic of the original MDP.  By operating away from these asymptotic limits in practice (we set $\beta=20$ and $\tau=1$), these relaxations provide a smooth approximation that carries informative gradients throughout the rollout. This aligns with established practices in differentiable simulation, where smoothing discrete transitions is essential to prevent chaotic or vanishing gradients~\citep{hu2020difftaichi, freeman2021brax, suh2022differentiable, metz2022gradientsnotall}. With these relaxations, we estimate the environment gradient $\nabla_z f(\theta,z)$ by backpropagating the surrogate loss through each simulated trajectory while holding its sampled initial state fixed, and then averaging the resulting derivatives over the sampled initial states. This procedure is known as pathwise, or reparameterization, gradient estimation~\citep{heess2015svg}; see also \citep[\S5]{mohamed2020mcgradient}. Because the smooth survival score modifies the reward and the straight-through action rule substitutes relaxed derivatives for discrete actions, the resulting estimator is generally biased relative to the original discrete-MDP objective~\cite{bengio2013estimating}.

With the issues related to $\nabla_z f(\theta, z)$ resolved, we finally address the third challenge regarding the map Jacobian $\frac{\partial T_\omega(\hat{z})}{\partial \omega}$. The physical parameter $z$ must satisfy the box constraints $\underline{z} \leq z \leq \overline{z}$ for the dynamics to remain well-defined. While particle-based baselines (such as PA and WFR) can enforce this via projected gradient ascent, which does not require differentiating through the projection operator itself, optimizing $\omega$ requires backpropagating through the entire mapping mechanism. If we bounded the map's output using a hard projection operator, the Jacobian $\partial T_\omega(\hat{z}) / \partial \omega$ would evaluate to zero everywhere outside the valid domain (\ie, whenever  $z>\overline{z}$ or $z<\underline{z}$), severing the gradient flow. To guarantee a non-vanishing Jacobian, we reparameterize $T_\omega$ such that its output is naturally bounded. We lift the normalized nominal parameters into an unconstrained latent space using the logistic sigmoid function $\sigma(x) = (1 + e^{-x})^{-1}$, apply a neural mapping $H_\omega$, and then map the result back to the physical bounds via the logistic sigmoid function $\sigma$. Expressed as a single composition, the full adversarial map becomes
\[
    T_\omega(\hat{z}) = \underline{z} + (\overline{z} - \underline{z}) \odot \sigma\left( H_\omega\left( \sigma^{-1}\Big( \frac{\hat{z} - \underline{z}}{\overline{z} - \underline{z}} \Big) \right) \right),
\]
where both $\sigma$ and $\sigma^{-1}$ are applied element-wise and $\odot$ denotes the Hadamard product. For the NN-DRO baseline, $H_\omega$ is parameterized as a standard Multi-Layer Perceptron. For our ICNN-DRO approach, $H_\omega$ is explicitly modeled as the Brenier gradient $\nabla \psi_\omega$. This formulation respects the physical box constraints while remaining differentiable end-to-end. 

\section{Additional Experiments}
This section collects experiments and analyses that complement the main text. Section~\ref{sec:exp_robust_ls} presents a robust least-squares experiment; Section~\ref{sec:cross-model-universality} examines the transferability of adversarial perturbations across classifiers; Section~\ref{app:monge_audit} audits empirical Monge gaps across penalty levels. Two further studies probe the geometry of ICNN-DRO itself: Section~\ref{sec:design_axes} reveals that the rank~$r_{\mathrm{out}}$ of the quadratic term in the ICNN's output layer---rather than its depth or width---is the binding constraint on adversarial capacity, and Section~\ref{sec:lambda_annealing} investigates the effect of annealing the penalty~$\lambda$ during training.

\subsection{Robust Least Squares}
\label{sec:exp_robust_ls}
\begin{wrapfigure}[18]{r}{0.4\textwidth}
    \vspace{-0.7cm}
    \includegraphics[width=\linewidth]{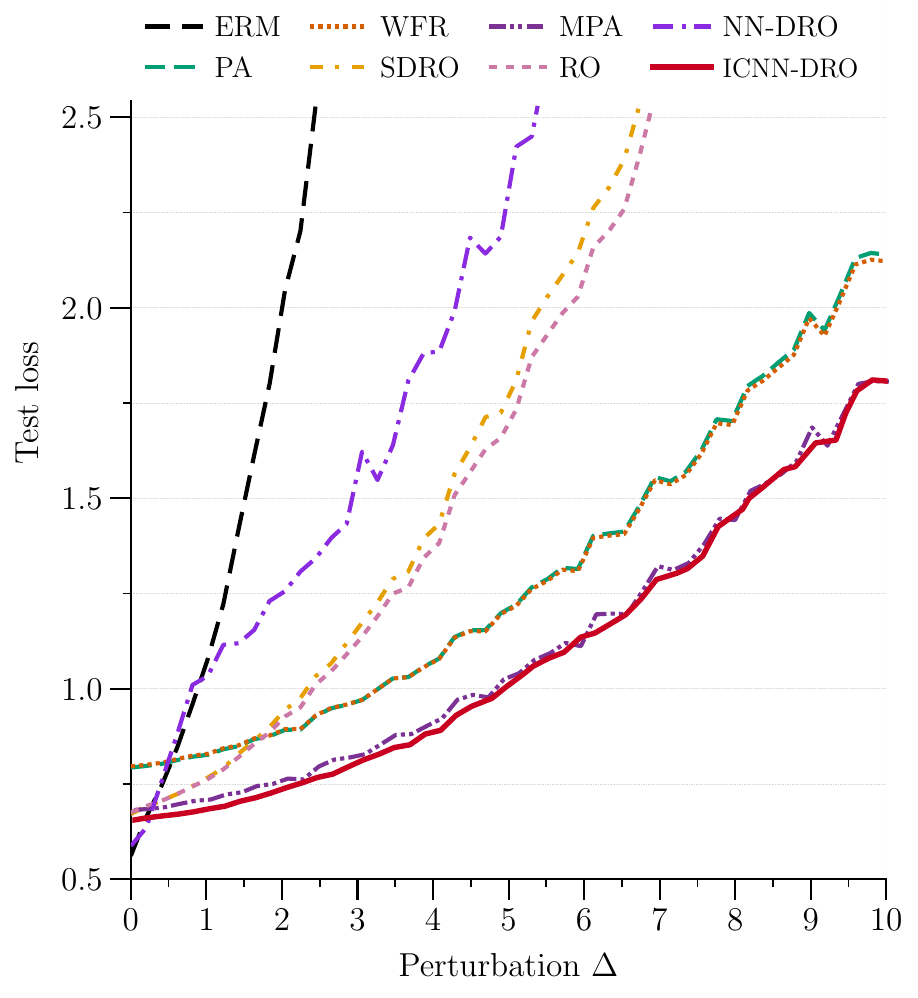}
    \caption{Test loss vs.\ perturbation level $\Delta$ on uncertain least squares.}
    \label{fig:Least_squares}
\end{wrapfigure}
We evaluate robustness on a synthetic \emph{uncertain least squares} problem, following the setup in~\cite{xu2025gradientflowsamplerbaseddistributionally}, which in turn builds on~\cite{zhu2021kernel, el1997robust}. The goal is to learn a parameter vector $\theta \in \Theta=\mathbb{R}^{10}$ for a least-squares problem in which uncertainty enters through the system matrix. Specifically, the uncertainty variable is one-dimensional, with values restricted to the uncertainty domain $[-1,1]$. The per-sample loss is defined as $f(\theta,z)\coloneqq\|A(z)\theta-b\|_2^2$, where $A(z){\coloneqq} A_0+zA_1 \in \mathbb{R}^{10\times 10}$. The matrices $A_0,A_1\in\mathbb{R}^{10\times 10}$ and the vector $b\in\mathbb{R}^{10}$ are fixed realizations with entries sampled independently from $\mathcal N(0,1)$. To simulate a limited observation window, we generate $N=10$ nominal training samples $\{\hat z_i\}_{i=1}^N$ independently from the uniform distribution on~$[-0.5,0.5]$ to construct the empirical reference measure $\hat{\mathbb P}=\frac{1}{N}\sum_{i=1}^N \delta_{\hat z_i}$. Robustness is then evaluated by testing the learned models on a shifted distribution $\PP_\Delta\coloneqq\mathrm{Unif}([-0.5(1+\Delta),\,0.5(1+\Delta)])$ for $\Delta\in[0,10]$, where larger $\Delta$ corresponds to a greater departure from the training distribution. 

All methods are trained for $E=10$ epochs using full-batch gradient descent ($N=B=10$) with an outer step size of $\alpha=0.01$ and a penalty parameter of $\lambda=0.1$. Every inner maximization is given the same budget of $3{,}000$ ascent iterations, with step sizes selected by the BB+Armijo procedure (\autoref{alg:adaptive_inner_stepsize}) under the best configuration from our
search, namely  $[\eta_{\min}, \eta_{\max}] = [10^{-6}, 10^{1}]$,  $\eta_0 = 10^{-1}$, and $(c, \tau, J) = (10^{-4}, 0.5, 10)$. MPA splits this
budget into $R=5$ rounds of $K=600$ iterations, which amounts to six
reassignments once the terminal one is counted, and RO has an
$\ell_2$-perturbation budget of $\varepsilon=0.316$. Both entropy-regularized baselines, SDRO and WFR, utilize a regularization strength of $\varepsilon_{\mathrm{ent}}=0.1$. Within this group, WFR approximates the adversarial distribution using $n=8$ auxiliary particles per data point, while SDRO employs a multilevel Monte Carlo gradient estimator based on a truncated geometric distribution with success probability $p=0.5$ and a truncation level of $M=4$. 
\looseness=-1

The results in \autoref{fig:Least_squares}, averaged over $10$ independent runs, compare ICNN-DRO against ERM and the established robust baselines. As $\Delta$ increases, all DRO variants substantially outperform ERM, confirming that robust training improves stability under distribution shift. Among the DRO baselines, the SDRO method degrades the fastest, likely due to the bias inherent in its nested Monte Carlo estimation. Since $m=1$, \autoref{prop:pa_monotonicity} guarantees that PA is already cyclically monotone, so
one might expect PA and MPA to perform similarly. Yet MPA still outperforms PA, showing that its
benefit extends beyond restoring cyclical monotonicity. Even when cyclical monotonicity already 
holds, the reassignment step acts as an implicit multi-start mechanism, letting each particle escape 
a weak local maximum by inheriting a stronger point discovered by a neighboring sample. ICNN-DRO attains the best performance in this experiment, followed closely by MPA, 
while PA and WFR trail by a wider margin. This behavior is expected because, in one dimension, the inner maximization pushes worst-case samples toward the boundaries of the uncertainty set, collapsing to either $-1$ or $1$. As a result, the various DRO formulations ultimately discover the same worst-case support, leaving limited room for our structured map-based approach to demonstrate a significant advantage. \looseness=-1

\vspace{-0.35cm}
\subsection{Cross-Model Universality}
\vspace{-0.25cm}
\label{sec:cross-model-universality}
When solving~\eqref{prob:regularized_dro} with an ICNN as described in \autoref{sec:explicit_maps}, the resulting adversarial perturbations are constructed with respect to a specific trained model $\theta$. This raises the question of whether these perturbations are tailored to that model or whether they generalize across different models. In this appendix, we investigate the extent to which perturbations generated against a given \emph{source model} remain effective when applied to other \emph{target models} that differ from it in training or in architecture. Here, ``universality'' should be understood in the \emph{cross-model} sense: the ICNN still produces an \emph{input-dependent} perturbation for each test example, but we ask whether that perturbation
continues to fool models other than the one used during adversary training.

\begin{wraptable}[20]{r}{0.55\textwidth}
  \vspace{-0.4cm}
  \centering
  \caption{\textbf{Cross-architecture transfer of the ICNN adversary trained against R1.} 
    \emph{Drop}: Clean-Adv.\ gap; \emph{Ovlp}: counts examples fooled by both
    the source and the target; \emph{\% src err}: percentage of source-model failures that transfer to the target;
    \emph{Lbl mis.\%}: percentage of overlap examples on which source and target predict different incorrect labels.}
  \label{tab:adv_transfer}
  \vspace{-0.1cm}
  \footnotesize
  \setlength{\tabcolsep}{3pt}
  \renewcommand{\arraystretch}{0.9}
  \begin{tabular}{@{}l ccc ccc@{}}
    \toprule
    & \multicolumn{3}{c}{\textbf{Accuracy (\%)}} & \multicolumn{3}{c}{\textbf{Transfer}} \\
    \cmidrule(lr){2-4} \cmidrule(lr){5-7}
    \textbf{Model} & \textbf{Clean}$\uparrow$ & \textbf{Adv.}$\uparrow$ & \textbf{Drop}$\downarrow$ & \textbf{Ovlp.} & \textbf{\% src} & \textbf{Lbl mis.\%} \\
    \midrule
    \textit{R1-ICNN (src)} & $92.38$ & $74.29$ & $18.09$ & $2{,}336$ & --- & --- \\
    \midrule
    ViT1 & $89.37$ & $49.39$ & $39.98$ & $1{,}903$ & $81.46$ & $56.96$ \\
    ViT2 & $78.22$ & $53.54$ & $24.68$ & $1{,}696$ & $72.60$ & $54.13$     \\
    ViT3 & $82.32$ & $54.38$ & $27.95$ & $1{,}763$ & $75.47$ & $55.30$     \\
    ViT4 & $92.42$ & $57.29$ & $35.12$ & $1{,}839$ & $78.72$ & $55.41$ \\
    \midrule
    R2   & $98.14$ & $42.34$ & $55.80$ & $2{,}120$ & $90.75$ & $42.78$     \\
    R3   & $91.54$ & $50.99$ & $40.55$ & $2{,}043$ & $87.46$ & $36.76$     \\
    R4   & $89.99$ & $46.35$ & $43.64$ & $2{,}029$ & $86.86$ & $42.34$     \\
    R5   & $91.71$ & $52.45$ & $39.26$ & $1{,}927$ & $82.49$ & $47.85$     \\
    \bottomrule
  \end{tabular}
  
\end{wraptable}
We consider a population of CIFAR-10 classifiers consisting of five ResNet-18 models (R1--R5) and four Vision Transformers (ViT1--ViT4)~\cite{dosovitskiy2021an}, where models within each group differ only in their training hyperparameters. We take R1 (the first ResNet-18) as the \emph{source model}, train an ICNN adversary against it, and denote the resulting pair by R1-ICNN. The remaining models, namely R2-R5 and ViT1-ViT4, serve as \emph{target models}. This distinction is important. Specifically, the source model is the one that appears inside the minimax problem during adversary training, whereas the target models are used only at evaluation time to test transfer.

We evaluate on a fixed subset of $N=9{,}085$ CIFAR-10 test points. For each point, we generate one perturbation against the source model R1 using its trained ICNN adversary, and then apply that perturbed input unchanged to every target model. The source model is misclassified on $2{,}336$ of these perturbed examples. In \autoref{tab:adv_transfer}, the accuracy block reports clean accuracy, adversarial accuracy, and their gap on this same fixed subset. The transfer block then measures how source failures propagate to each target: \emph{Ovlp.} counts perturbed examples misclassified by both source and target (and is therefore at most $2{,}336$), \emph{\% src} is this overlap as a percentage of the $2{,}336$ source adversarial failures, and \emph{Lbl mis.\%} is the percentage of overlap examples on which source and target predict different incorrect labels.

It is clear that transfer is much stronger within the ResNet family than across architectures. For the held-out ResNets R2-R5, adversarial accuracy drops by $39.26$ to $55.80$ points and the transferred attack recovers $82\%$--$91\%$ of the source errors. For ViT1-ViT4, the drops are smaller, $24.68$ to $39.98$ points, and the recovery rate falls to $73\%$--$81\%$. This indicates that the learned perturbations are not merely exploiting one particular set of weights; rather, they capture vulnerabilities that are shared most strongly within the architectural class of the source model. ResNet targets are not only fooled more often (\emph{\% src}: $82\%$--$91\%$ vs.\ $73\%$--$81\%$ for ViTs), but also more often agree with the source on the wrong label (\emph{Lbl mis.\ (\%)}: $36.76\%$--$47.85\%$ vs.\ $54.13\%$--$56.96\%$). This suggests that perturbations learned against R1 transfer more faithfully within the ResNet family. This is a desirable failure mode for a robust training method: transfer persists beyond the source model, but degrades across architectures. \looseness=-1

\subsection{An Empirical Monge Gap}
\label{app:monge_audit}
\begin{table}[!t]
  \centering
  \renewcommand{\arraystretch}{1.00}
  \caption{Monge gap on CIFAR10.}
  \resizebox{\linewidth}{!}{
    \begin{tabular}{@{}c l S S S S S@{}}
      \toprule
       & & \multicolumn{5}{c}{Regularization strength $\lambda$} \\
      \cmidrule(lr){3-7}
       & Method & {$2$} & {$5$} & {$10$} & {$15$} & {$30$} \\
      \midrule
      
       & ERM      & {$0$} & {$0$} & {$0$} & {$0$} & {$0$} \\
      \hdashline\noalign{\vskip 1pt}
      \multirow{4}{*}{\rotatebox[origin=c]{90}{Baselines}}
       & PA       & 1.24e-2 & 1.15e-2 & 3.01e-3 & 7.93e-4 & 7.97e-4 \\
       & WFR      & 1.13e-2 & 1.02e-2 & 1.73e-3 & 5.68e-4 & 5.58e-4 \\
       & SDRO     & 1.83e-0 & 1.04e-0 & 4.52e-3 & 1.27e-3 & 1.29e-3 \\
       & NN-DRO   & 1.37e-2 & 1.21e-2 & 3.54e-3 & 3.28e-3 & 3.16e-3 \\
      \cmidrule(l){2-7}
      \multirow{2}{*}{\rotatebox[origin=c]{90}{Ours}}
       & MPA                                    & 1.22e-3 & 1.08e-3 & 9.03e-4 & 5.89e-4 & 5.07e-4 \\
    
       & ICNN-DRO & 1.52e-7 & 1.33e-7 & 1.28e-7 & 1.25e-7 & 1.12e-7 \\
      \bottomrule
    \end{tabular}
  }
  \label{tab:monge_gap_vs_lambda_lr}
  \vspace{-0.5cm}
\end{table}
The analysis in \S~\ref{sec:Preliminaries} shows that cyclical monotonicity is
tied to the quality of the inner maximization. By~\autoref{prop:wasted_transport},
transport waste induced by a non-monotone admissible map appears as a positive
Monge gap, and hence as suboptimality in problem~\eqref{eq:maximize-over-maps}.
In this appendix, we examine this effect empirically by estimating the Monge
gap in two settings: the feature-space logistic regression experiment in
\S~\ref{sec:Adversarial_Logistic Regression} and the robust CartPole control
experiment in \S~\ref{sec:exp_rl_main}. \looseness=-1

Recall from \autoref{def:monge_gap} that, for any $\hat\PP\in\cP$, the Monge gap of a map $T\in\cT$ is given by 
\[
    \cM_{\hat\PP}(T)
    \;=\;
    \EE_{\hat\PP}\bigl[\|\hat Z-T(\hat Z)\|_2^2\bigr] \; - \; \wasserstein^2 \bigl(T_\#\hat\PP,\,\hat\PP\bigr).
\]
In this appendix we specialize this definition to the empirical measure $\hat\PP = \frac{1}{B}\sum_{i=1}^{B}\delta_{\hat z_i}$. Since $\wasserstein^2(T_\#\widehat\PP,\widehat\PP)$ is costly to evaluate
directly for high-dimensional empirical measures, we follow
the approach in~\cite{uscidda2023monge} and replace it with the 
Sinkhorn divergence. Specifically, for $\PP,\hat\PP\in\cP$, we first define the entropy-regularized optimal transport cost as
\begin{equation*}
    \OT(\PP,\hat\PP)
    \;\coloneqq\;
    \inf\big\{ \EE_{\gamma}\!\bigl[\|Z-\hat Z\|_2^2\bigr] + \rho\,\mathrm{KL} \big(\gamma\,\|\,\PP\otimes\hat\PP\big) :\gamma\in\Gamma(\PP,\hat\PP) \big\},
\end{equation*}
and we write $\mathsf{S}(\PP,\hat\PP)$ for the associated Sinkhorn
divergence as $ \mathsf{S}(\PP,\hat\PP) \coloneqq \OT(\PP,\hat\PP)-\tfrac{1}{2}\OT(\PP,\PP)-\tfrac{1}{2}\OT(\hat\PP,\hat\PP)$~\cite{feydy2019interpolating}. We then estimate the Monge gap  $\cM_{\widehat\PP}(T)$ by
\begin{equation*}
    \widehat{\cM}_{\widehat\PP}(T)
    \;\coloneqq\;
    \EE_{\hat\PP}\big[\|T(\hat Z)- \hat Z\|_2^2\big]
    \;-\;
    \mathsf{S}\bigl(T_\#\hat\PP, \hat\PP\bigr).
\end{equation*}
The parameter $\rho>0$ is used only in this Sinkhorn computation and is distinct
from the training-time entropic regularization $\varepsilon_{\mathrm{ent}}$
used in SDRO and WFR. Moreover, we have $\widehat{\cM}_{\widehat\PP}(T)\geq\cM_{\widehat\PP}(T)$, and the two coincide when $\rho=0$; see~\cite[Appendix~A.3]{uscidda2023monge}. Thus, we obtain an upper bound on the actual Monge gap. In all experiments we set
\[
    \rho = 0.01 \cdot \mathrm{median}\big\{\|u-v\|_2^2 : u,v \in \{\widehat z_i\}_{i=1}^B\cup\{T(\widehat z_i)\}_{i=1}^B,\; u\neq v \big\},
\]
that is, $\rho$ is $1\%$ of the median pairwise squared distance between distinct points in the combined set of empirical and adversarial samples. We solve the resulting Sinkhorn problem for $1{,}000$ iterations to solver tolerance $10^{-7}$. Next, we report the (approximate) Monge gap for the experiments in \S~\ref{sec:Adversarial_Logistic Regression} and \S~\ref{sec:exp_rl_main}; see \autoref{tab:monge_gap_vs_lambda_lr} and \autoref{tab:monge_gap_vs_lambda_rl}. A common pattern in both tables is that, as $\lambda$ increases, the quadratic penalty in~\eqref{eq:regularized_dro-dual} and~\eqref{eq:maximize-over-maps} becomes more dominant, and thus the objective function of the inner maximization problem becomes more concave. Hence, the adversarial map is forced closer to the identity, and the problem moves closer to ERM. Accordingly, the Monge gap decreases across methods. The ERM row is identically zero in both tables, since its map is the identity. \looseness=-1

\myparagraph{Adversarial multi-class logistic regression (\S~\ref{sec:Adversarial_Logistic Regression}).}
We evaluate $\widehat{\cM}_{\widehat\PP}(T)$ on the full CIFAR-10 dataset, so
that $N=B=60{,}000$, and report the results in
\autoref{tab:monge_gap_vs_lambda_lr} for
$\lambda\in\{2,5,10,15,30\}$; the main-text experiment in
\S~\ref{sec:Adversarial_Logistic Regression} uses $\lambda=10$. At
$\lambda=10$, \methodname is the best-performing method in \autoref{fig:Combined_Logistic_regression} and also attains the smallest Monge gap estimate, essentially at numerical precision. This is consistent with the
role of transport waste in the inner-maximization error. Still, the Monge gap does not determine the full empirical ranking: MPA has the second-smallest Monge gap at $\lambda=10$ and improves over PA by roughly an order of magnitude (see \autoref{tab:monge_gap_vs_lambda_lr}), yet it is not the second-best
method empirically. This is not surprising, since the Monge gap appears only in a lower bound on the adversarial suboptimality; \cf~\autoref{prop:wasted_transport}.
\looseness=-1

\myparagraph{Robust control (\S~\ref{sec:exp_rl_main}).} For this experiment, we evaluate $\widehat{\cM}_{\widehat\PP}(T)$ on a set of $B=32{,}768$ points sampled uniformly from the CartPole box $(m_p,\ell)\in[0.05,0.20]\times[0.25,0.75]$, using adversaries trained for $H=1{,}000$ episode steps. The resulting estimates are reported in \autoref{tab:monge_gap_vs_lambda_rl} for $\lambda\in\{0.1,0.5,1\}$. Across all three values of $\lambda$, ICNN-DRO attains the smallest Monge gap estimate, while MPA is consistently second and improves over PA by roughly three orders of magnitude. The method with the smallest gap throughout, ICNN-DRO, yields also the best-performing controller in \S~\ref{sec:exp_rl_main}.  At $\lambda=1$ (\ie, the value used in \autoref{fig:radar_RL_main} in~\S~\ref{sec:exp_rl_main}), ICNN-DRO achieves the best control performance, with MPA as the runner-up; these are also the two methods with the smallest Monge gap estimates in \autoref{tab:monge_gap_vs_lambda_rl}. Thus, although the Monge gap enters only in a lower bound on the adversarial error, it still appears to be a useful indicator of empirical performance in this low-dimensional
control setting.

\begin{table}[t]
  \centering
  \renewcommand{\arraystretch}{1.0}
  \caption{Monge gap on robust control (CartPole).}
  \begin{tabular}{@{}c l S S S@{}}
    \toprule
     & & \multicolumn{3}{c}{Regularization strength $\lambda$} \\
    \cmidrule(lr){3-5}
     & Method & {$0.1$} & {$0.5$} & {$1$} \\
    \midrule
    
     & ERM      & {$0$} & {$0$} & {$0$} \\
    \hdashline\noalign{\vskip 1pt}
    \multirow{4}{*}{\rotatebox[origin=c]{90}{Baselines}}
     & PA       & 4.87e-2 & 2.01e-2 & 1.99e-2 \\
     & WFR      & 1.87e-2 & 1.41e-2 & 1.13e-2 \\
     & SDRO     & 2.18e-2 & 2.03e-2 & 1.93e-3 \\
     & NN-DRO   & 2.17e-2 & 2.07e-2 & 1.98e-3 \\
    \cmidrule(l){2-5}
    \multirow{2}{*}{\rotatebox[origin=c]{90}{Ours}}
     & MPA      & 6.16e-6 & 1.16e-6 & 1.15e-6 \\
     & ICNN-DRO & 6.92e-7 & 6.82e-7 & 8.81e-8 \\
    \bottomrule
  \end{tabular}
  \label{tab:monge_gap_vs_lambda_rl}
\end{table}

\subsection{Design Axes of the ICNN Adversary}
\label{sec:design_axes}

Three design choices govern the effective capacity of the ICNN adversary with architecture~\eqref{ICNN Architecture}. The first is the depth and the width of the hidden layers. The second is the placement of the quadratic injections, either at every hidden layer or only at the readout. The third is the rank $r_{\mathrm{out}}$ of the matrix $A_L\in\mathbb{R}^{r_{\mathrm{out}}\times m}$ that generates the quadratic term of the readout. Depth, width, and injection placement are the conventional levers, whereas $r_{\mathrm{out}}$ has no counterpart in standard network design. We now isolate the contribution of each axis and show that the third one dominates.

\myparagraph{Design grid.}
For the depth and width axes, we consider three backbones, \texttt{[64,64,64,64]} (simple), \texttt{[128,128,128,128]} (moderate) and \texttt{[1024,512,512,256,128,64]} (complex). For the quadratic injections, we consider two variants: quadratic only in the readout (\textit{i.e.}, ICNN-DRO) and quadratic at all layers (\textit{i.e.}, ICNN-DRO-Q). For the readout rank we sweep $r_{\mathrm{out}}\in\{32,128,512\}$ in the readout-only variant and $r_{\mathrm{out}}\in\{8,32,128\}$ in the all-layers variant. We repeat the robust image classification experiment in~\autoref{sec:Robustness_AA_Full} with all of these network architectures. The results are collected in~\autoref{tab:icnn-capacity}. We observe that a simple ICNN equipped with a high-rank readout outperforms substantially larger ICNNs with low-rank readouts, while requiring markedly less computation. Hence, $r_{\mathrm{out}}$ serves as the effective capacity knob of our architecture and determines the strength of the learned adversarial map more directly than network width or depth, and than the injection of quadratics.

\looseness=-1

\begin{table*}[h!]
    \centering
    \vspace{-0.3cm}
    \caption{\textbf{Adversary capacity versus robustness.}
    Best results per column are in bold, and second-best results are underlined.
    PGD cells are shaded green for configurations satisfying our robustness
    criterion (PGD accuracy $\geq 56\%$), with darker green indicating stronger
    robustness, and red otherwise, with darker red indicating a larger shortfall.
    Runtime cells are shaded green only for configurations satisfying this
    robustness criterion, with darker green indicating lower computational cost.
    A runtime marked with $\dagger$ corresponds to a configuration with
    PGD accuracy below $56\%$. Although such configurations may be substantially
    faster, we do not regard them as successful accelerations of ICNN training
    because they fail our minimum robustness criterion. The objective of this
    ablation is specifically to reduce the approximately $10^4$\,s wall-clock
    training time of the full-width quadratic ICNN while retaining at least
    $56\%$ PGD accuracy.}

    \vspace{-0.2cm}

    \setlength{\tabcolsep}{20pt}
    \renewcommand{\arraystretch}{1.25}

    \resizebox{0.99\linewidth}{!}{%
    \begin{tabular}{
        l
        !{\hspace{5pt}} c c
        !{\hspace{8pt}} c c
        !{\hspace{8pt}} c
    }
        \toprule
        \multirow{2.5}{*}{Method}
        & \multirow{2.5}{*}{$r_{\mathrm{out}}$}
        & \multirow{2.5}{*}{\shortstack{Total\\params}}
        & \multicolumn{2}{c}{CIFAR-10}
        & \multirow{2.5}{*}{\shortstack{Wall-clock\\Runtime~(s)}} \\
        \cmidrule(lr){4-5}
        & & & Clean & PGD$_{\varepsilon=0.5}$ & \\
        \midrule

        \rowcolor{gray!12}
        \multicolumn{6}{l}{
            \textit{Complex backbone}~
            \texttt{[1024, 512, 512, 256, 128, 64]}
        } \\
        \addlinespace[1pt]

        ICNN-DRO-Q & 8 & 31,693,314
        & 91.37
        & \cellcolor{gtrade!17}57.36
        & \cellcolor{gtrade!4}10854 \\

        ICNN-DRO-Q & 32 & 31,767,042
        & 91.30
        & \cellcolor{gtrade!18}57.41
        & \cellcolor{gtrade!3}10908 \\

        ICNN-DRO-Q & 128 & 32,061,954
        & 91.21
        & \cellcolor{gtrade!21}57.58
        & \cellcolor{gtrade!2}11325 \\

        ICNN-DRO & 8 & 8,665,602
        & \textbf{91.45}
        & \cellcolor{gtrade!16}57.28
        & \cellcolor{gtrade!17}6439 \\

        ICNN-DRO & 32 & 8,739,330
        & \underline{91.42}
        & \cellcolor{gtrade!18}57.40
        & \cellcolor{gtrade!15}6813 \\

        ICNN-DRO & 128 & 9,034,242
        & 91.35
        & \cellcolor{gtrade!20}57.51
        & \cellcolor{gtrade!14}6909 \\

        ICNN-DRO & 512 & 10,213,890
        & 91.35
        & \cellcolor{gtrade!40}\textbf{58.46}
        & \cellcolor{gtrade!13}7095 \\

        \midrule
        \rowcolor{gray!12}
        \multicolumn{6}{l}{
            \textit{Moderate backbones}
            ~\texttt{[128, 128, 128, 128]}
        } \\
        \addlinespace[1pt]
        ICNN-DRO-Q
        & 8 & 6,402,690
        & 90.61
        & \cellcolor{rtrade!23}43.38
        & \cellcolor{rtrade!8}6271$^{\dagger}$ \\

        ICNN-DRO-Q
        & 32 & 6,476,418
        & 90.81
        & \cellcolor{rtrade!8}53.48
        & \cellcolor{rtrade!8}6548$^{\dagger}$ \\

        ICNN-DRO-Q
        & 128 & 6,771,330
        & 91.10
        & \cellcolor{rtrade!3}55.70
        & \cellcolor{rtrade!8}6729$^{\dagger}$ \\

        ICNN-DRO
        & 32 & 1,733,250
        & 90.85
        & \cellcolor{rtrade!8}53.46
        & \cellcolor{rtrade!8}4089$^{\dagger}$ \\

        ICNN-DRO
        & 128 & 2,028,162
        & 91.07
        & \cellcolor{rtrade!3}55.68
        & \cellcolor{rtrade!8}4135$^{\dagger}$ \\

        ICNN-DRO
        & 512 & 3,207,810
        & 91.36
        & \cellcolor{gtrade!35}\underline{58.22}
        & \cellcolor{gtrade!31}4473 \\

        \addlinespace[5pt]

        \midrule
        \rowcolor{gray!12}
        \multicolumn{6}{l}{
            \textit{Simple backbones}
            ~\texttt{[64, 64, 64, 64]}
        } \\
        \addlinespace[1pt]

        ICNN-DRO-Q
        & 8 & 3,219,778
        & 90.60
        & \cellcolor{rtrade!23}43.40
        & \cellcolor{rtrade!8}4462$^{\dagger}$ \\

        ICNN-DRO-Q
        & 32 & 3,293,506
        & 90.81
        & \cellcolor{rtrade!11}52.44
        & \cellcolor{rtrade!8}4483$^{\dagger}$ \\

        ICNN-DRO-Q
        & 128 & 3,588,418
        & 90.93
        & \cellcolor{rtrade!4}55.25
        & \cellcolor{rtrade!8}4665$^{\dagger}$ \\

        ICNN-DRO
        & 16 & 860,482
        & 90.67
        & \cellcolor{rtrade!18}47.56
        & \cellcolor{rtrade!8}3118$^{\dagger}$ \\

        ICNN-DRO
        & 32 & 909,634
        & 90.81
        & \cellcolor{rtrade!11}52.42
        & \cellcolor{rtrade!8}3415$^{\dagger}$ \\

        ICNN-DRO
        & 64 & 1,007,938
        & 90.83
        & \cellcolor{rtrade!9}53.17
        & \cellcolor{rtrade!8}3568$^{\dagger}$ \\

        ICNN-DRO
        & 128 & 1,204,546
        & 90.92
        & \cellcolor{rtrade!4}55.21
        & \cellcolor{rtrade!8}3713$^{\dagger}$ \\

        ICNN-DRO
        & 256 & 1,597,762
        & 91.32
        & \cellcolor{gtrade!10}56.81
        & \cellcolor{gtrade!40}\textbf{3867} \\

        ICNN-DRO
        & 512 & 2,384,194
        & 91.41
        & \cellcolor{gtrade!33}58.11
        & \cellcolor{gtrade!37}\underline{4037} \\

        \bottomrule
    \end{tabular}%
    }
    \label{tab:icnn-capacity}
\end{table*}

\myparagraph{Why the readout rank is the binding constraint.}
To understand why the expressiveness of the adversarial map is mainly determined by $r_{\mathrm{out}}$, we isolate the readout quadratic by setting $\delta_L=0$ in~\eqref{ICNN Architecture}. With $A_L =[a_1,\ldots,a_{r_{\mathrm{out}}}]^\top$, the transport map $T_\omega(z) = \nabla_z\psi_\omega(z)$ expands to
{%
\setlength{\abovedisplayskip}{2pt plus 1pt minus 1pt}%
\setlength{\belowdisplayskip}{2pt plus 1pt minus 1pt}%
\setlength{\abovedisplayshortskip}{0pt plus 1pt}%
\setlength{\belowdisplayshortskip}{1pt plus 1pt minus 1pt}%
}%
\[
T_\omega(z)
=\nabla_z\phi_\omega(z)+\sum_{k=1}^{r_{\mathrm{out}}}(a_ka_k^\top)z +w_L^z,
\]%

where $\phi_\omega(z)\coloneqq\exp(w_L^y)^\top y_L(z)$ collects the output of the hidden layers. The adversary perturbs nominal samples via two main components. The first component is the neural network gradient $\nabla_z\phi_\omega(z)$, and the second is a linear transformation of the input sample $\sum_{k=1}^{r_{\mathrm{out}}}(a_ka_k^\top)z$. Crucially, the second term is confined to $\operatorname{span}\{a_1,\ldots, a_{r_{\mathrm{out}}}\}$, meaning it can only ever explore a fixed set of at most $r_{\mathrm{out}}$ directions. While $\nabla_z\phi_\omega(z)$ could theoretically supply additional directions, it is empirically dominated by the second term, meaning that the relative norm $\|\nabla_z\phi_\omega(z)\|_2/\| A_L^\top A_Lz\|_2$ is empirically small. More precisely, across $4{,}096$ training samples, this relative norm is merely $0.0711$ for $r_{\mathrm{out}}=1$ and $0.0910$ for $r_{\mathrm{out}}=2$. Because the second term dictates the overall perturbation magnitude, $r_{\mathrm{out}}$ effectively governs the map's total directional capacity. This explains why increasing the hidden backbone width or depth (\autoref{tab:icnn-capacity}) cannot overcome a low readout rank, and increasing $r_{\mathrm{out}}$ serves as the main component of the expressivity of the ICNN.

\subsection{Penalty Annealing for the Global Map}
\label{sec:lambda_annealing}
\begin{table}[t!]%
    \centering
    \caption{\textbf{Penalty annealing versus fixed-penalty adversarial training on CIFAR-10.} The $K=20$ fixed-$\lambda$ baseline is reproduced from \autoref{tab:CIFAR-10-CIFAR101-CIFAR102} for comparison. Best results per column are in bold, and the second best are underlined.}
    \setlength{\tabcolsep}{5pt}
    \renewcommand{\arraystretch}{1.2}
    \resizebox{0.99\linewidth}{!}{%
    \begin{tabular}{l | c | ccc c c cccccc}
        \toprule
        \multirow{2.5}{*}{Method}
        & \multirow{2.5}{*}{$K$}
        & \multicolumn{3}{c}{CIFAR-10}
        & \multicolumn{1}{c}{CIFAR-10.1}
        & \multicolumn{1}{c}{CIFAR-10.2}
        & \multicolumn{6}{c}{CIFAR-10-C (Common Corruption)} \\
        \cmidrule(lr){3-5} \cmidrule(lr){6-6} \cmidrule(lr){7-7} \cmidrule(lr){8-13}
        & & Clean & PGD$_{\varepsilon=0.5}$ & AA$_{\varepsilon=0.5}$ & Clean & Clean & 1 & 2 & 3 & 4 & 5 & Avg \\
        \midrule
        ICNN-DRO~(fixed $\lambda=30$)
        & 50
        & 91.20 & \underline{59.83} & \underline{58.81} & \underline{82.57} & 79.98 & \underline{89.28} & \underline{86.85} & \underline{84.10} & \underline{80.00} & \underline{74.06} & \underline{82.85} \\
        ICNN-DRO~(fixed $\lambda=30$)
        & 20
        & \underline{91.41} & 58.11 & 57.71 & 82.56 & \textbf{80.12} & 89.20 & 86.79 & 84.01 & 79.92 & 73.95 & 82.77 \\
        
        ICNN-DRO~+~$\lambda$-annealing
        & 20
        & \textbf{91.48} & \textbf{59.94} & \textbf{58.92} & \textbf{82.60} & \underline{80.10} & \textbf{89.31} & \textbf{86.93} & \textbf{84.15} & \textbf{80.04} & \textbf{74.12} & \textbf{82.91} \\
        \bottomrule
    \end{tabular}
    }
    \label{tab:CIFAR-10-CIFAR101-CIFAR102-lambdaannealing}
\end{table}%
The penalty parameter $\lambda$ in~\eqref{prob:regularized_dro} controls the strength of the adversary. A large $\lambda$ severely constrains the adversary, forcing the worst-case map toward the identity and reducing the objective to standard ERM. Conversely, a small $\lambda$ permits a highly potent adversary. Rather than committing to a single $\lambda$, we anneal it along a decreasing schedule $\lambda_1 > \lambda_2 > \cdots > \lambda_S$ and, crucially, \emph{never reset the adversarial map between stages}. The ICNN weights $\omega$ and the classifier $\theta$ learned under $\lambda_s$ initialize the next phase under $\lambda_{s+1}$. This defines a smooth sequence of penalized DRO problems, one per penalty level. Note that warm-starting each stage from the previous solution is coherent because our adversary is a single global map $T_\omega = \nabla_z \psi_\omega$ defined on all of $\mathbb{R}^m$, \ie, what carries across stages is the entire transport, not a set of transductive particles. \looseness=-1 

This scheduling serves two main purposes. First, it creates an adversarial curriculum for the classifier. The adversary begins near the ERM baseline at $\lambda_1$ and steadily strengthens as $\lambda$ decreases. Second, it stabilizes the ICNN optimization by allowing the map's displacement to grow incrementally. Because the worst-case displacement scales as $1/\lambda$, warm-starting avoids the unstable jumps of a cold solve, instead expanding the transport map by a modest factor at each stage. Note that $T_\omega$ is always the gradient of a convex potential, the refined map remains cyclically monotone throughout the entire schedule. \looseness=-1

\myparagraph{Setup.} We evaluate this schedule using the CIFAR-10 image-space pipeline from~\autoref{tab:CIFAR-10-CIFAR101-CIFAR102} (\S~\ref{sec:Robustness_AA_Full}). We retain the exact ICNN architecture used in the main text (a $[64,64,64,64]$ backbone, readout-only quadratic injection, and $r_{\mathrm{out}}=128$). Consequently, this run differs from the fixed-penalty baseline solely in its penalty schedule. The penalty decreases along $\lambda \in \{60, 50, 40, 35, 30\}$ over $S=5$ stages. Each stage lasts for $E=10$ adversarial epochs ($50$ in total, preceded by a $2$-epoch adversary-only warmup at $\lambda_1=60$). Because the schedule terminates at the pipeline's standard operating point ($\lambda_S=30$), it steadily transitions from a mild, near-ERM adversary down to our standard robust formulation. In each stage, we perform $K=20$ BB+Armijo inner ascent steps per batch. Crucially, neither the classifier $\theta$ nor the ICNN weights $\omega$ are reset at stage boundaries, allowing the transport map to be continuously refined across penalty levels rather than re-solved from scratch.\looseness=-1

\myparagraph{Results.} As shown in~\autoref{tab:CIFAR-10-CIFAR101-CIFAR102-lambdaannealing}, penalty annealing attains the best performance across nearly every metric. Compared to the matched fixed-$\lambda$ baseline at $K=20$, the curriculum improves every reported metric except CIFAR-10.2, where it is lower by $0.02$ percentage points. Remarkably, it even yields a higher PGD accuracy than the more computationally expensive fixed-$\lambda$ configuration ($59.94\%$ vs.\ $59.83\%$), despite using $2.5\times$ fewer inner ascent steps ($K=20$ vs.\ $50$). This demonstrates that the curriculum discovers a stronger adversary at a fraction of the inner-loop cost. Crucially, this avoids the typical robustness-accuracy trade-off: the annealed model simultaneously achieves the highest clean CIFAR-10 accuracy ($91.48\%$) and superior performance on the natural shifts of CIFAR-10.1, CIFAR-10.2, and CIFAR-10-C. These gains are consistent with the warm-started curriculum successfully stabilizing the classifier $\theta$ before the adversary reaches full strength.  \looseness=-1



\end{document}